\documentclass[12pt]{article}

\usepackage{hyperref}
\hypersetup{
  colorlinks=true,
  linkcolor={red!60!black},
  citecolor={blue!60!black},
  urlcolor={magenta!60!black}
} 
\usepackage{times}
\usepackage{verbatim}
\usepackage{mathtools,stmaryrd}
\usepackage{subcaption,wrapfig}

\usepackage{amsmath, amssymb, amsthm, amsfonts}
\usepackage{times}
\usepackage{xcolor}
\usepackage{graphicx}
\usepackage[letterpaper, margin=1in]{geometry}
\usepackage[shortlabels]{enumitem}
\setlist[itemize]{leftmargin=*}
\setlist[enumerate]{leftmargin=*}
\usepackage[capitalize]{cleveref}
\usepackage{titlesec}
\usepackage{mathrsfs}
\usepackage{physics}
\usepackage{booktabs}
\usepackage{tensor}
\usepackage{algpseudocode,algorithm}
\usepackage[numbers]{natbib} 

\titlespacing*{\section}{0pt}{6pt}{0pt}
\titlespacing*{\subsection}{0pt}{6pt}{0pt}

\DeclareMathOperator{\E}{\mathbb{E}}
\DeclareMathOperator{\Var}{Var}

\newcommand{\R}{\mathbb{R}}

\newcommand{\1}{\mathbf{1}}

\renewcommand{\op}{\mathrm{op}}

\DeclareMathOperator{\diag}{diag}
\DeclareMathOperator{\sgn}{sign}

\DeclareMathOperator{\poly}{poly}

\let\vec\relax\DeclareMathOperator{\vec}{Vec}

\DeclareMathOperator{\F}{F}
\DeclareMathOperator{\pol}{polar}

\newtheorem{theorem}{Theorem}[section]
\newtheorem{lemma}[theorem]{Lemma}
\newtheorem{corollary}[theorem]{Corollary}
\newtheorem{proposition}[theorem]{Proposition}
\newtheorem{assumption}[theorem]{Assumption}
\newtheorem*{conjecture}{Conjecture}

\theoremstyle{remark}
\newtheorem*{remark}{Remark}

\newcommand{\bone}{\boldsymbol{1}}
\ifdefined\usebigfont

\usepackage{times}
\usepackage[fontsize=13pt]{scrextend}
\AtBeginDocument{
	\newgeometry{letterpaper,left=1.56in,right=1.56in,top=1.71in,bottom=1.77in}
}
\else
\fi
\def\safedef#1{% 
   \ifx#1\undefined
      \expandafter\def\expandafter#1%
   \else
      \errmessage{The \string#1 is defined already}% 
      \expandafter\def\expandafter\tmp
   \fi
}

\theoremstyle{definition} % no italics

\setenumerate[1]{leftmargin=*, label={\upshape(\arabic*)}}
\newcommand{\nnorm}[1]{\lVert#1\rVert}

\newcommand{\Ebig}[1]{\mathbb{E}\left[#1\right]}
\newcommand{\EE}[2]
{\mathbb{E}_{#1}[#2]}

\newcommand*{\rd}{\mathop{}\!\mathrm{d}}
\newcommand*{\eps}{\epsilon}

\newcommand*{\lam}{\lambda}
\newcommand{\deq}{\mathrel{\stackrel{d}{=}}}
\newcommand{\bmat}[1]{\begin{bmatrix}#1\end{bmatrix}}
\newcommand{\db}[1]{\llbracket #1 \rrbracket}
\newcommand{\loo}{^{(-i)}}
\newcommand{\abi}{^{(+i)}}
\newcommand{\bbG}{\bar{\bG}}

\DeclareMathOperator{\Unif}{Unif}

\DeclareMathOperator{\Cov}{Cov}
\DeclareMathOperator{\polylog}{polylog}

\DeclareMathOperator{\Lip}{Lip}

\DeclareMathOperator{\Bin}{Bin}

\DeclareMathOperator{\tail}{tail}

\DeclareMathOperator{\SGD}{SGD}
\DeclareMathOperator{\Spec}{Spec}
\DeclareMathOperator{\Muon}{Muon}
\DeclareMathOperator{\Newton}{Newton}

\DeclareMathOperator*{\argmax}{arg\,max}

\def\ddefloop#1{\ifx\ddefloop#1\else\ddef{#1}\expandafter\ddefloop\fi}

\def\ddef#1{\expandafter\def\csname #1#1\endcsname{\ensuremath{\mathbb{#1}}}}
\ddefloop ABCDFGHIJKLMNOPQRSTUVWXYZ\ddefloop % except for E

\def\ddef#1{\expandafter\def\csname c#1\endcsname{\ensuremath{\mathcal{#1}}}}
\ddefloop ABCDEFGHIJKLMNOPQRSTUVWXYZ\ddefloop

\def\ddef#1{\expandafter\def\csname b#1\endcsname{\ensuremath{{\mathbf{#1}}}}}
\ddefloop ABCDEFGHIJKLMNOPQRSTUVWXYZ\ddefloop  
\def\ddef#1{\expandafter\def\csname b#1\endcsname{\ensuremath{{\boldsymbol{#1}}}}}
\ddefloop abcdeghijklmnopqrtsuvwxyz\ddefloop  % except for bf, which often causes an issue

\def\ddef#1{\expandafter\def\csname h#1\endcsname{\ensuremath{\hat{#1}}}}
\ddefloop ABCDEFGHIJKLMNOPQRSTUVWXYZabcdefghijklmnopqrsuvwxyz\ddefloop % except for 't'
\def\ddef#1{\expandafter\def\csname hc#1\endcsname{\ensuremath{\hat{\mathcal{#1}}}}}
\ddefloop ABCDEFGHIJKLMNOPQRSTUVWXYZ\ddefloop
\def\ddef#1{\expandafter\def\csname hb#1\endcsname{\ensuremath{\hat{\mathbf{#1}}}}}
\ddefloop ABCDEFGHIJKLMNOPQRSTUVWXYZ\ddefloop % 
\def\ddef#1{\expandafter\def\csname hb#1\endcsname{\ensuremath{\hat{\boldsymbol{#1}}}}}
\ddefloop abcdefghijklmnopqrstuvwxyz\ddefloop % 

\def\ddef#1{\expandafter\def\csname t#1\endcsname{\ensuremath{\tilde{#1}}}}
\ddefloop ABCDEFGHIJKLMNOPQRSTUVWXYZabcdefgijklmnpqtsuvwxyz\ddefloop % except for o and h and r
\def\ddef#1{\expandafter\def\csname tc#1\endcsname{\ensuremath{\tilde{\mathcal{#1}}}}}
\ddefloop ABCDEFGHIJKLMNOPQRSTUVWXYZ\ddefloop
\def\ddef#1{\expandafter\def\csname tb#1\endcsname{\ensuremath{\tilde{\mathbf{#1}}}}}
\ddefloop ABCDEFGHIJKLMNOPQRSTUVWXYZ\ddefloop
\def\ddef#1{\expandafter\def\csname tb#1\endcsname{\ensuremath{\tilde{\boldsymbol{#1}}}}}
\ddefloop abcdefghijklmnopqrstuvwxyz\ddefloop % 

\def\ddef#1{\expandafter\def\csname bar#1\endcsname{\ensuremath{\bar{#1}}}}
\ddefloop ABCDEFGHIJKLMNOPQRSTUVWXYZabcdefghijklmnopqrtsuvwxyz\ddefloop
\def\ddef#1{\expandafter\def\csname barc#1\endcsname{\ensuremath{\bar{\mathcal{#1}}}}}
\ddefloop ABCDEFGHIJKLMNOPQRSTUVWXYZ\ddefloop
\def\ddef#1{\expandafter\def\csname barb#1\endcsname{\ensuremath{\bar{\mathbf{#1}}}}}
\ddefloop ABCDEFGHIJKLMNOPQRSTUVWXYZ\ddefloop
\def\ddef#1{\expandafter\def\csname barb#1\endcsname{\ensuremath{\bar{\boldsymbol{#1}}}}}
\ddefloop abcdefghijklmnopqrstuvwxyz\ddefloop % 

\def\ddef#1{\expandafter\def\csname war#1\endcsname{\ensuremath{\overline{#1}}}}
\ddefloop ABCDEFGHIJKLMNOPQRSTUVWXYZabcdefghijklmnopqrtsuvwxyz\ddefloop
\def\ddef#1{\expandafter\def\csname warc#1\endcsname{\ensuremath{\overline{\mathcal{#1}}}}}
\ddefloop ABCDEFGHIJKLMNOPQRSTUVWXYZ\ddefloop
\def\ddef#1{\expandafter\def\csname warb#1\endcsname{\ensuremath{\overline{\mathbf{#1}}}}}
\ddefloop ABCDEFGHIJKLMNOPQRSTUVWXYZ\ddefloop
\def\ddef#1{\expandafter\def\csname warb#1\endcsname{\ensuremath{\overline{\boldsymbol{#1}}}}}
\ddefloop abcdefghijklmnopqrstuvwxyz\ddefloop % 

\safedef\tilr{\tilde r}
\safedef\bff{{\boldsymbol f}}
\safedef\hbff{{\hat{\boldsymbol f}}}
\safedef\hatt{\hat{t}}
\safedef\tilo{{\tilde{o}}}
\safedef\tilh{{\tilde{h}}}
\safedef\bell{{{\boldsymbol\ell}}}
\safedef\tell{\ensuremath{\tilde{\ell}}} 
\safedef\btell{\ensuremath{\widetilde{\boldsymbol{\ell}}}} 
\safedef\hell{{{\hat\ell}}}

\usepackage{pgffor}
\safedef\greeksymbols{alpha,beta,gamma,delta,eps,epsilon,zeta,eta,theta,iota,kappa,lambda,mu,nu,xi,pi,rho,sigma,tau,phi,chi,psi,omega,Gamma,Delta,Theta,Lambda,Pi,Sigma,Phi,Psi,Omega,Xi}
\safedef\greeksymbolsnoeta{alpha,beta,gamma,delta,eps,epsilon,zeta,theta,iota,kappa,lambda,mu,nu,xi,pi,rho,sigma,tau,phi,chi,psi,omega,Gamma,Delta,Theta,Lambda,Pi,Sigma,Phi,Psi,Omega,Xi} % except for eta

\foreach \x in \greeksymbolsnoeta{\expandafter\xdef\csname b\x\endcsname{\noexpand\ensuremath{\noexpand\boldsymbol{\csname \x\endcsname}}}}
\safedef\bfeta{{\boldsymbol \eta}}

\foreach \x in \greeksymbols{\expandafter\xdef\csname h\x\endcsname{\noexpand\ensuremath{\noexpand\hat{\csname \x\endcsname}}}}
\foreach \x in \greeksymbolsnoeta{\expandafter\xdef\csname hb\x\endcsname{\noexpand\ensuremath{\noexpand\hat{\noexpand\boldsymbol{ \csname \x\endcsname}}}}}
\safedef\hbfeta{{\hat{\boldsymbol \eta}}}

\foreach \x in \greeksymbols{\expandafter\xdef\csname bar\x\endcsname{\noexpand\ensuremath{\noexpand\bar{\csname \x\endcsname}}}}
\foreach \x in \greeksymbolsnoeta{%
\expandafter\xdef\csname barb\x\endcsname{\noexpand\ensuremath{\noexpand\bar{\noexpand\boldsymbol{ \csname \x\endcsname}}}}
}
\safedef\barbfeta{{\bar{\boldsymbol \eta}}}

\foreach \x in \greeksymbols{\expandafter\xdef\csname t\x\endcsname{\noexpand\ensuremath{\noexpand\tilde{\csname \x\endcsname}}}}
\foreach \x in \greeksymbolsnoeta{\expandafter\xdef\csname tb\x\endcsname{\noexpand\ensuremath{\noexpand\tilde{\noexpand\boldsymbol{ \csname \x\endcsname}}}}}
\safedef\tbfeta{{\tilde{\boldsymbol \eta}}}

\newcommand{\jk}[1]{{\color{orange}[JK: #1]}}
\newcommand{\new}[1]{{#1}}

\title{Sharp Capacity Scaling of Spectral Optimizers in Learning Associative Memory}

\ifdefined\usebigfont

\usepackage[fontsize=13pt]{scrextend}
\makeatletter
\@ifpackageloaded{geometry}{\AtBeginDocument{\newgeometry{letterpaper,left=1.56in,right=1.56in,top=1.71in,bottom=1.77in}}}{\usepackage[letterpaper,left=1.56in,right=1.56in,top=1.71in,bottom=1.77in]{geometry}}
\AtBeginDocument{\newgeometry{letterpaper,left=1.56in,right=1.56in,top=1.71in,bottom=1.77in}}
\usepackage{hyperref} % watch out for error: ! Illegal parameter number in definition of \Hy@tempa.
\else
\usepackage{times}
\fi

\author{
Juno Kim\textsuperscript{1,}\thanks{Equal contribution.}
\qquad
Eshaan Nichani\textsuperscript{2,}\footnotemark[1]
\\[0.5em]
Denny Wu\textsuperscript{3,4}
\qquad
Alberto Bietti\textsuperscript{4}
\qquad
Jason D. Lee\textsuperscript{1}
\\[0.5em]
\normalsize \textsuperscript{1}UC Berkeley \quad
\textsuperscript{2}Princeton University \quad
\textsuperscript{3}New York University \quad
\textsuperscript{4}Flatiron Institute
}
\date{\today\\[-1em]}

\setlist[itemize]{leftmargin=*}
\setlist[enumerate]{leftmargin=*}

\titlespacing*{\paragraph}{0pt}{0.8ex}{1em}

\begin{document}

\maketitle

\allowdisplaybreaks

\begin{abstract}
Spectral optimizers such as Muon have recently shown strong empirical performance in large-scale language model training, but the source and extent of their advantage remain poorly understood. We study this question through the linear associative memory problem, a tractable model for factual recall in transformer-based models. In particular, we go beyond orthogonal embeddings and consider Gaussian inputs and outputs, which allows the number of stored associations to greatly exceed the embedding dimension. Our main result sharply characterizes the recovery rates of one step of Muon, SGD, and Newton's method on the logistic regression loss under a power law frequency distribution. We show that the storage capacity of Muon significantly exceeds that of SGD, and even matches Newton's method while only using first-order information. Moreover, Muon saturates at a larger critical batch size. We further analyze the multi-step dynamics under a thresholded gradient approximation and show that Muon achieves a substantially faster initial recovery rate than SGD, while both methods eventually converge to the information-theoretic limit at comparable speeds. Experiments on synthetic tasks validate the predicted scaling laws. Our analysis provides a quantitative understanding of the signal amplification of spectral preconditioners and lays the groundwork for establishing scaling laws across more practical language modeling tasks and optimizers.
\end{abstract}

\section{Introduction}

Large language models (LLMs) with billions of parameters are typically trained using adaptive first-order optimization algorithms. The workhorse of modern neural network optimization has long been the Adam optimizer and its variants~\citep{kingma2014adam,loshchilov2019decoupledweightdecayregularization}. However, there has been growing interest in matrix-based or \emph{spectral} optimizers \citep{martens2015optimizing,gupta2018shampoo,vyas2024soap,jordan2024muon}, which explicitly utilize the matrix structure of neural network parameters. Among these methods, Muon~\citep{jordan2024muon} has shown strong empirical performance in large-scale pretraining studies~\citep{liu2025muon}, even outperforming Adam at sufficiently large batch sizes~\citep{wen2025fantastic, semenov2025benchmarking}. 
Muon updates each weight matrix in the approximate direction of the polar factor, or spectral orthogonalization, of the negative gradient. Ignoring accumulation, this can also be interpreted as steepest descent with respect to the spectral norm~\citep{bernstein2024oldoptimizernewnorm}. However, it remains unclear which aspects of modern language model training make this update particularly effective.

To investigate this question, we analyze the dynamics of Muon versus stochastic gradient descent (SGD) and Newton's method on the task of learning linear \emph{associative memory}. The associative memory task, introduced in \citet{cabannesscaling,nichani2024understanding}, provides a simple model of factual recall in language models \citep{roberts2020much,allen2024physics}, and captures the ability of transformer-based models to store factual knowledge within the self-attention matrices. The goal is to store a collection of atomic associations (i.e., facts), expressed as~$N$ pairs of input and output embeddings $\{(v_i, u_i)\}_{i \in [N]} \subset \R^d$, using a weight matrix~$\bW\in\RR^{d\times d}$ so that $u_i\approx \bW v_i$. We train~$\bW$ by casting this task as a multiclass logistic regression problem with logits given by $u_j^\top\bW v_i$ and optimizing the cross-entropy loss. A formal description of the problem is given in Section~\ref{sec:setting}.

Recent work has studied the benefit of spectral optimizers on related associative memory tasks~\citep{wang2025muon,li2026muon}, but these results rely on an orthogonality assumption on the embeddings $u_i$ and $v_i$. While this assumption simplifies the optimization analysis, it also requires the embedding dimension $d$ to be larger than the number of stored items $N$. In contrast, we study the regime in which $u_i$ and $v_i$ are drawn i.i.d.~from an isotropic Gaussian distribution, so that the number of stored items can greatly exceed the embedding dimension ($N \gg d$). This captures the ability of language models to store items, or features, in \emph{superposition}~\citep{elhage2022superposition}, where the total number of features is far greater than the ambient dimension. Indeed, under this random-embedding model, it is information-theoretically possible for $\bW$ to store up to $\widetilde{\Theta}(d^2)$ items~\citep{nichani2024understanding}. At the same time, removing orthogonality makes the learning dynamics substantially more intricate~\citep{vurallearning} and they remain poorly understood.

Motivated by Zipf's law for language modeling, we assume that the~$i$th item appears with power-law frequency~$p_i\sim i^{-\alpha}$, parallel to previous theoretical analyses on scaling laws~\citep{michaud2023quantization,bordelon2024dynamical, lin2024scaling, paquette20244+, ren2025emergence, kunstner2025scaling}. We also consider the minibatch versions of SGD and Muon, where at each timestep a new batch of size $B$ is sampled with replacement.
Under this setting, our main result sharply characterizes the one-step recovery of Muon, showing that Muon outperforms SGD and stores significantly more items than in the orthogonal case.

\begin{theorem}[Informal version of Theorems~\ref{thm:main},~\ref{thm:gd},~\ref{thm:newton}] Let~$d$ be the embedding dimension and~$B$ be the batch size, and suppose the~$i$th item has power law frequency $p_i \propto i^{-\alpha}$ for $\alpha > 1$. One step of Muon on the associative memory task recovers the top $\widetilde\Theta(\min\{d^{1 + \frac{1}{2\alpha}}, B^{\frac{1}{\alpha}}\})$ most frequent items, \new{matching the Newton update,} while one step of SGD recovers the $\widetilde\Theta(\min\{d^{\frac{1}{2\alpha}}, B^{\frac{1}{\alpha}}\})$ items.
\end{theorem}

\new{Surprisingly, Muon is even able to match Newton's method -- the gold standard of curvature-aware optimization -- using only first-order information, demonstrating the power of spectral preconditioning.} The theorem also implies that the \emph{critical batch size}, beyond which increasing batch size does not yield performance gains, is much larger for Muon compared to SGD. The capacity exponents and batch size saturation predicted by our theory are empirically verified by our experiments (Figure~\ref{fig:intro}).

\begin{figure}[t]
% \vspace{-2.5mm} 
\centering
\begin{subfigure}[t]{0.49\linewidth}
\centering
{\includegraphics[height=0.77\textwidth]{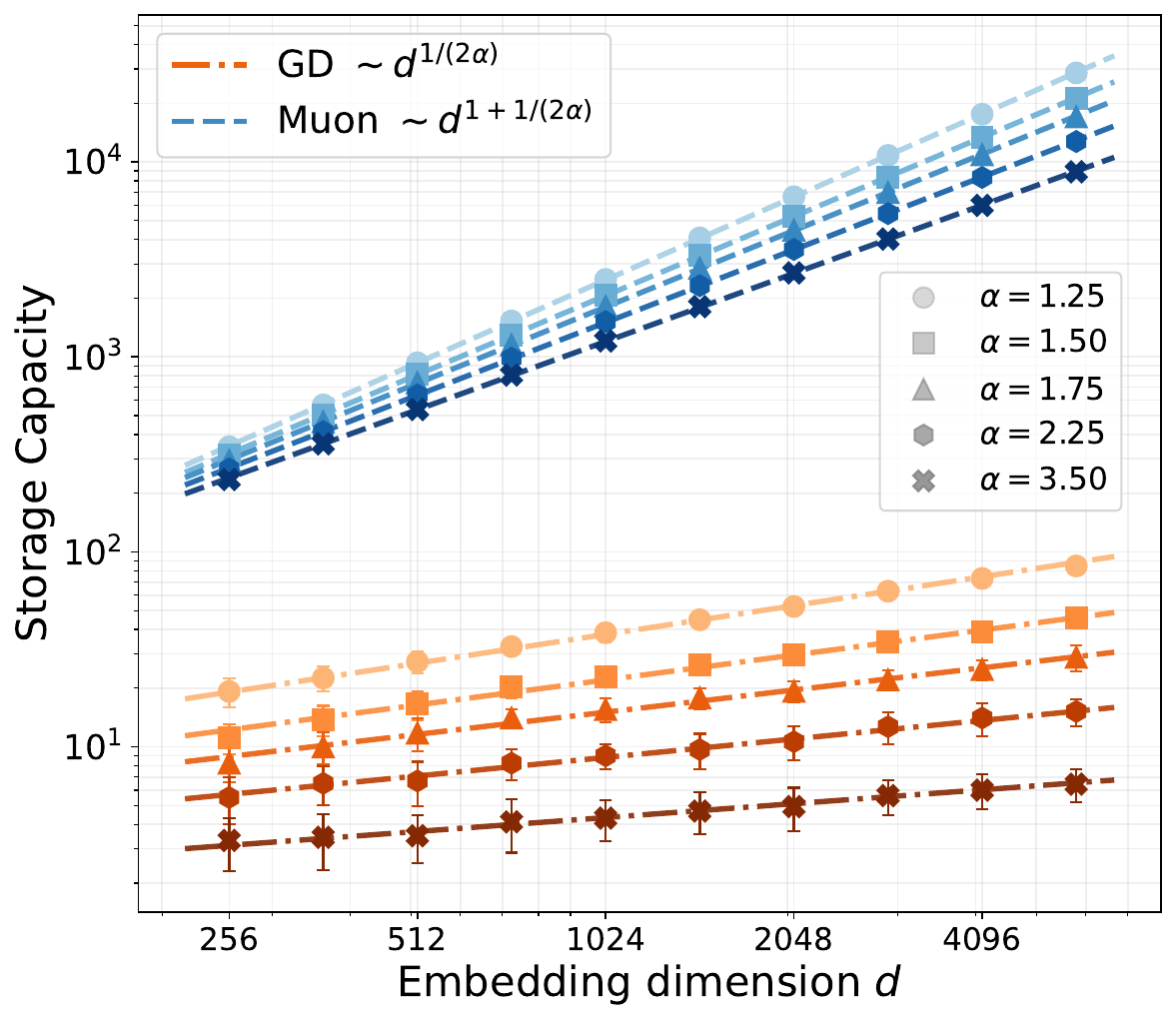}} 
%\vspace{-0.5mm}
\caption{Capacity scaling with embedding dimension $d$.}
\label{fig:intro-population}
\end{subfigure}%
\begin{subfigure}[t]{0.49\linewidth}
\centering 
{\includegraphics[height=0.77\textwidth]{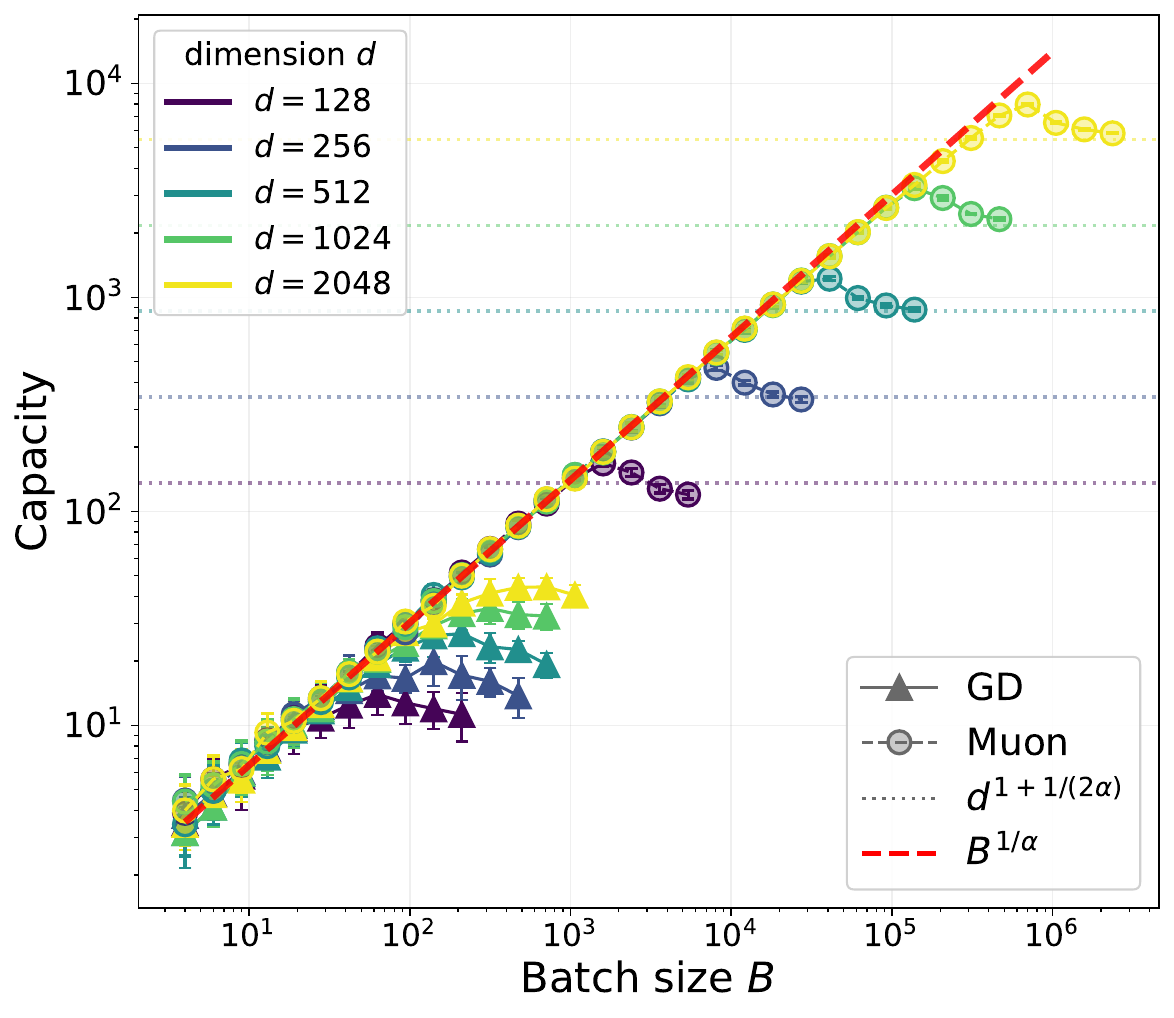}} 
%\vspace{-0.5mm}
\caption{\small Capacity scaling with minibatch size $B$.} 
\label{fig:intro-batch}
\end{subfigure}%
% \vspace{-1mm} 
\caption{\small \textbf{(a)} Capacity achieved by one Muon and GD step on the population objective; Muon improves the storage capacity when frequency is power-law distributed with exponent $\alpha>1$. \textbf{(b)} Critical batch size for the first Muon and SGD step ($\alpha=1.5$); the Muon capacity saturates at a much larger batch size than SGD.}  
% \vspace{-0.5mm}
\label{fig:intro} 
\end{figure}  

Furthermore, we study the multi-step trajectories of Muon and SGD under a simplifying thresholded gradient update, and show the following scaling laws for the recovery rate (for brevity, we only state the population version below).

\begin{theorem}[Informal version of Theorems~\ref{thm:multi},~\ref{thm:gd-multi}]
Under the thresholded update,~$t$ steps of Muon recover the top $\widetilde{\Theta}(d^{2-(1-\frac{1}{2a})^t})$ items. In contrast, $t$~steps of SGD recover the top~$d_t$ items where~$d_t$ is given by the recursion $d_{t+1} = \widetilde{\Theta}(d^{\frac{1}{2\alpha}}d_t)$ if $d_t\lesssim d$, and $d_{t+1} = \widetilde{\Theta}(d^\frac{1}{\alpha}d_t^{1-\frac{1}{2\alpha}})$ if $d_t\gtrsim d$.
\end{theorem}
The main takeaways from our analysis are as follows.
\begin{enumerate}[topsep=0.5mm,itemsep=0.5mm]
    \item \textbf{Muon improves storage efficiency.} In the population regime $B \rightarrow \infty$, one step of Muon is able to recover the top~$d^{1 + \frac{1}{2\alpha}}$ items, \new{matching Newton's method,} while one step of SGD only recovers the top~$d^{\frac{1}{2\alpha}}$ (see Figure~\ref{fig:intro-population}). Noticeably, a single step of Muon is able to store more than~$d$ items, which is the maximal value when embeddings are constrained to be orthogonal. In other words, Muon  effectively stores more features than dimensions via \textit{superposition}.
    \item \textbf{The benefit of Muon comes at larger batch sizes.} When the batch size~$B$ is small, Muon and SGD both recover the top~$B^{\frac{1}{\alpha}}$ items. However, the performance of SGD saturates at a batch size of $B \asymp \sqrt{d}$, while Muon saturates at the much larger batch size of $B \asymp d^{\alpha + \frac12}$ (see Figure~\ref{fig:intro-batch}). This provides evidence that Muon has the ability to handle much larger batch sizes, and aligns with empirical observations that Muon significantly outperforms non-spectral optimizers only at large batch sizes~\citep{wen2025fantastic}.
    \item \textbf{Muon accelerates early in training.} The SGD exponent initially exhibits a slow linear scaling and requires $\lceil 2\alpha\rceil$ steps to reach $d_t\gtrsim d$, while Muon exceeds this in a single step. However, once SGD enters this regime, both updates obey the same recursion $d_{t+1} \sim d^\frac{1}{\alpha}d_t^{1-\frac{1}{2\alpha}}$ with the same convergence rate to the optimal $\widetilde{\Theta}(d^2)$ capacity. Thus, the main benefits of Muon (with appropriate batch size) appear earlier in training when gradients are strongly anisotropic, explaining the short-term gains observed in~\citep{semenov2025benchmarking}.
\end{enumerate}

The rest of the paper is organized as follows. In Section \ref{sec:setting}, we formally define the associative memory task and the family of spectral optimizers considered in our analysis. Section \ref{sec:one_step} contains our main results on the scaling of a single step of Muon, SGD, and Newton. In Section \ref{sec:new}, we argue that Muon is the asymptotically optimal one-step update, and extend our scaling analysis to multiple steps along the Muon and SGD trajectories. We conclude in Section~\ref{sec:experiments} with simulations verifying our predicted scaling laws and batch size analysis, as well as experiments with transformers on a more sophisticated in-context recall task.

\section{Related Work}

\paragraph{Associative memory and factual recall.} Associative memory has a long history in neural computation~\citep{willshaw1969non, Kohonen1972CorrelationMM, hopfield1982neural}. Recent work has shown that transformer weights can be viewed as associative memories storing input-output mappings between pairs of concepts~\citep{geva2021transformer, bietti2023birth, cabannesscaling, jiang2024llms}. This perspective is especially useful for modeling factual recall~\citep{allen2024physics}, where such mechanisms encode factual knowledge directly in the weights of a transformer~\citep{meng2022locating, geva2023dissecting, nichani2024understanding}.

Our most relevant points of comparison are \citet{wang2025muon, li2026muon}, which analyze Muon on similar associative memory tasks. Their analysis assumes that the embeddings $\{u_i\}_{i\in[N]}$ and $\{v_i\}_{i\in[N]}$ are pairwise orthogonal, greatly simplifying the study of the polar map. However, this assumption also limits the model capacity to at most $N \le d$ stored associations. By contrast, because $\bW$ has $d^2$ parameters, the information-theoretically optimal capacity is $\widetilde{\Theta}(d^2)$~\citep{nichani2024understanding}, which requires storing embeddings in superposition. In this regime, we show that a single Muon step already recovers $\widetilde{\Theta}\!\left(d^{1+\frac{1}{2\alpha}}\right)$ items, far beyond what is possible under orthogonality. We further derive scaling laws for the multi-step Muon and SGD dynamics, and show that both updates indeed approach the optimal capacity.

\paragraph{Theoretical analyses of Muon.} A number of recent works have sought to rigorously characterize the benefits of Muon and other matrix-based optimizers over SGD. One line of work derives convergence guarantees using tools from convex optimization and online learning. \citet{shen2025convergence,chen2025muon,kim2026convergence} prove convergence guarantees for Muon that depend on smoothness in the spectral norm or the spectral norm of the weight matrix itself. \citet{jiang2026adaptive} derive regret bounds and corresponding non-convex optimization rates. Beyond Muon, \citet{xie2025structured} develop a framework for analyzing convergence rates of a broad class of matrix preconditioners on smooth convex problems, while \citet{lau2025polargrad} adopt a structure-aware preconditioning perspective to introduce a new family of matrix-based optimizers.

A second line of work studies the loss reduction after a single descent step. \citet{davis2025spectral} show that Muon achieves a larger one-step loss reduction than SGD when the gradient rank exceeds the activation rank. \citet{su2025isotropic} introduces an ``isotropic curvature model'' based on a single optimization step and show that Muon is optimal in certain regimes. However, \citet{gonon2026insights} demonstrate that such single-step arguments can fail to predict full end-to-end convergence rates.

Finally, other works compare Muon and SGD on specific problem classes. \citet{fan2025implicit} show that, for separable classification, Muon converges to the solution maximizing the spectral-norm margin. For matrix-valued linear regression, \citet{wang2025high} characterize the risk of Muon on isotropic data, while \citet{vasudeva2025muon} show faster convergence than SGD under imbalanced covariates. \citet{ma2026preconditioning} further show that, in matrix factorization, Muon attains a convergence rate faster than SGD and independent of the condition number.

\paragraph{Adaptivity to heavy-tailed data.} Our results show that Muon is particularly effective when the fact distribution is power-law distributed with heavy tail. Similar advantages have also been observed for other adaptive optimizers. \citet{kunstner2024heavy,yadav2025provable} show that Adam and its limiting variant SignSGD outperform SGD when the class distribution follows a power law. \citet{kunstner2025scaling} further prove that SignSGD outperforms SGD for learning a bigram model, while \citet{kimscaling} derive scaling laws for SignSGD in the power-law random features model.

\section{Setting: Associative Memory}\label{sec:setting}

\paragraph{Linear associative memory.} The goal of the associative memory task is to store a collection of atomic associations, or \emph{facts}. Let~$[N]$ be the input and output vocabulary. A set of facts is defined by a bijection $f^*:[N] \rightarrow [N]$, where the input token~$i$ is mapped to the output token $f^*(i)$. Each token is assigned an embedding vector $v_i \in \mathbb{R}^d$ and an unembedding vector $u_i \in \mathbb{R}^d$, sampled i.i.d. from the distribution $\cN(0, \frac{1}{d}\bI_d)$. Without loss of generality, we will assume that~$f^*(i) = i$ for all $i \in [N]$. As an illustrative example, consider the set of countries \textcolor{red!60!black}{$\mathcal{S} = \{\text{USA, France, Japan,}\dots\}$} and the set of capitals \textcolor{blue!60!black}{$\mathcal{A} = \{\text{Washington~D.C., Paris, Tokyo} \dots\}$}, with the goal being to store the mapping between each country and its capital. The embeddings of the countries \textcolor{red!60!black}{USA, France, Japan, $\dots$} are \textcolor{red!60!black}{$v_1, v_2, v_3, \dots$} respectively, and the embeddings of the capitals \textcolor{blue!60!black}{Washington~D.C., Paris, Tokyo, $\dots$} are \textcolor{blue!60!black}{$u_1, u_2, u_3, \dots$} respectively.

We consider training a linear associative memory model, given by a weight matrix $\bW\in\RR^{d\times d}$, to store the fact dataset as the following multi-class classification problem. The score prediction for the unembedding token $u_j$ associated to $v_i$ is defined as
\begin{align*}
    \hat p_{\bW} (j \mid i) := \frac{\exp(u_j^\top \bW v_i)}{\sum_{k \in [N]} \exp(u_k^\top \bW v_i)}, \quad\forall j\in[N].
\end{align*}
Let $p \in \Delta^N$ denote the vector of probabilities of each item in the dataset. The population cross-entropy loss is then defined as
\begin{align}
    L(\bW) := \EE{i \sim p}{-\log p_{\bW}(i \mid i)} = - \sum_{i \in [N]}p_i \Bigg(u_i^\top \bW v_i - \log \sum_{j \in [N]} \exp(u_j^\top \bW v_i)\Bigg).
    \label{eq:population-loss}
\end{align}
We assume that $p$ follows a power law, $p_i \sim i^{-\alpha}$ with exponent $\alpha>1$. This condition is motivated by Zipf's law in statistical linguistics, which states that the frequency of a word decays approximately as a power of its rank, indicating that such heavy-tailed structure arises naturally in language~\citep{piantadosi2014zipf}. Such a power law source condition is common in prior analyses of scaling laws~\citep{caponnetto2007optimal,michaud2023quantization,bordelon2024dynamical, lin2024scaling, paquette20244+, ren2025emergence, kunstner2025scaling,li2026muon}.

We will consider optimizing $\bW$ via the minibatch variant of Muon and SGD. Let $B$ be the batch size. A \emph{minibatch} $\cB$ is defined as a collection of tokens $\cB := \{i_1, \dots, i_B\}$, where each token is sampled i.i.d from~$p$. The loss on a minibatch $\cB$ is defined by
\begin{align*}
    L(\bW; \cB) := \frac{1}{B}\sum_{i \in \cB} - \log \hp_{\bW}(i \mid i) = - \sum_{i \in [N]}q_i \Bigg(u_i^\top \bW v_i - \log \sum_{j \in [N]} \exp(u_j^\top \bW v_i) \Bigg),
\end{align*}
where $q_i := \frac{1}{B}\sum_{j \in \cB} \mathbf{1}_{\{i=j\}}$ are the empirical frequencies of each token in the batch $\cB$. The negative gradient at some $\bW$ is thus 
\begin{align}\label{eq:gradient}
    -\nabla_\bW L(\bW;\cB) = \sum_{i\in[N]} q_i\Bigg(u_i - \sum_{j\in[N]} u_j \hat p_\bW(j \mid i)\Bigg)v_i^\top.
\end{align}

\paragraph{Muon.} The Muon optimizer~\citep{jordan2024muon} directly operates on weight matrices. Let $\bG = -\nabla_{\bW}L(\bW; \cB)$ be the negative gradient of the loss (we will omit momentum in our treatment). Denote by $\bG = \bU \bS \bV^\top$ the singular value decomposition (SVD) of $\bG$. The polar map is defined as $\pol(\bG) := \bU\bV^\top$; if $\bG$ is full rank, then $\pol(\bG) = \bG(\bG^\top\bG)^{-1/2}$. The Muon update is
\begin{align*}
    \bW \leftarrow \bW + \eta \cdot \pol(\bG).
\end{align*}
In practice, rather than computing the exact SVD, one instead approximates $\pol(\bG)$ via a constant number of \emph{Newton--Schulz} iterations. Let $\varphi(z)$ be a quadratic or higher-order polynomial. A single Newton--Schulz iteration computes the mapping
\begin{align*}
    \bG\mapsto \bG\varphi(\bG^\top\bG) = \bU \bS\varphi(\bS^2) \bV^\top.
\end{align*}
The output of multiple steps of Newton--Schulz is thus of the form $\bU h(\bS)\bV^\top$, where $h(z)$ is the function obtained by composing $z \mapsto z\varphi(z^2)$ with itself multiple times. $\varphi$ is typically chosen so that $h(z) \approx 1$. This motivates a broad class of \emph{spectral optimizers}: given a function $h: \R_{\ge 0} \rightarrow \R_{\ge 0}$ satisfying $h(0) = 0$, one can define the spectral map $h(\bG) = \bU h(\bS)\bV^\top$ and update the weight matrix as $\bW \leftarrow \bW + \eta h(\bG)$. Within this scheme, gradient descent corresponds to $h(z) = z$, while exact Muon corresponds to $h(z) = \sgn(z)$.

\begin{wrapfigure}{r}{0.22\textwidth}
    \vspace{-1.3\baselineskip}
    \centering
    \includegraphics[width=\linewidth]{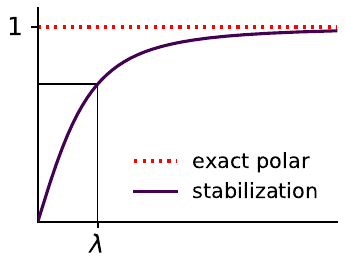}
    \vspace{-1.5\baselineskip}
\end{wrapfigure}

In this work, we will focus on a stabilized approximation to Muon: $h_\lam(z) = \frac{z}{\sqrt{z^2 + \lambda^2}}$ for a hyperparameter $\lambda$ (see right), which as we will see determines the `resolution' of the singular spectrum. The limit $\lam\to 0^+$ recovers the exact polar map. Given schedules $\{\eta_t\}_{t \ge 0}$, $\{\lam_t\}_{t \ge 0}$, the Muon updates $\{\bW_t\}_{t \ge 0}$ are defined by
\begin{align}\label{eq:muon-iterates}
    \bW_{t+1} = \bW_t + \eta_t \bG_t (\bG_t^\top\bG_t + \lam_t^2 \bI_d)^{-1/2}, \quad \bG_t := -\nabla_{\bW}L(\bW_t; \mathcal{B}_t)
\end{align}
initialized at $\bW_0 = 0_{d \times d}$. We also denote the estimated scores by~$\hp_t = \hp_{\bW_t}$.

\begin{remark}
We consider the stabilized approximation $h_\lambda$ primarily for technical convenience. In Section \ref{sec:optimal}, we give a heuristic argument that this update is asymptotically optimal for sufficiently small~$\lam$, and moreover the exact polar map should attain the same rates. A similar smoothed approximation is considered in \citet{jiang2026adaptive}, and, unlike Muon, is shown to converge on non-smooth, non-convex problems. Relatedly, \citet{gonon2026insights} demonstrate problems where using a constant number of Newton-Schulz iterations outperforms the exact polar map.
\end{remark}

\section{One Step of Muon}\label{sec:one_step}

\subsection{One-step recovery of Muon}

%Recall that the weight matrix~$\bW$ recovers item~$i$ iff the diagonal or \emph{signal} logit $u_i^\top\bW v_i$ dominates the off-diagonal or \emph{interaction} logits $u_j^\top\bW v_i$ for all $j\ne i$.
We say that the weight matrix~$\bW\in\RR^{d\times d}$ \emph{recovers} or \emph{stores} item~$i$ if $\argmax_{j\in[N]} \hp(j\mid i)=i$, i.e., the diagonal or \emph{signal} logit ($j=i$) dominates all off-diagonal or \emph{interaction} logits ($j\ne i$):
\begin{align}\label{eq:recovery}
u_i^\top \bW v_i > \max_{j \neq i} u_j^\top \bW v_i.
\end{align}
Our main result sharply characterizes the set of recovered items after one Muon update.

\begin{theorem}[one-step recovery of Muon]\label{thm:main}
Let $u_i,v_i$ for $i\in[N]$ be i.i.d. $\cN(0,\frac1d\bI_d)$ vectors. Let $\bG_0 =-\nabla_{\bW} L(\bW_0;\cB)$ be the negative gradient at initialization of the empirical loss on a minibatch~$\cB$ of size~$B$, or the population loss (equivalently $B=\infty$). Suppose $N=\poly(d)$, $N\gtrsim d^{2\alpha+2}$ and set
\begin{align*}
\lam \asymp \max\left\{\frac{(\log d)^{2\alpha+2}}{d^\alpha}, \frac{(\log d)^2}{B} \right\}.
\end{align*}
Then with high probability, the one-step Muon update $\bW_1^{\Muon} \propto h_\lam(\bG_0)$ recovers all items up to
\begin{align}\label{eq:muon-recovery}
i\lesssim \min\left\{ i^\star, B^\frac{1}{\alpha} (\log d)^{-\frac{1}{\alpha}}\right\}, \quad i^\star \asymp d^{1+\frac{1}{2\alpha}} (\log d)^{-2-\frac{5}{\alpha}}.
\end{align}
\end{theorem}
This bound is tight (up to polylog factors) in the sense that for items $i\gg i^\star$, the signal and interaction terms in Eq.~\eqref{eq:recovery} will be of the same order, so recovery cannot be guaranteed; moreover, items $i\gg B^{1/\alpha}$ have vanishing probability to even be observed in the minibatch~$\cB$, and hence will only be learned sporadically.
From this, we see that the \emph{critical batch size}, beyond which increasing~$B$ no longer yields gains in recovery, is
\begin{align*}
B_{\Muon}^\star = \widetilde{\Theta}((i^\star)^\alpha) = \widetilde{\Theta}(d^{\alpha+\frac12}),
\end{align*}
and this allows us to recover $\widetilde{\Theta}(d^{1+\frac{1}{2\alpha}})$ items. We also note that the condition $N\gtrsim d^{2\alpha+2}$ can be removed by considering correlation loss.

As a corollary, we obtain the following guarantee for the loss decrease after one step.

\begin{corollary}\label{cor:muon-loss}
Taking the learning rate $\eta \asymp (\log d)^{-4}\sqrt{d}$ in the setting of Theorem~\ref{thm:main}, the one-step update $\bW_1^{\Muon} = \eta h_\lam(\bG_0)$ achieves loss
\begin{align}\label{eq:muon-loss}
L(\bW_1^{\Muon}) \le \widetilde{O} \qty(\max\left\{d^{\frac{1}{2\alpha}+\frac12-\alpha}, B^{\frac{1}{\alpha} -1}\right\})
\end{align}
and moreover no item is significantly misclassified, that is, $\hp_1(i\mid i) \ge (1-o(1))\sup_{j\ne i}\hp_1(j\mid i)$ for all $i\in[N]$.
\end{corollary}

The proof of Theorem~\ref{thm:main} is developed throughout Appendices~\ref{sec:a},~\ref{sec:interaction}; a sketch of the main ideas is provided in Section~\ref{sec:sketch}. Here, we make some basic observations. From Eq.~\eqref{eq:gradient}, the gradient of the cross-entropy loss~$L(\bW)$ at initialization is roughly $\bG_0 \approx \sum_i p_iu_iv_i^\top$. If the embeddings~$u_i, v_i$ were orthogonal, Muon would then output $\bW_1^{\Muon} \propto h_\lam(\bG_0) \approx \sum_i u_iv_i^\top$ which already classifies all items correctly; however this constrains the number of items~$N\le d$, far less than the information-theoretic optimum~$\widetilde{\Theta}(d^2)$. In our non-orthogonal setting,~$N$ can be much larger, but we must now account for the correlations between embeddings.

Let us now focus on each signal term $u_i^\top h_\lam(\bG_0) v_i$. The main contribution comes from the aligned rank-one spike $p_iu_iv_i^\top$ in the gradient~$\bG_0$. \new{We quantify this through an \emph{add-back-in} argument: starting from the leave-one-out component $\bG_{-i} = \sum_{j\ne i} p_ju_jv_j^\top$, we analyze how quickly the logit grows as the $u_iv_i^\top$~spike coefficient increases from~$0$ to~$p_i$.} At the same time, a large fraction of the singular values of~$\bG_0$ lie below~$d^{-\alpha}$. The map $h_\lam$ amplifies these by a factor of $\lam^{-1}\sim d^\alpha$, which also boosts the logit growth rate by the same factor, allowing us to recover items with lower frequencies~$p_i$. Thus, $\lam$ effectively acts as a \emph{scale of resolution} for the singular spectrum; the implications of this is discussed in Section~\ref{sec:optimal}.

\subsection{One-step recovery of SGD and Newton}
\label{subsec:SGD-newton}

\paragraph{Stochastic gradient descent.} In contrast with Theorem~\ref{thm:main}, we next prove a tight bound on the number of recovered items for vanilla SGD on the same objective.

\begin{theorem}[one-step recovery of SGD]\label{thm:gd}
In the setting of Theorem~\ref{thm:main} and $N\gtrsim d$, the number of items recovered by the one-step SGD update $\bW_1^{\SGD} = \eta\bG_0$ is
\begin{align*}
\widetilde{\Theta}\qty(\min\left\{d^\frac{1}{2\alpha},B^\frac{1}{\alpha}\right\})
%\min\left\{d^\frac{1}{2a}(\log d)^{1+\frac1a}, B^\frac{1}{a}\right\}
\end{align*}
with high probability. In addition, for any choice of learning rate $\eta$, the loss is lower bounded as
\begin{align*}
L(\bW_1^{\SGD}) \ge\widetilde{\Omega} \qty(\max\left\{ d^{\frac{1}{2\alpha}-\frac12}, B^{\frac{1}{\alpha}-1}\right\}). 
\end{align*}
\end{theorem}

The intuitive reason for this threshold is that an item beyond this has signal $p_i\ll 1/\sqrt{d}$, which is drowned out by the noise of order $p_j/\sqrt{d}$ from the $j\lesssim \log d$ most frequent items, and so is unlikely to have the highest score. 
Thus, the number of items recovered by Muon greatly improves upon that of SGD by a factor of~$d$. Furthermore, the critical batch size for SGD is $B_{\SGD}^\star = \widetilde{\Theta}(\sqrt{d})$, which is much smaller compared to $B_{\Muon}^\star = \widetilde{\Theta}(d^{\alpha+\frac12})$.

\new{

\paragraph{Newton's method.} To put the capacity gain of Muon into perspective, we now show that for the first optimization step, Muon in fact matches the recovery rate of Newton's method for linear associative memory. The Newton update is defined as the direction towards the minimizer of the local quadratic approximation to $L$:
\begin{align*}
\bW\gets\bW + \eta\cdot [\nabla_{\bW}^2 L(\bW;\cB)]^{-1} \bG.
\end{align*}

\begin{theorem}[one-step recovery of Newton]\label{thm:newton}
Denote the Hessian of the loss at initialization as $\cH = \nabla_{\bW}^2 L(\bW_0;\cB)$.
In the setting of Theorem~\ref{thm:main}, if $B=\widetilde{\Omega}(d^\alpha)$ and $\eta=\widetilde{\Theta}(1/\sqrt{d})$, the one-step Newton update $\bW_1^{\Newton}= \eta \cH^{-1}[\bG_0]$ achieves the same recovery rate Eq.~\eqref{eq:muon-recovery} and loss decrease Eq.~\eqref{eq:muon-loss} of Muon, up to log factors.
\end{theorem}
This result is particularly surprising because Newton's method is often considered the gold standard of local second-order or curvature-aware optimization. Many popular preconditioned optimizers -- such as Gauss-Newton, Adagrad~\citep{duchi2011adaptive}, K-FAC~\citep{martens2015optimizing}, and Shampoo~\citep{gupta2018shampoo} -- can be viewed as tractable approximations which avoid the cost of computing and inverting the full Hessian. Theorem~\ref{thm:newton} shows that Muon can match Newton \emph{without accessing any second-order information}, even though the updates are structurally different.

To gain intuition on this comparison, we explicitly write down the Newton update and compare with Muon~(taking $\lam=0$ for simplicity):
\begin{align}\label{eq:explicit}
\bW_1^{\Newton}\propto \Big(\underbrace{\textstyle\frac1N \sum_i u_iu_i^\top - \bar u\bar u^\top}_{=:\bSigma_u}\Big)^{-1}\bG_0 \Big(\underbrace{\textstyle \sum_i q_i v_i v_i^\top}_{=:\bM_v}\Big)^{-1} \quad\text{vs.}\quad \bW_1^{\Muon}\propto \bG_0(\bG_0^\top \bG_0)^{-1/2}.
\end{align}
The Hessian admits a Kronecker factorization $\cH = \bM_v \otimes \bSigma_u$ at initialization, thus the Newton update is equivalent to the K-FAC update with preconditioning on both sides, while Muon is only preconditioned on the right. In particular, the left preconditioner~$\bSigma_u$ is the (unweighted) empirical covariance matrix, which has the effect of whitening the unembedding vectors~$u_i$. We note that the batch size condition $B\gtrsim d^\alpha$ is necessary for the Hessian to be well-conditioned (Lemma~\ref{lem:proportional}). Regarding the proof, we analyze the Newton step by applying the add-back-in argument to each spike in~$\bG_0,\bM_v$. Compared to Muon, the analysis is relatively straightforward as we can directly apply the Sherman--Morrison formula to~$\bM_v^{-1}$ to compute the change in logits.

\paragraph{Handling anisotropic embeddings.} While Muon matches Newton's method in our isotropic Gaussian setting, an important limitation appears for anisotropic data. Suppose the unembedding and embedding vectors follow $u_i \sim \cN(0,\bXi_u)$ and $v_i \sim \cN(0,\bXi_v)$ with general covariance matrices $\bXi_u,\bXi_v$. It is apparent from Eq.~\eqref{eq:explicit} that the logits $u_j^\top\bW_1^{\Newton}v_i$ (and thus the estimated likelihoods) of the Newton update are invariant under the transformations $u_i\mapsto \bXi_u^{-1/2}u_i$ and $v_i\mapsto \bXi_v^{-1/2}v_i$, and therefore achieve the same recovery rate as in Theorem~\ref{thm:newton}. By contrast, the polar map is not invariant under these transformations, so Muon cannot in general be expected to retain the same rate. This is consistent with the experiments in Figure~\ref{fig:newton-population}: when $u_i$ and $v_i$ have identity covariance, Muon achieves capacity comparable to Newton's method, but the gap widens as the unembedding vectors become more anisotropic.
% \jk{TODO experiments // need to fix figures vspacing}
}

\subsection{Proof sketch of Theorem~\ref{thm:main}}\label{sec:sketch}

%Before we begin, we encourage the casual reader to still skim the first part of the sketch, which contains the essence of the analysis and will be used in Section~\ref{sec:optimal} to establish optimality. The second part is more technical and can be skipped on a first reading.

To start, we approximate the gradient as a sum of independent rank-one terms:
\begin{align*}
-\nabla_{\bW} L(\bW_0;\cB) = \sum_{i\in[N]} q_i(u_i - \bar{u}) v_i^\top \approx \sum_{i\in[N]} q_iu_i v_i^\top =: \bG
\end{align*}
and define the logits $\gamma_{ij} := u_j^\top h_\lam(\bG)v_i$ (omitting $\eta$). We analyze the signal and interaction terms separately.

\paragraph{Lower bounding signal logits (Appendix~\ref{sec:a}).}
Denote the leave-one-out gradient $\bG_{-i} := \bG -q_iu_iv_i^\top$, so that~$u_i^\top h_\lam(\bG_{-i})v_i$ is random with size $\widetilde{O}(1/\sqrt{d})$. We study how the $i$th logit behaves as we gradually add the $u_iv_i^\top$ term back in via the auxiliary function
\begin{align}\label{eq:signal-strength}
\phi(q) = u_i^\top h_\lam(\bG_{-i} + qu_iv_i^\top)v_i, \quad q\ge 0.
\end{align}
By definition, $\gamma_{ij} = \phi(q_i)$, which we analyze via Taylor expansion. The slope $\phi'(0)$ can be computed explicitly via the Daleckii--Krein formula (Proposition~\ref{prop:daleckii}), which computes the Fr\'{e}chet derivative of a matrix function along a specified perturbation direction. Let $\bG_{-i} = \bA\bS\bB^\top$ be the SVD with singular values $\bS = \diag(s_1,\cdots,s_d)$ in decreasing order and $a=\bA^\top u_i$, $b=\bB^\top v_i$. Then $\phi'(0)$ is given as the sum of nonnegative terms
\begin{align*}
\frac14 \sum_{k \neq \ell} \frac{h_\lam(s_k) + h_\lam(s_\ell)}{s_k + s_\ell} (a_kb_\ell - a_\ell b_k)^2 + \frac{h_\lam(s_k) - h_\lam(s_\ell)}{s_k - s_\ell} (a_kb_\ell + a_\ell b_k)^2 + \sum_k h_\lam'(s_k) a_k^2 b_k^2.
\end{align*}
We focus on the first term. As $a,b$ are i.i.d. Gaussian conditioned on~$\bG_{-i}$, $(a_kb_\ell-a_\ell b_k)^2 \approx 1/d^2$. Since $z\mapsto h_\lam(z)/z$ is decreasing, the sum is dominated by small singular values. We then show that the `bulk' singular value $s_{d/2} = \widetilde{O}(d^{-\alpha})$ with high probability~(Lemmas \ref{lem:spectrum-new}, \ref{lem:symmetrize}). Hence choosing $\lam$ at the same scale,
\begin{align}\label{eq:vanda}
\phi'(0) \gtrsim \frac{h_\lam(s_{d/2})}{s_{d/2}} = \frac{1}{\sqrt{\smash[b]{s_{d/2}^2}+\lam^2}} \gtrsim \lam^{-1}.
\end{align}
Moreover, we show that $\phi$ is nondecreasing and $|\phi''(0)| \lesssim \lam^{-2}$. Taylor expanding $\phi$ around zero and optimizing the radius yields the lower bound
\begin{align*}
\gamma_{ii} = \phi(q_i) \gtrsim \min\left\{\frac{q_i}{\lam}, 1\right\} - \sqrt{\frac{\log d}{d}}.
\end{align*}
\paragraph{Upper bounding interaction logits (Appendix~\ref{sec:interaction}).} The interaction logits turn out to be much more challenging to control. A naive approach is to use that $h_\lam$ is $\lam^{-1}$-Lipschitz w.r.t. operator norm (Proposition~\ref{prop:operator-lipschitz}), so $(u_j,v_j) \mapsto\gamma_{ij} = u_j^\top h_\lam(\bG_{-j} + q_ju_jv_j^\top) v_i$ is a Lipschitz mapping of Gaussians and so exhibits good concentration. This argument works when either $q_i$ or $q_j\ll\lam$, but fails to bound `large' interactions where both $i,j\le r$ for a threshold $r\approx d$. For these terms, we first invoke a block resolvent integral representation amenable to series expansion, then develop a nonasymptotic perturbative analysis reminiscent of moment methods in random matrix theory. A technical overview is provided in Appendix~\ref{sec:overview} for the interested reader. In the end, we show:
\begin{align*}
|\gamma_{ij}| \lesssim \frac{(\log d)^3}{\sqrt{d}}\quad \forall j\ne i.
\end{align*}
We have thus proved that $\gamma_{ii} > \max_{j\ne i}\gamma_{ij}$ if $q_i/\lam \gtrsim 1/\sqrt{d}$ (ignoring log factors). In the population regime, from $p_i\asymp i^{-\alpha}$ we conclude that all items $i \lesssim d^{1+\frac{1}{2\alpha}}$ are recovered w.h.p. In the minibatch setting, we incur an additional information-theoretic threshold: items $i\gg B^{1/\alpha}$, that is $p_i\ll 1/B$, are unlikely to be observed in the minibatch at all, and thus will not be learned.

%Our analysis can be extended to polar maps $h$ which have the following properties: (1) $h$ is monotone increasing, sufficiently smooth and $h'(0) \approx \lam^{-1}$ where $\lam^{-1}$ is the desired scale of resolution, $\widetilde{O}(d^{\alpha})$ for the first step; (2) $h(z)/z$ is monotone decreasing and does not decay faster than $1/z$, and further admits certain resolvent expansions. This class includes local approximate polynomials which do not necessarily `saturate' at $1$, thus our analysis may also be extended to certain Newton--Schulz schemes.

\section{Optimality and Convergence Rate of Muon}\label{sec:new}

\subsection{Optimality of Muon}\label{sec:optimal}

A natural follow-up question to Theorem~\ref{thm:main} is: is the~$d^{1+\frac{1}{2\alpha}}$ one-step recovery rate optimal among first-order methods for the linear associative memory task, or can it be improved by choosing a different estimator $h(\bG_0)$ of~$\bW$? For example, our analysis used $\lam\sim d^{-\alpha}$ while the limit $\lam\to 0$ recovers the exact polar map; can choosing a sharper resolution improve our results?

In this section, we give a negative answer to this question by providing a heuristic argument for the one-step optimality of our stabilized variant of Muon; \new{this intuitively aligns with the fact that Muon already matches the recovery rate of Newton's method (Theorem~\ref{thm:newton}).} We first show that any Bayes optimal gradient-based estimator must be spectrally equivariant, that is for the SVD of the gradient $\bG_0=\bU\bS\bV^\top$, it holds that $h(\bG_0) = \bU h(\bS)\bV^\top$ and $h(\bS)$ is diagonal.

\begin{proposition}\label{prop:hunt-stein}
Let $\Spec(d)$ denote the set of bi-orthogonally equivariant measurable maps $h:\RR^{d\times d}\to \RR^{d\times d}$ such that $\sup\nnorm{h}_{\F}<\infty$, that is, $h(\bU\bX\bV^\top)=\bU h(\bX)\bV^\top$ for all $\bX\in\RR^{d\times d}$ and $\bU,\bV\in O(d)$. %Then $\Spec(d)$ is precisely the set of bounded spectral optimizers, and
The Bayes optimal update rule w.r.t. $L$ is in $\Spec(d)$, that is, for the Bayes risk $\cR(h) := \E_{(u_i,v_i)_{i\in[N]},\cB}[L(h(\bG_0))]$ it holds that
\begin{align*}
\inf_{h:\sup\nnorm{h}_{\F}<\infty} \cR(h) = \inf_{h\in\Spec(d)} \cR(h).
\end{align*}
\end{proposition}

This is essentially a corollary of the Hunt--Stein theorem on minimax tests of invariant statistical problems. Any bi-orthogonal conjugate $h^{\bU,\bV}(\bX):= \bU^\top h(\bU\bX\bV^\top)\bV$ of~$h$ will have the same Bayes risk due to rotation invariance. Then the estimator $\bar h\in\Spec(d)$ constructed by averaging~$h^{\bU,\bV}$ over Haar measure $\bU,\bV\sim O(d)\times O(d)$ satisfies $\cR(\bar h)\le\cR(h)$ due to convexity of~$L$.

We remark that $h\in\Spec(d)$ does not preclude nonseparable maps where each diagonal entry $h(\bS)_{ii}$ can depend on the entire spectrum $\bS$. Nonetheless, such maps are in general difficult to implement as they require computing the full SVD, which Muon (with Newton--Schulz iterations) is designed to avoid. Thus, we restrict our attention to separable maps $h(\bS) = \diag (h(s_i))$ for a scalar-valued function $h$. Since the inputs to~$h$ are bounded by $\nnorm{\bG_0}_\op = O(1)$ w.h.p., we can always rescale~$h$ to have bounded outputs. We also assume a mild monotonicity property:

\begin{assumption}\label{assp:h}
$h:\RR_{\ge 0}\to [0,1]$ is~$C^1$ and $h(z)/z$ is nonincreasing.
\end{assumption}

\begin{wrapfigure}{r}{0.25\textwidth}
    \vspace{-1.2\baselineskip}
    \centering
    \includegraphics[width=\linewidth]{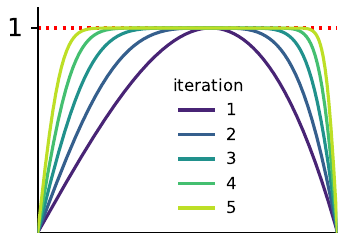}
    \vspace{-1.7\baselineskip}
\end{wrapfigure}
Intuitively, this means that smaller singular values are blown up by a larger multiplicative factor. Note that we do not require~$h$ itself to be monotonic. For example, the classical cubic Newton--Schulz iteration $h(z)=\frac32 z-\frac12 z^3$ \citep{bernstein2024newtonschulz} satisfies this assumption, as well as its higher-order iterates on the interval of convergence (see right figure) --- see Appendix~\ref{sec:newton} for details.

We now show that any~$h$ satisfying Assumption~\ref{assp:h} \emph{cannot} improve the signal~$\gamma_{ii}$. As in Eq.~\eqref{eq:signal-strength}, we compute the strength of the signal via a first-order approximation of the auxiliary map~$\phi$, and the slope is given via the singular values of the leave-one-out gradient~$\bG_{-i} = \sum_{j\ne i} q_i u_iv_i^\top$ as
\begin{align*}
\phi'(0) &\asymp \frac{1}{d^2} \sum_{k,\ell} \frac{h(s_k) + h(s_\ell)}{s_k + s_\ell} + \frac{h(s_k) - h(s_\ell)}{s_k - s_\ell} + \frac{1}{d^2} \sum_k h'(s_k) \lesssim \frac1d \sum_k \frac{h(s_k)}{s_k},
\end{align*}
where the inequality follows from $\frac{h(s_k) + h(s_\ell)}{s_k + s_\ell} \le \frac{h(s_k)}{s_k}+\frac{h(s_\ell)}{s_\ell}$, $\frac{h(s_k) - h(s_\ell)}{s_k - s_\ell} \le \min\{\frac{h(s_k)}{s_k}, \frac{h(s_\ell)}{s_\ell}\}$ and $h'(s_k) \le \frac{h(s_k)}{s_k}$ under Assumption~\ref{assp:h}. That is, the signal strength is roughly determined by \emph{how much the average singular value is blown up} by~$h$. If~$h$ is Lipschitz, this is uniformly bounded by~$\nnorm{h}_{\Lip}$ (thus Eq.~\eqref{eq:vanda} is tight for $h_\lam$). However, even without Lipschitz control, this is fundamentally limited by the average scale of the singular spectrum as we prove below.

\begin{lemma}\label{lem:sval}
Let $s_1\ge\cdots\ge s_d$ be the singular values of the leave-one-out gradient $\bG_{-i}$. It holds w.h.p. that $s_d\gtrsim d^{-\alpha-1}(\log d)^{-1}$ and $\sum_{k=1}^d s_k^{-1} \lesssim d^{\alpha+1}(\log d)^2$.
\end{lemma}

As such, we must have $\phi'(0) \lesssim d^{\alpha}$ regardless of the choice of $h$, therefore the signal $\gamma_{ii} = \phi(q_i)$ is upper bounded (ignoring higher-order terms) as $\phi(0) + \phi'(0)q_i \lesssim \frac{1}{\sqrt{d}} + q_i d^\alpha$. In contrast, the noise and interaction terms~$\phi(0)$ and $\gamma_{ij}$ are generally of size $\widetilde{\Theta}(1/\sqrt{d})$. As a consequence, we indeed require $q_i\gg d^{-\alpha-\frac12}$, equivalently $i\ll d^{1+\frac{1}{2\alpha}}$ in the population regime to ensure recovery, matching the rate obtained in Theorem~\ref{thm:main}.%\footnote{To make this argument fully rigorous, one would need to show an upper bound for $\phi$, which requires additional smoothness assumptions on~$h$, and also show an anti-concentration argument for the interaction terms, which seems out of reach with current techniques.}

For example, Lemma~\ref{lem:sval} implies that taking the resolution as $\lam\ll d^{-\alpha-1}$ instead of $\lam\sim d^{-\alpha}$ in Theorem~\ref{thm:main} essentially gives the exact polar map $h(z)\approx\sgn(z)$, as this scale will never be `seen' by the singular values. Nonetheless, even in this regime the average blowup of the singular values is of order~$d^\alpha$, and so we will not see any improvement from using the polar map. In fact, from this argument we expect roughly the same recovery rate, which is indeed what we observe in our experiments in Section~\ref{sec:experiments}.

\subsection{Multiple steps of Muon}\label{sec:multi_step}

We now turn our attention to the entire update trajectory of Muon and SGD. To study the macroscopic scaling behavior of these processes, we will adopt a simplifying heuristic. %We say an item~$i$ is \emph{strongly recovered} if $\hp_{\bW}(i\mid i)>1-d^{-\omega(1)}$.
At step~$t$, suppose all items $i=1,\cdots,d_t$ have been recovered with $\hp_t(i\mid i)\approx 1$. We presume all items $i>d_t$ have not been recovered at all, i.e. $\hp_t(i\mid i)\ll 1$, and approximate the current gradient as
\begin{align}\label{eq:heuristic}
\bG_t \approx \sum_{i\in[N]} q_i^{(t)} (1-\hp_t(i\mid i))u_i v_i^\top \approx \sum_{i > d_t} q_i^{(t)} u_i v_i^\top =: \bbG_t
\end{align}
where~$q^{(t)}$ is the frequency vector of the~$t$th minibatch. This can be viewed as a \emph{deflation} process: starting from all items $\bG_0\approx \sum_{i\in[N]} q_i u_i v_i^\top$, the already-recovered items are continually removed from the gradient after each update. Under this simplification, we recursively derive the \emph{recovery threshold} $\{d_t\}_{t\ge 1}$ after each Muon update, which yields the following sharp scaling law.

\begin{theorem}[multi-step recovery of Muon]\label{thm:multi}
Let $d_0=0$, $T\in\NN$ and
\begin{align*}
d_t =\widetilde{\Theta}\qty(\min\{d^{2-(1-\frac{1}{2\alpha})^t}, B^\frac{1}{\alpha}\}), \quad \lam_t = \widetilde{\Theta}\qty(d_{t+1}^{-\alpha} \sqrt{d}), \quad \eta\asymp (\log d)^{-4}\sqrt{d}.
\end{align*}
Then for sufficiently large $d$, the iterates~$\{\bW_t\}_{t\ge 0}$ defined as $\bW_0=0$, $\bW_{t+1} = \bW_t + \eta h_{\lam_t}(\bbG_t)$ recover all items $i = 1,\cdots,d_t$ at all steps $t\le T$, and moreover $L(\bW_t) \le \widetilde{O}(d_t^{1-\alpha})$.
\end{theorem}

Thus for sufficiently large~$B$, the recovery exponent $2-(1-\frac{1}{2\alpha})^t$ converges exponentially to the information-theoretic maximum~$2$ with a fixed learning rate. In comparison, for multiple steps of SGD, we show a strictly suboptimal scaling law for \emph{any} learning rate schedule.

\begin{theorem}[multi-step recovery of SGD]\label{thm:gd-multi}
Let $T\in\NN$ and $\{\eta_t\}_{t\ge 0}$ be any learning rate schedule. Suppose the SGD iterates~$\{\bW_t\}_{t\ge 0}$ defined as  $\bW_0=0$, $\bW_{t+1} = \bW_t + \eta_t\bbG_t$ recover all items $i = 1,\cdots,d_t$ with constant probability at all steps $t\le T$. Then it must hold that 
\begin{align*}
d_{t+1} \lesssim \begin{cases}
\min\{d^\frac{1}{2\alpha}d_t, B^\frac{1}{\alpha}\} & d_t\lesssim d, \\
\min\{ d^\frac{1}{\alpha} d_t^{1-\frac{1}{2\alpha}}, B^\frac{1}{\alpha}\} & d_t \gtrsim d.
\end{cases}
\end{align*}
Moreover, this rate is achieved (up to polylog factors) by taking $\eta_t = \widetilde{\Theta}(d_{t+1}^\alpha)$.
\end{theorem}

In words, for the first $\lceil 2\alpha\rceil$ steps, the recovery exponent of SGD increases linearly until $d_t\gtrsim d$; note that Muon already achieves this with a single update. After this point, however, the improvement recursion $d_{t+1}\sim d^\frac{1}{\alpha} d_t^{1-\frac{1}{2\alpha}}$ matches that of Muon. Hence Muon accelerates recovery earlier in training, but the convergence behavior for large $t$ is comparable to that of SGD. This aligns with the empirical observations in \citep{semenov2025benchmarking}, where preconditioned optimizers such as Muon and SOAP~\citep{vyas2024soap} are found to outperform AdamW on shorter runs, but the gap narrows over longer horizons.

We give a brief intuition for this phenomenon. In the non-orthogonal setting, each item $i>d_t$ to be recovered must compete with noise from the top individual unclassified items $j\sim d_t$ with large frequencies, as well as the aggregate fluctuation from the bulk of the unclassified items, which leads to the two thresholds in Theorem~\ref{thm:gd-multi}. Hence, Muon can be interpreted as effectively removing the first threshold by amplifying the bulk (but not top) singular directions. Once $d_t>d$, the gradient becomes relatively more isotropic and so the effect of orthogonalization is less pronounced; the second threshold becomes the limiting factor for both optimizers.

\begin{remark}
The rates in Theorem~\ref{thm:multi} and Theorem~\ref{thm:gd-multi} are given for a constant (or at most logarithmically diverging) horizon $T$, and hence can achieve $\Omega(d^{2-\eps})$ recovery rate for any $\eps>0$. As in Section~\ref{sec:optimal}, this scaling is likely near-optimal in this regime. However when $T$ grows further, we expect that the batch size dependency should more accurately scale as $(TB)^{1/\alpha}$, as this is the information-theoretic upper bound on the set of all observed items after $T$ steps.

We also emphasize that the approximation in Eq.~\eqref{eq:heuristic} is heuristic. Under the exact Muon dynamics after~$t$ steps, items with indices $i \ge \bar d_t := d_t \polylog(d)$ have nearly uniform logits, but for items in the intermediate range $d_t \le i \le \bar d_t$, the predicted scores $\hp_t(i\mid i)$ can take any value between~$\frac1N$ and~$1$. These scores also depend in a complicated way on all embeddings $\{u_j,v_j\}_{j\in[N]}$, preventing a direct extension of the proof of Theorem~\ref{thm:main}. One way to bypass this is to assume gaps in the power-law spectrum as in \citet{li2026muon}, but we do not take this route. Instead, we leave the precise end-to-end guarantee as a conjecture below and empirically validate predictions of Theorem~\ref{thm:multi} in Figure~\ref{fig:multi-step-population} (on the exact Muon iterates). From Theorem~\ref{thm:newton}, we also conjecture that Muon continues to match Newton's method throughout training, suggesting an intrinsic curvature-aware property.
\end{remark}

\begin{conjecture}
The recovery and convergence rates of Theorem~\ref{thm:multi} also hold for the exact Muon iterates $\bW_{t+1} = \bW_t + \eta h_{\lam_t}(\bG_t)$. \new{Moreover, Muon matches Newton's method throughout training.}
\end{conjecture}

\section{Experiments}\label{sec:experiments}

\subsection{Linear associative memory}

We quantify the benefits of Muon over SGD in synthetic experimental settings. First, we consider the linear associative memory model introduced in Section~\ref{sec:setting}. 
For convenience, in the Muon update we keep the regularization hyperparameter fixed to $\lambda_t\equiv 0$ and compute the exact polar update. 

\paragraph{First gradient step.} Figure~\ref{fig:one-step-population} shows the storage capacity scaling of the memory matrix $\bW$ after a single population Muon or GD step. We fix the vocabulary size at $N=100{,}000$ and vary both the power law exponent~$\alpha$ and the embedding dimension~$d$. Muon (Figure~\ref{fig:muon-population}) indeed achieves a dramatically larger storage capacity than GD (Figure~\ref{fig:GD-population}). Moreover, the fitted scaling exponents (bottom right) agree with our theoretical predictions in the population limit: Muon stores $d^{1+\frac{1}{2\alpha}}$ items (Theorem~\ref{thm:main}), whereas GD stores only $d^{\frac{1}{2\alpha}}$ items (Theorem~\ref{thm:gd}).

In Figure~\ref{fig:one-step-batch}, we study the empirical loss or minibatch setting to probe the critical batch size. We fix $N=100{,}000$ and $\alpha=1.5$, and vary the batch size $B$. At small batch sizes (up to roughly $B \approx 100$), both Muon and GD are bottlenecked by the information-theoretic rate $B^{\frac{1}{\alpha}}$. As the batch size increases, however, the capacity of SGD quickly plateaus, whereas Muon continues to benefit from larger batches. This is consistent with the empirical observation that Muon’s computational gains are accompanied by a much larger critical batch size \citep{wen2025fantastic}.

\begin{figure}[!t] 
%\vspace{-2.5mm} 
\centering
\begin{subfigure}[t]{0.49\linewidth}
\centering
{\includegraphics[height=0.72\textwidth]{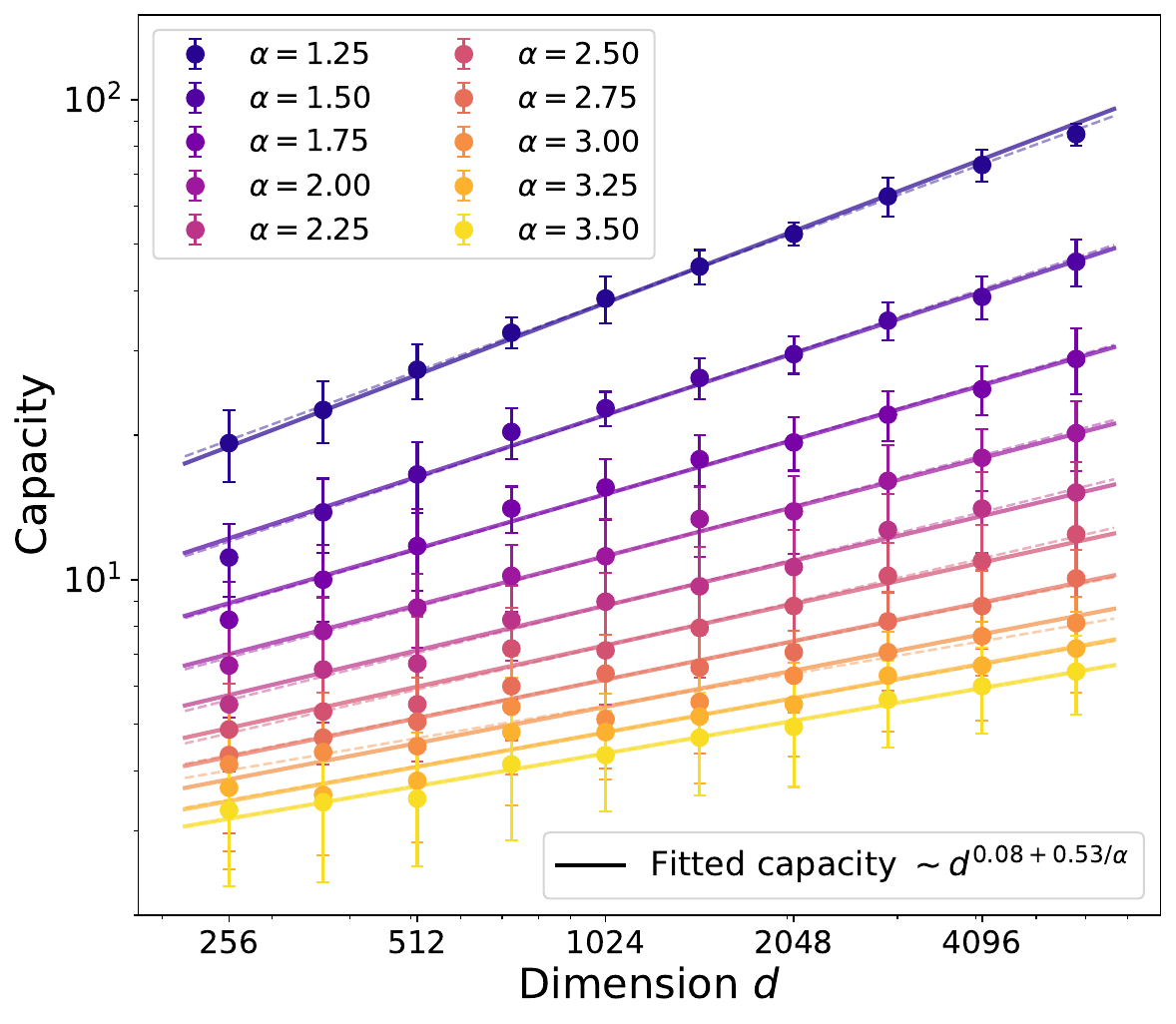}} 
%\vspace{-1mm}
\caption{Gradient descent (population).}
\label{fig:GD-population}
\end{subfigure}%
\begin{subfigure}[t]{0.49\linewidth}
\centering 
{\includegraphics[height=0.72\textwidth]{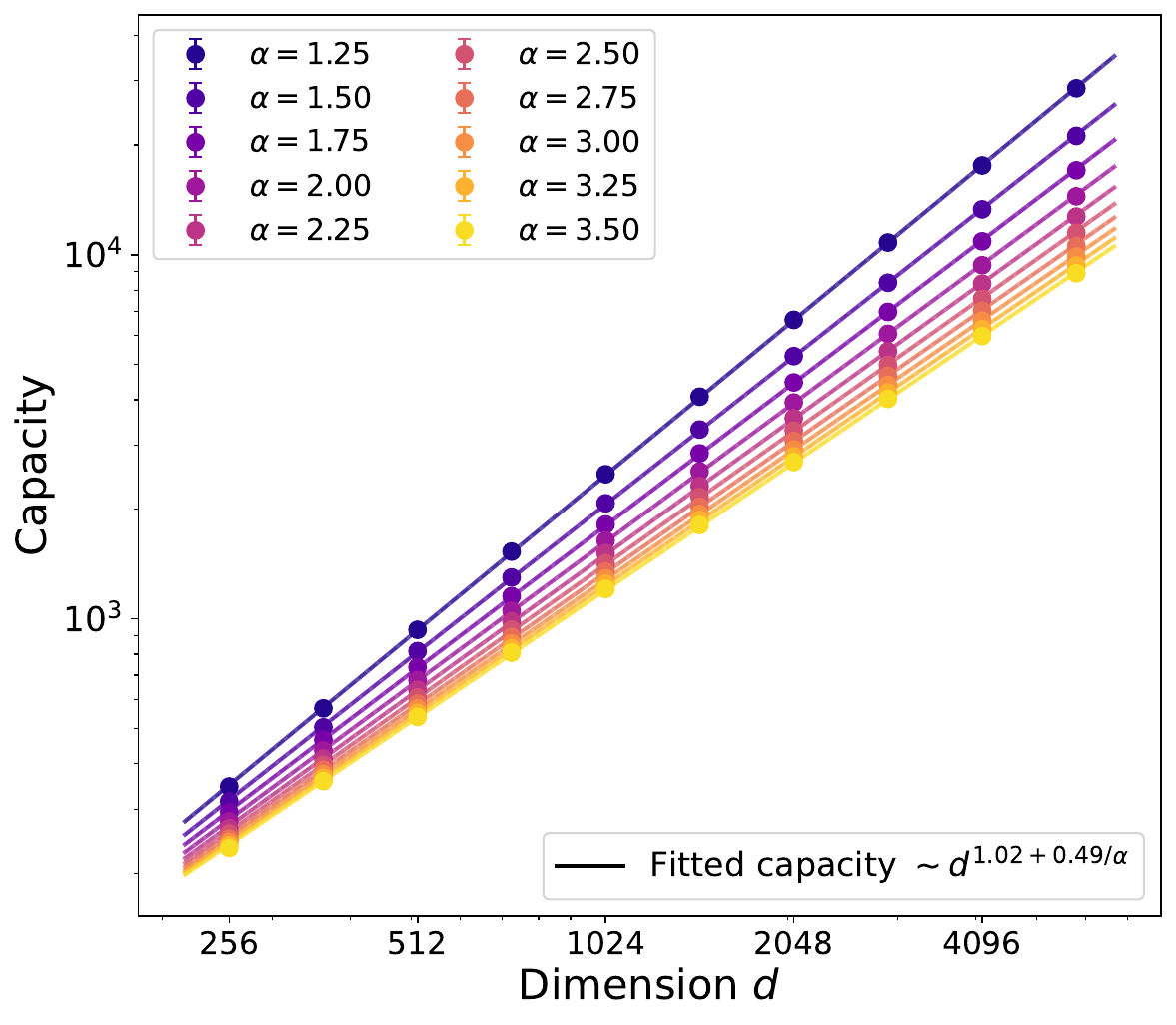}} 
%\vspace{-1mm}
\caption{\small Muon (population).} 
\label{fig:muon-population}
\end{subfigure}%
%\vspace{-1mm} 
\caption{\small Capacity scaling after one population Muon and GD step. We set $N=100,000$ and vary $d,\alpha$. Each experiment is repeated 16 times. For each $\alpha$, we fit the dimension exponents of the mean capacity $d^{C_\alpha}$ (dashed lines), and then find the best fit of exponents $C_\alpha$ in the form of $C_\alpha = c_1 + \frac{c_2}{\alpha}$ (solid lines). Observe that Muon achieves much higher storage than GD, and the exponents are consistent with Theorems~\ref{thm:main}, \ref{thm:gd}. 
}  
%\vspace{-2mm}
\label{fig:one-step-population} 
\end{figure}

\begin{figure}[!t]
% \vspace{-2mm} 
\centering
\begin{subfigure}[t]{0.49\linewidth}
\centering
{\includegraphics[height=0.72\textwidth]{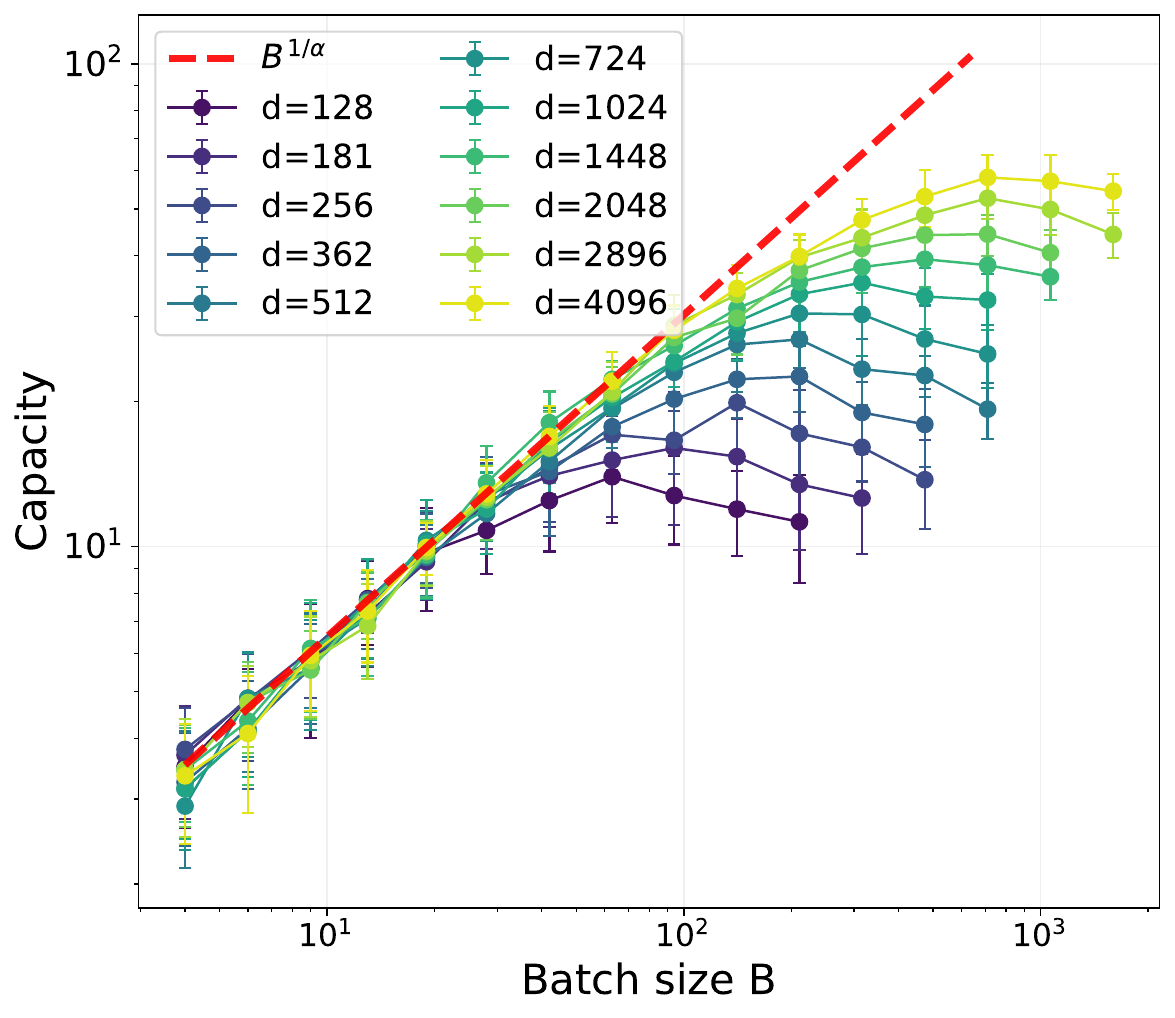}} 
%\vspace{-1mm}
\caption{Gradient descent (minibatch).}
\label{fig:GD-batch}
\end{subfigure}%
\begin{subfigure}[t]{0.49\linewidth}
\centering 
{\includegraphics[height=0.72\textwidth]{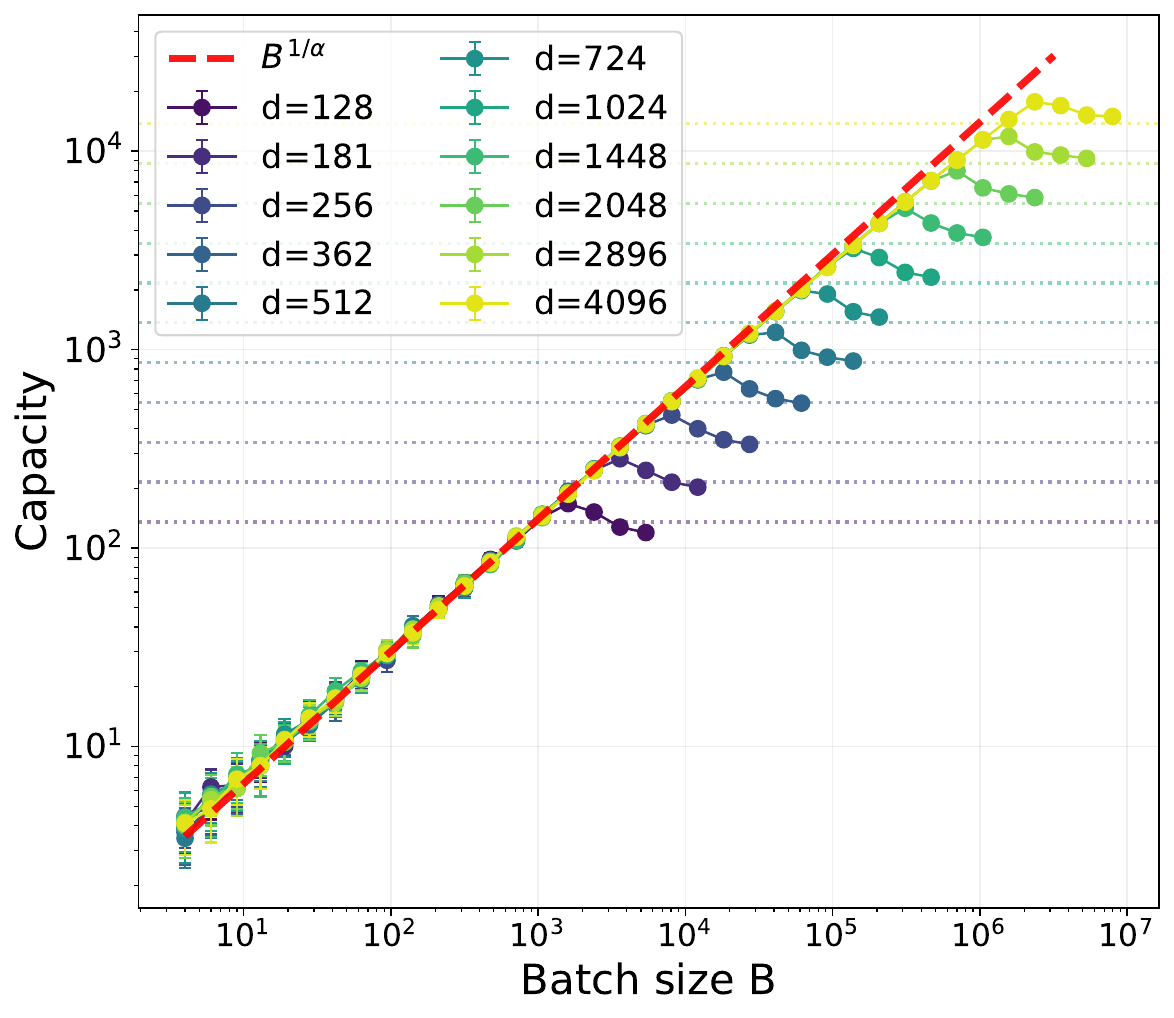}} 
%\vspace{-1mm}
\caption{\small Muon (minibatch).} 
\label{fig:muon-batch}
\end{subfigure}%
%\vspace{-1mm} 
\caption{\small Capacity scaling after one Muon and SGD step on empirical loss. We set $N=100{,}000, \alpha=1.5$, and vary the minibatch size $B$. Each experiment is repeated 16 times. The dashed red line indicates the information-theoretic rate, and the horizontal dashed lines in Figure~\ref{fig:muon-batch} correspond to the $d^{1+\frac{1}{2\alpha}}$ ceiling; the predicted critical batch sizes are given by their intersections. Observe that Muon offers capacity gain over SGD only at sufficiently large~$B$, and the empirical critical batch sizes match well with our predictions.}
% \vspace{-0.5mm}
\label{fig:one-step-batch} 
\end{figure}  

\paragraph{Multiple gradient steps.} In Figure~\ref{fig:multi-step-population}, we examine the multi-step capacity scaling of population Muon to test the predictions of Theorem~\ref{thm:multi} (which assumes the deflation heuristic). We fix $N=250{,}000$, vary $d$ and $\alpha$, and run Muon with $\lambda_t=0$ on the population cross-entropy objective for $T$ steps using a fixed learning rate $\eta \asymp \sqrt{d}$. After each step, we measure the storage capacity and fit a power law to extract its scaling exponent in $d$. As shown in Figures~\ref{fig:step-2}, \ref{fig:step-3}, \ref{fig:step-4}, the capacity increases with the number of training steps. Moreover, Figure~\ref{fig:step-convergence} shows that, after sufficiently many steps, the weight matrix approaches the optimal $\widetilde{\Theta}(d^2)$ capacity \citep{nichani2024understanding}. Figure~\ref{fig:multi-step-exponent} compares the fitted capacity exponents for all $(T,\alpha)$ pairs against the predictions of Theorem~\ref{thm:multi}. We find good overall agreement, with larger deviations at smaller $\alpha$ and larger $T$ where non-asymptotic effects are expected to be more pronounced. Overall, these results suggest that the heuristic approximation we introduced in Eq.~\eqref{eq:heuristic} captures the scaling behavior of the training dynamics reasonably well.

% In Figure~\ref{fig:multi-step-2}, we further compare the multi-step capacity scaling of GD and Muon with $N=100{,}000$ and $\alpha=1.5$. Figure~\ref{fig:GD-vs-muon-multistep} compares their performance in minimizing the population cross-entropy loss. Muon attains much higher capacity than GD in the first few steps. On the other hand, with the increasing learning rate schedule from Theorem~\ref{thm:gd-multi}, GD catches up later in training, and both methods eventually reach the optimal $d^2$ capacity (we however note that this increasing~$\eta_t$ schedule for GD is numerically unstable, especially when~$\alpha$ is large). These results suggest that Muon’s acceleration is most significant early in training, matching the predictions in Theorems~\ref{thm:multi} and~\ref{thm:gd-multi}. 
% Finally, Figure~\ref{fig:multistep-batch} shows minibatch Muon run for up to $T=20$ steps with batch sizes $B \in \{2^{11},\dots,2^{16}\}$. For batch sizes below the critical threshold, the capacity stays close to the information-theoretic limit $(BT)^{1/\alpha}$, i.e., the number of items that can be observed after $T$ steps with batch size $B$. For larger batch sizes (darker curves), the capacity saturates and sample efficiency worsens.

\begin{figure}[t]
%\vspace{-2mm} 
\centering
\begin{subfigure}[t]{0.325\linewidth}
\centering
{\includegraphics[height=0.85\textwidth]{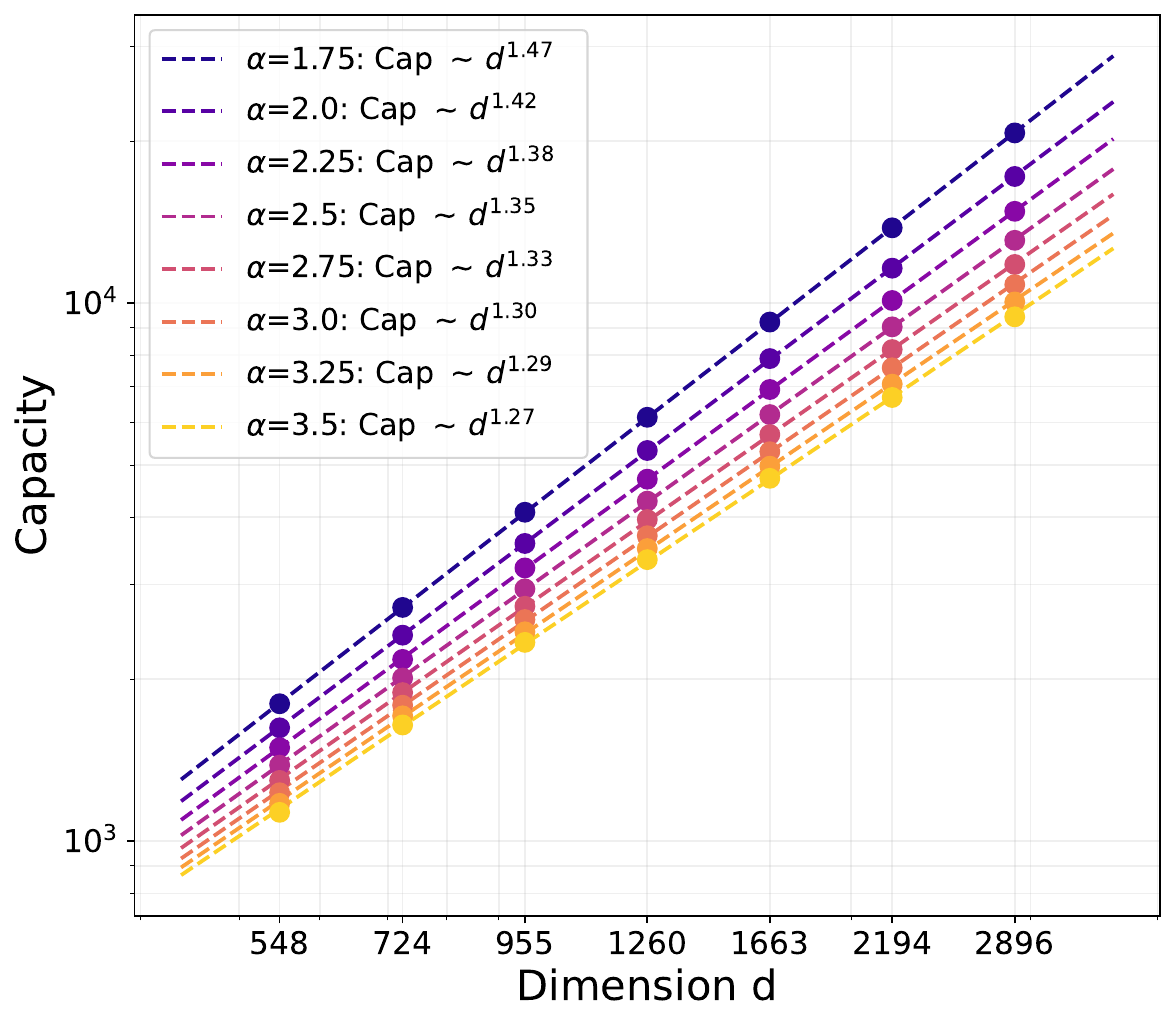}} 
%\vspace{-1mm}
\caption{Capacity at $T=2$.}
\label{fig:step-2}
\end{subfigure}%
\begin{subfigure}[t]{0.325\linewidth}
\centering
{\includegraphics[height=0.85\textwidth]{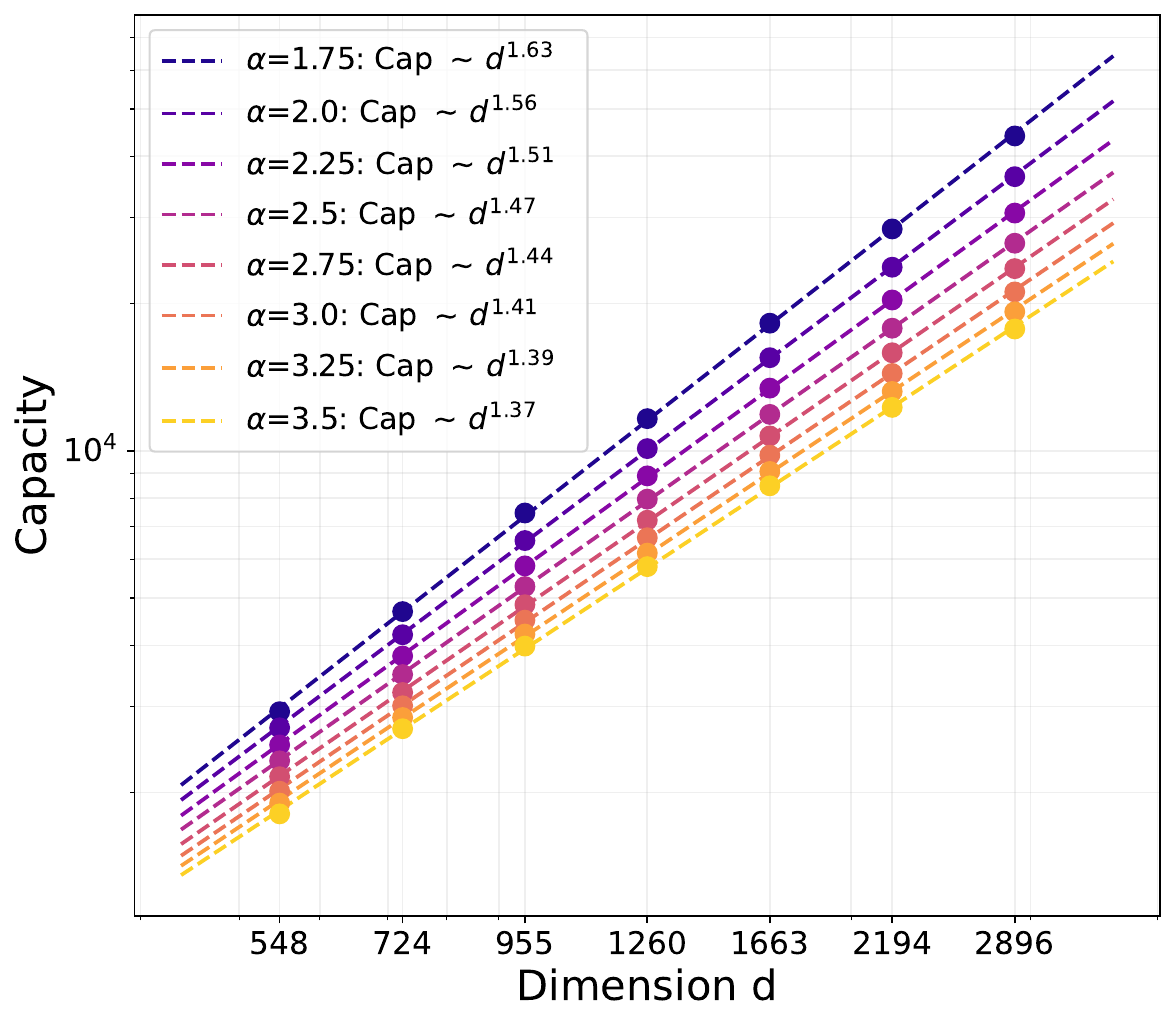}} 
%\vspace{-1mm}
\caption{Capacity at $T=3$.}
\label{fig:step-3}
\end{subfigure}%
\begin{subfigure}[t]{0.325\linewidth}
\centering
{\includegraphics[height=0.85\textwidth]{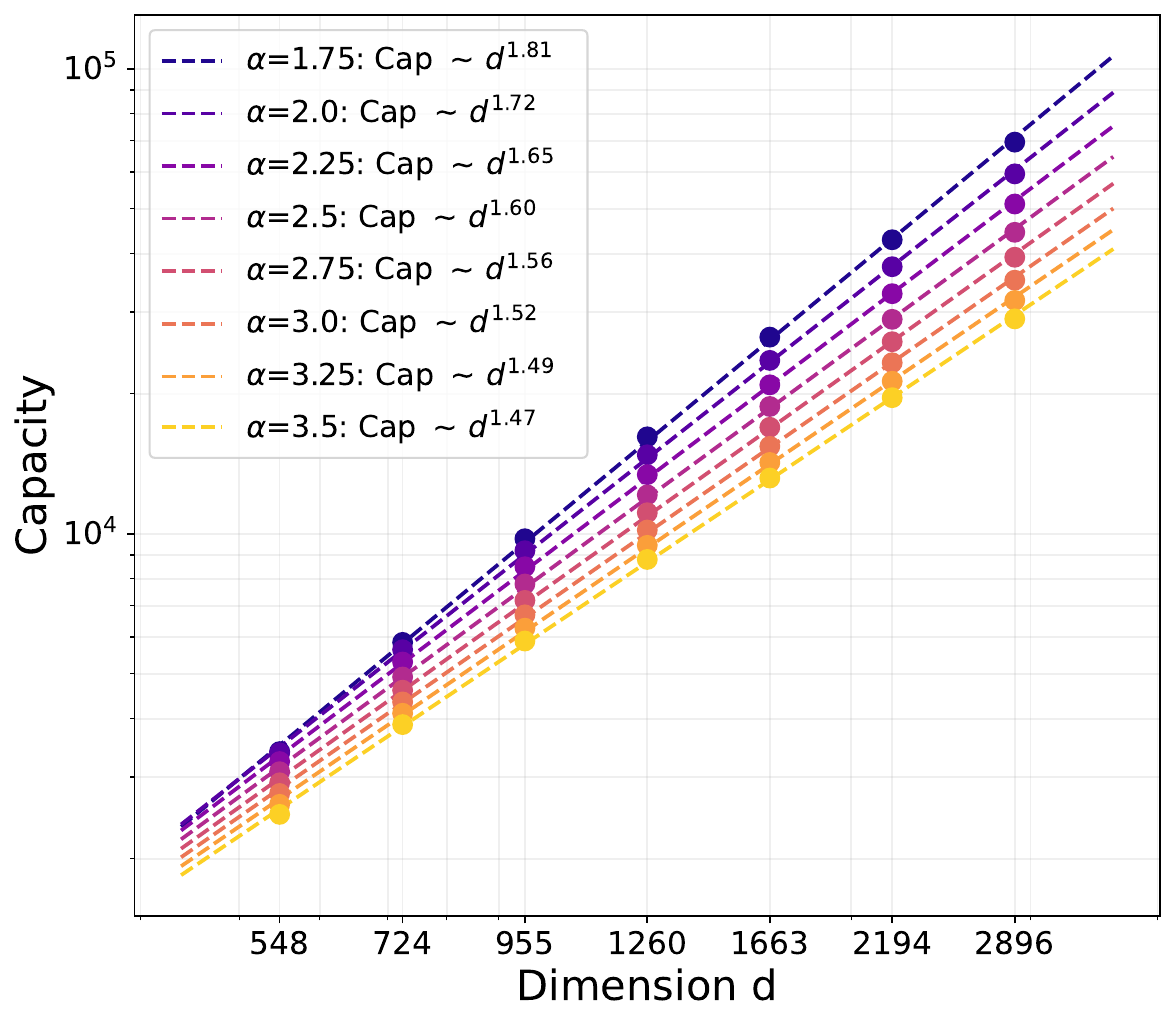}} 
%\vspace{-1mm}
\caption{Capacity at $T=4$.}
\label{fig:step-4}
\end{subfigure}
\\
\vspace{3mm}
\begin{subfigure}[t]{0.42\linewidth}
\centering
{\includegraphics[height=0.69\textwidth]{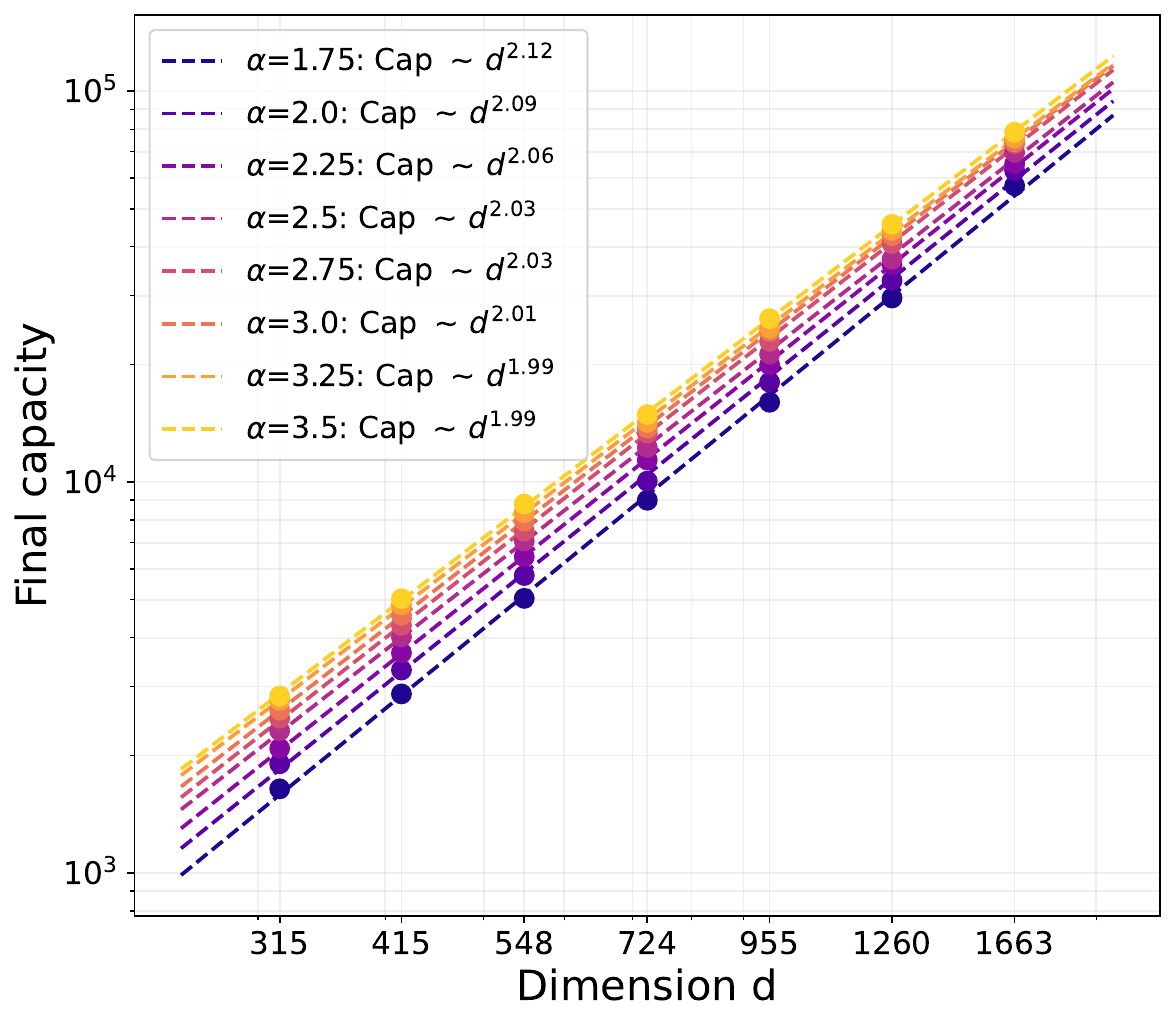}} 
%\vspace{-1mm}
\caption{Capacity near convergence.} 
\label{fig:step-convergence}
\end{subfigure}%
\begin{subfigure}[t]{0.42\linewidth}
\centering
{\includegraphics[height=0.691\textwidth]{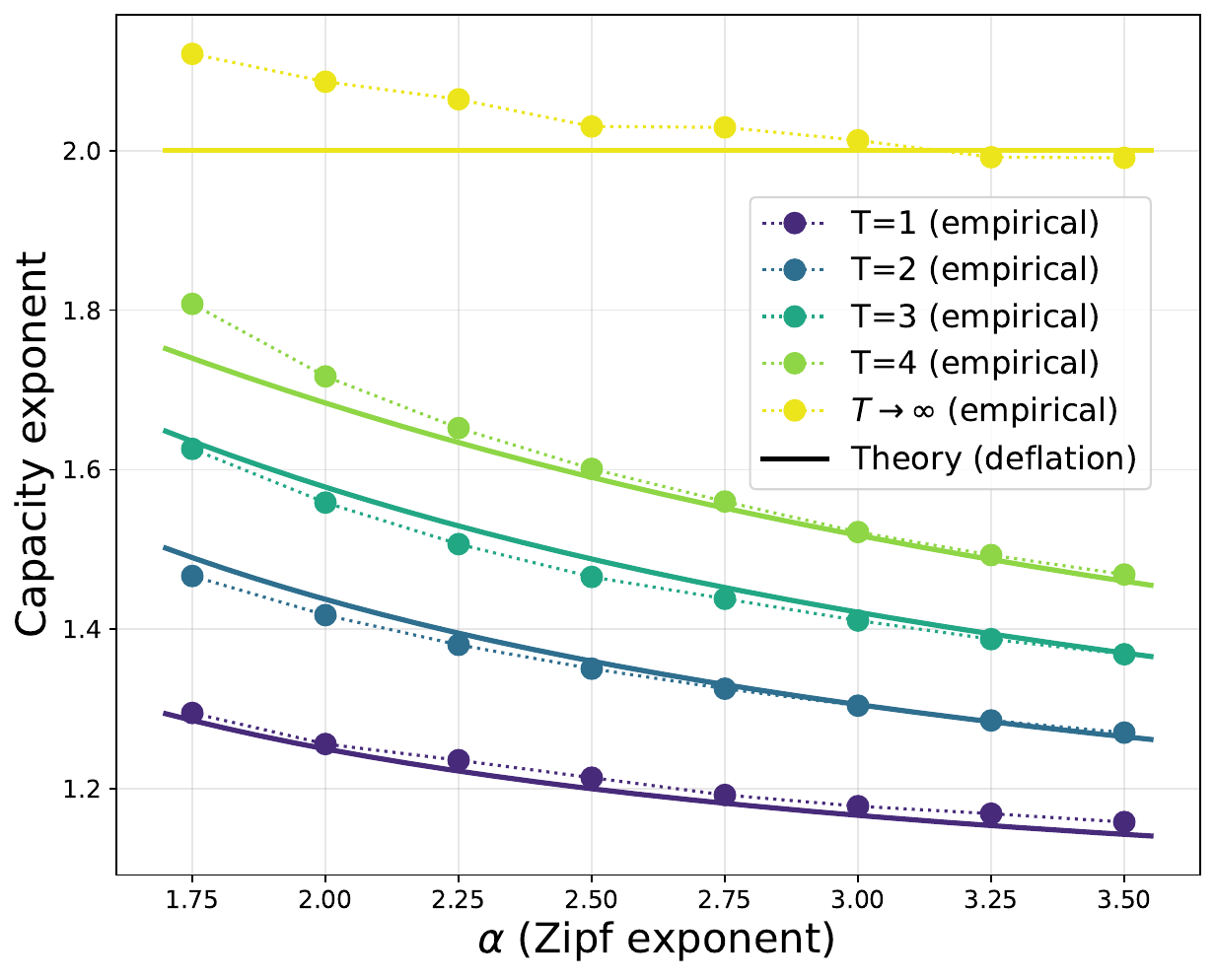}} 
%\vspace{-1mm}
\caption{Capacity scaling exponent across $T$.} 
\label{fig:multi-step-exponent}
\end{subfigure}%
%\vspace{-1mm} 
\caption{\small Capacity after~$T$ Muon steps on the population cross-entropy loss. We set $N=250{,}000$, $\eta=2\sqrt{d}$.  Figures~\ref{fig:step-2}, \ref{fig:step-3}, \ref{fig:step-4} report the capacity at $T = 2,3,4$, respectively (see Figure~\ref{fig:muon-population} for $T=1$); Figure~\ref{fig:step-convergence} presents the capacity at large $T$: we run Muon for up to $500$ steps and early stop when the capacity improvement over $10$ steps drops below $0.5\%$. Figure~\ref{fig:multi-step-exponent} compares the fitted dimension exponents against predictions of Theorem~\ref{thm:multi}; observe that the exponents agree except at small~$\alpha$ and large~$T$. 
}  
\label{fig:multi-step-population} 
\end{figure}   

\begin{figure}[!htb]
% \vspace{-2mm}
\centering
\begin{subfigure}[t]{0.49\linewidth}
\centering
{\includegraphics[height=0.69\textwidth]{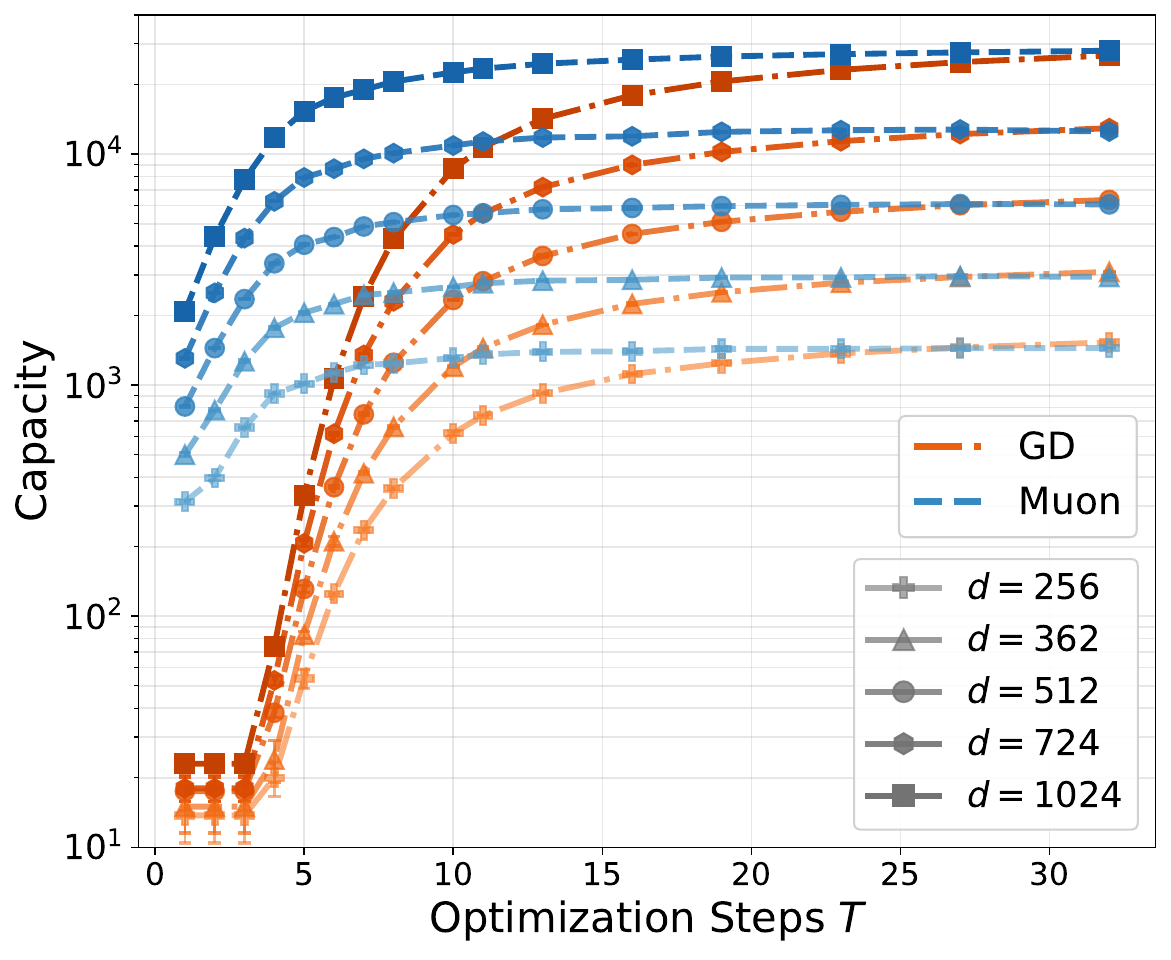}} 
%\vspace{-1mm}
\caption{Multi-step GD vs.~Muon (population).}
\label{fig:GD-vs-muon-multistep}
\end{subfigure}%
\begin{subfigure}[t]{0.49\linewidth}
\centering 
{\includegraphics[height=0.69\textwidth]{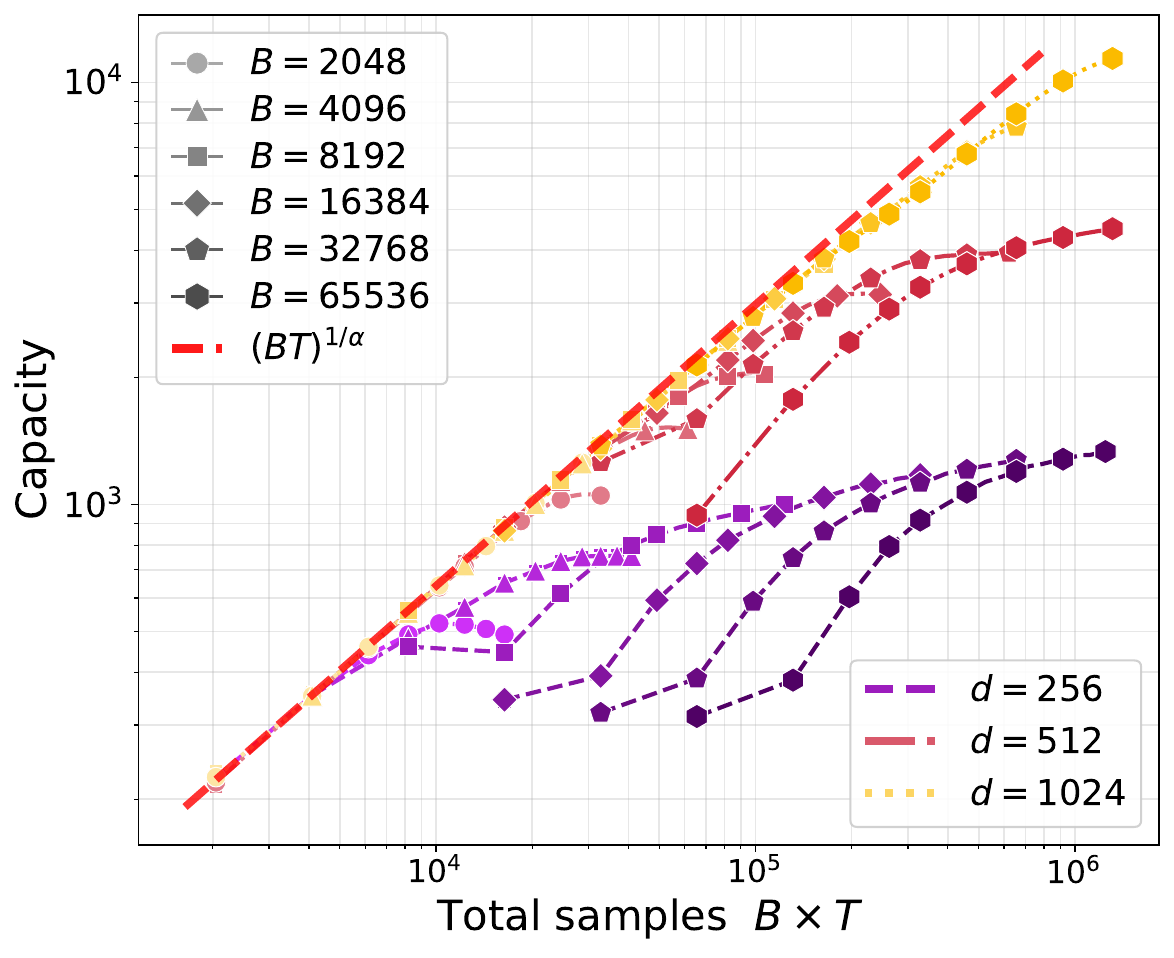}} 
%\vspace{-1mm}
\caption{\small Capacity of multi-step Muon (minibatch).} 
\label{fig:multistep-batch}
\end{subfigure}%
\vspace{-1mm} 
\caption{\small Capacity scaling of multi-step Muon and GD. We set $N=100,000$, $\alpha=1.5$. 
% Each experiment is repeated 5 times. 
\textbf{(a)} Population update: for GD we implement an increasing learning rate schedule (see Theorem~\ref{thm:gd-multi}) with $\eta_1 = 0.01\sqrt{d}$; for Muon we use a fixed step size $\eta=\sqrt{d}$. Observe that the benefit of Muon is most visible in the ``early phase'' of training (the initial plateau of GD in the first 3 steps is due to small $\eta_1$ chosen for numerical stability).   
\textbf{(b)} Capacity of minibatch Muon vs.~total sample size $B\times T$; for each batch size $B$, we run minibatch Muon for $T=20$ steps with $\eta=\sqrt{d}$. Dashed red line indicates the information-theoretic rate $(BT)^{1/\alpha}$. 
} 
% \vspace{-0.5mm} 
\label{fig:multi-step-2} 
\end{figure}

In Figure~\ref{fig:multi-step-2}, we further compare the multi-step capacity scaling of GD and Muon with $N=100{,}000$ and $\alpha=1.5$. Figure~\ref{fig:GD-vs-muon-multistep} compares their performance in minimizing the population cross-entropy loss. Muon attains much higher capacity than GD in the first few steps. On the other hand, with the increasing learning rate schedule from Theorem~\ref{thm:gd-multi}, GD catches up later in training, and both methods eventually reach the optimal $d^2$ capacity (we however note that this increasing~$\eta_t$ schedule for GD is numerically unstable, especially when~$\alpha$ is large). These results suggest Muon’s acceleration is most significant early in training, matching the predictions in Theorems~\ref{thm:multi} and~\ref{thm:gd-multi}. 

Figure~\ref{fig:multistep-batch} shows minibatch Muon up to $T=20$ steps with batch sizes $B \in \{2^{11}\!,\dots,2^{16}\}$. For batch sizes below the critical threshold, the capacity stays close to the information-theoretic limit $(BT)^{1/\alpha}$, i.e., the number of items that can be observed after $T$ steps with batch size $B$. For larger batch sizes (darker curves), the capacity saturates and sample efficiency worsens. 

\begin{figure}[!b]
\vspace{-1mm} 
\centering
\begin{subfigure}[t]{0.325\linewidth}
\centering
{\includegraphics[height=0.86\textwidth]{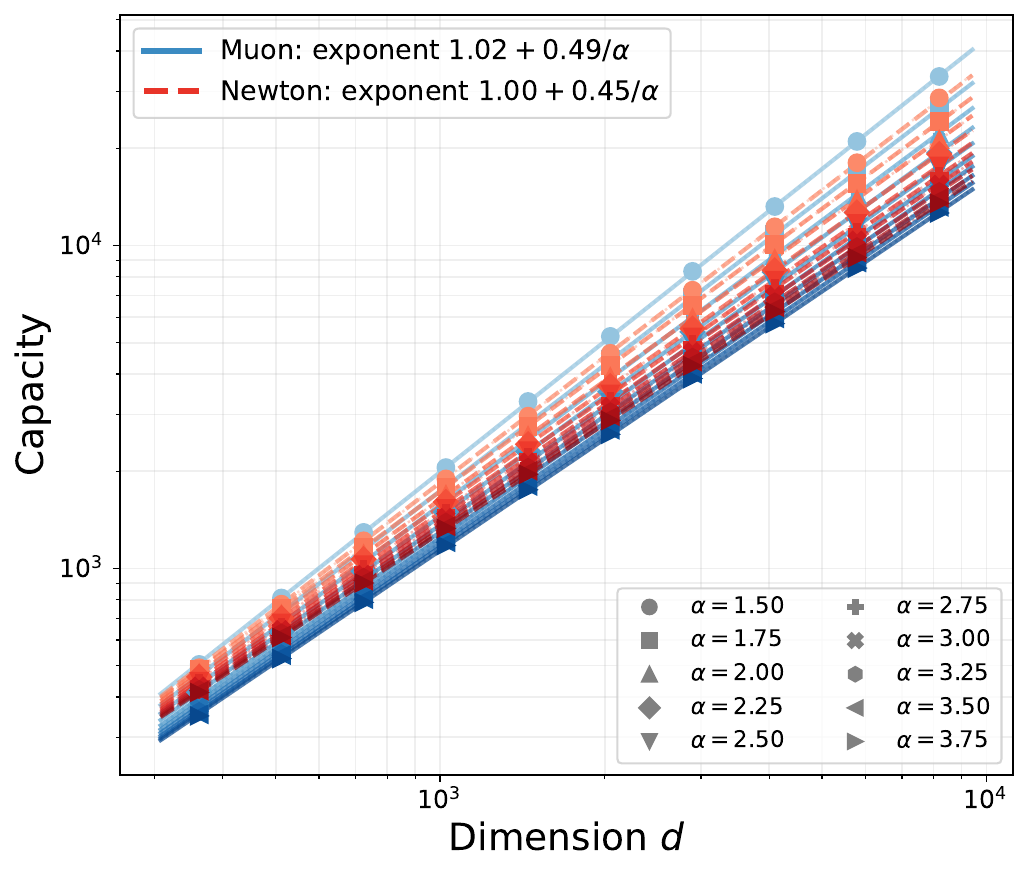}} 
%\vspace{-1mm}
\caption{Muon vs.~Newton ($\kappa=0$).}
\label{fig:kappa=0}
\end{subfigure}%
\begin{subfigure}[t]{0.325\linewidth}
\centering
{\includegraphics[height=0.86\textwidth]{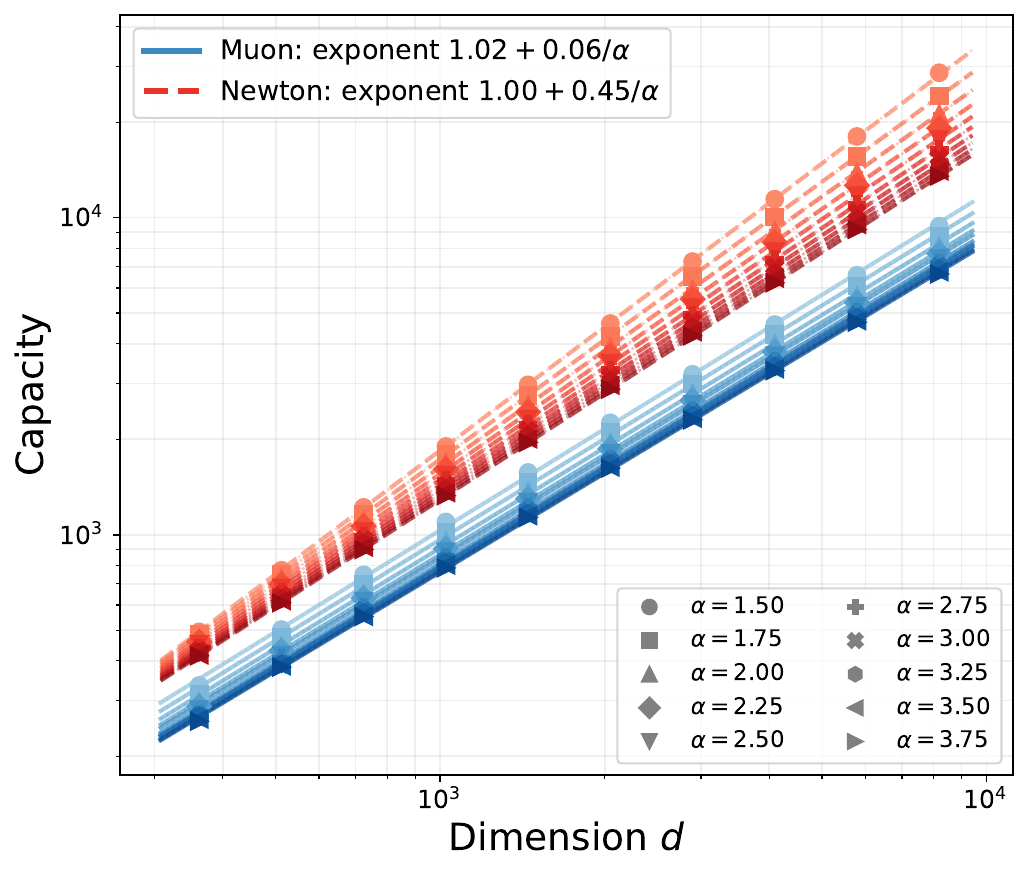}} 
%\vspace{-1mm}
\caption{Muon vs.~Newton ($\kappa=1$).}
\label{fig:kappa=1}
\end{subfigure}%
\begin{subfigure}[t]{0.325\linewidth}
\centering
{\includegraphics[height=0.86\textwidth]{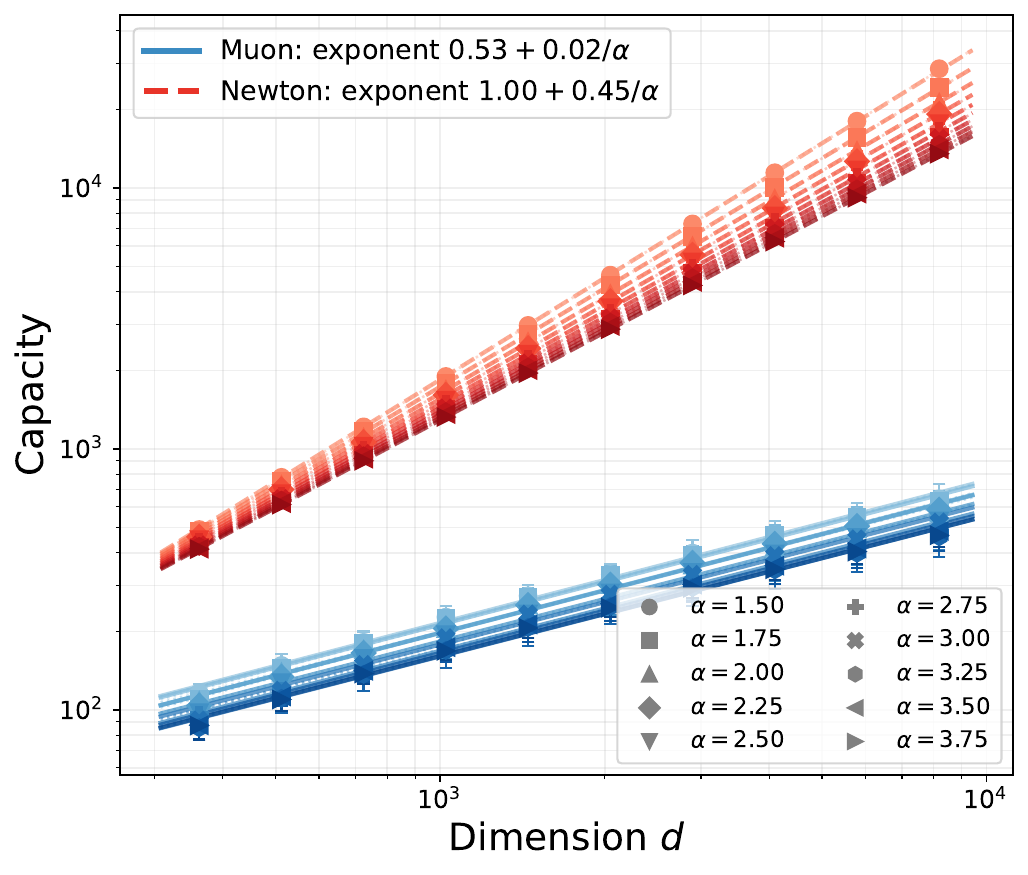}} 
%\vspace{-1mm}
\caption{Muon vs.~Newton ($\kappa=1.5$).}
\label{fig:kappa=2}
\end{subfigure}
\vspace{-1mm}
\caption{\small Capacity scaling after one (population) Muon and Newton step in the anisotropic setting: we choose $u_i \sim \cN(0,\frac{1}{d}\bI_d)$, $v_i \sim \cN(0,\bXi_v)$, where $\bXi_v$ is a trace-normalized diagonal matrix with $\lambda_{i}(\bXi_v)\asymp i^{-\kappa}$, $\kappa\ge 0$. 
We set $N=100,000$ and vary $d,\alpha$. For Newton's method we add a ridge regularization $\lambda=10^{-8}$ for numerical stability when the preconditioner is rank-deficient. Observe that when $\kappa=0$ (isotropic, Figure~\ref{fig:kappa=0}), Muon and Newton both achieve $d^{1+\frac{1}{2\alpha}}$ capacity, but as $\kappa$ increases (Figures \ref{fig:kappa=1}, \ref{fig:kappa=2}), the performance of Muon worsens while the Newton update remains invariant. 
}  
\label{fig:newton-population} 
\vspace{-1.5mm} 
\end{figure}

\paragraph{Comparison with Newton's method.} Figure~\ref{fig:newton-population} compares the storage efficiency of Muon and Newton's update in the population setting. As argued in Section~\ref{subsec:SGD-newton}, Muon should match Newton's method when the embedding and unembedding vectors are isotropic. To vary the anisotropy of the associative memory model, we consider
$
u_i \sim \cN (0,\tfrac{1}{d}\bI_d), v_i \sim \cN(0,\bXi_v),
$
where the covariance matrix $\bXi_v$ has eigenvalues $\lambda_{i}(\bXi_v)\asymp i^{-\kappa}$, $\kappa\ge 0$. We observe that in the isotropic case ($\kappa=0$), Muon matches the storage capacity of one Newton step; but as the data become more anisotropic, the gap between the two methods grows, and only Newton's method retains the $d^{1+\frac{1}{2\alpha}}$ capacity.

\subsection{In-context recall with transformers}

We next consider a simple associative recall task that can be solved by a two-layer transformer via the \emph{induction head} mechanism \citep{olsson2022incontextlearninginductionheads}. An induction head is a circuit composed of two attention heads that enables the model to copy a bigram from context, for example by predicting~$\bb$ after observing $[\ldots,\ba,\bb,\ldots,\ba]$. As shown by~\citet{bietti2023birth}, this mechanism can be implemented using a small number of associative memory matrices, making it a natural testbed for understanding how our perspective may extend to richer architectures such as multilayer transformers.

\paragraph{Data distribution.}
To study how optimizers interact with heavy-tailed data, we consider a variant of the synthetic model from~\citet{bietti2023birth} in which selected tokens follow power-law distributions. Specifically, we introduce two disjoint vocabularies~$\cQ$ and~$\cV$, each of size~$N$, together with three power-law distributions: the \emph{trigger} distribution~$p^{(t)}$, supported on~$\cQ$; the \emph{output} distribution~$p^{(o)}$, supported on~$\cV$; and the \emph{noise} distribution~$p^{(n)}$, also supported on~$\cV$ but potentially with a different frequency ordering from~$p^{(o)}$. 
Each sequence is generated by first sampling~$K$ triggers $q_1, \ldots, q_K \in \cQ$ from~$p^{(t)}$ without replacement, and~$K$ outputs $o_1, \ldots, o_K \in \cV$ from~$p^{(o)}$ with replacement. The resulting bigrams $(q_k, o_k)$ are then inserted into~$M$ random positions in a sequence of length~$T$. All remaining positions are filled with noise tokens~$n_j \in \cV$ sampled from~$p^{(n)}$. For instance, when $K=1$, $M=2$, and $T=8$, a sequence takes the form
\begin{align*}
[n_1, q_1, o_1, n_2, n_3, q_1, o_1, n_4].
\end{align*}
In this example, the second occurrence of~$o_1$ is perfectly predictable from the preceding context, and we train the two-layer transformer using cross-entropy loss only on such predictable output tokens. In all experiments, we set~$T=24$ and~$K=M=2$. For simplicity, we take~$p^{(n)}$ to be uniform over~$\cV$, while~$p^{(t)}$ and~$p^{(o)}$ follow power laws with exponents~$\alpha_t$ and~$\alpha_o$.

\paragraph{Architecture and optimizers.}
We use a two-layer transformer with single-head attention, Pre-LayerNorm, and basic RMS normalization, optionally augmented with feed-forward layers using a ReLU MLP. We compare Muon, SGD, and AdamW~\citep{loshchilov2019decoupledweightdecayregularization}, all trained with a constant step size and weight decay~$0.01$. For SGD and Muon, we use momentum~$0.9$, while for AdamW we set $(\beta_1,\beta_2)=(0.9,0.99)$. In the Muon implementation, we apply AdamW to the embedding and unembedding layers, and use $5$ Newton-Schulz iterations for the remaining layers.

\begin{figure}[!b]
  \centering
  \includegraphics[width=0.99\linewidth]{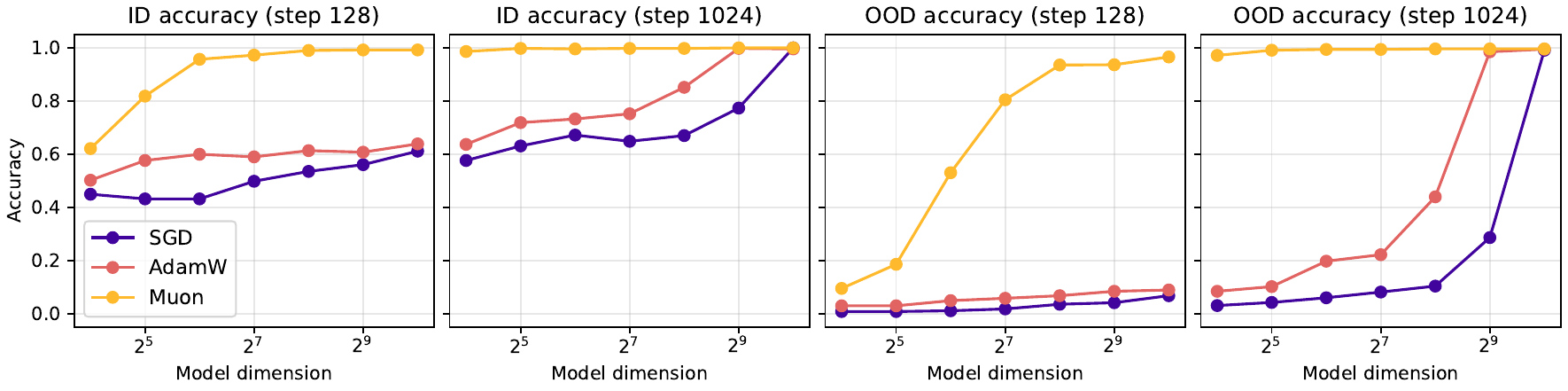}
  \caption{\small
    ID (left two) and OOD (right two) accuracy on the in-context recall task as a function of model dimension, for Muon, AdamW, and SGD, with batch size $256$ at iterations $128$ and $1024$. For each $(\mathrm{dim}, \mathrm{optimizer})$ pair, the learning rate and batch size are chosen to maximize accuracy.
  }
  \label{fig:acc_vs_dim}
  %\vspace{-1mm}
\end{figure}

\begin{figure}[!b]
\vspace{-1.5mm}
  \centering
  % Row 1: OOD accuracy at fixed iteration
  \begin{subfigure}{\textwidth}
    %\includegraphics[width=0.325\linewidth]{figures/ihead/acc_vs_batchsize_muon.pdf}
    %\hfill
    %\includegraphics[width=0.325\linewidth]{figures/ihead/acc_vs_batchsize_adamw.pdf}
    %\hfill
    %\includegraphics[width=0.325\linewidth]{figures/ihead/acc_vs_batchsize_sgd.pdf}
    \includegraphics[width=0.99\linewidth]{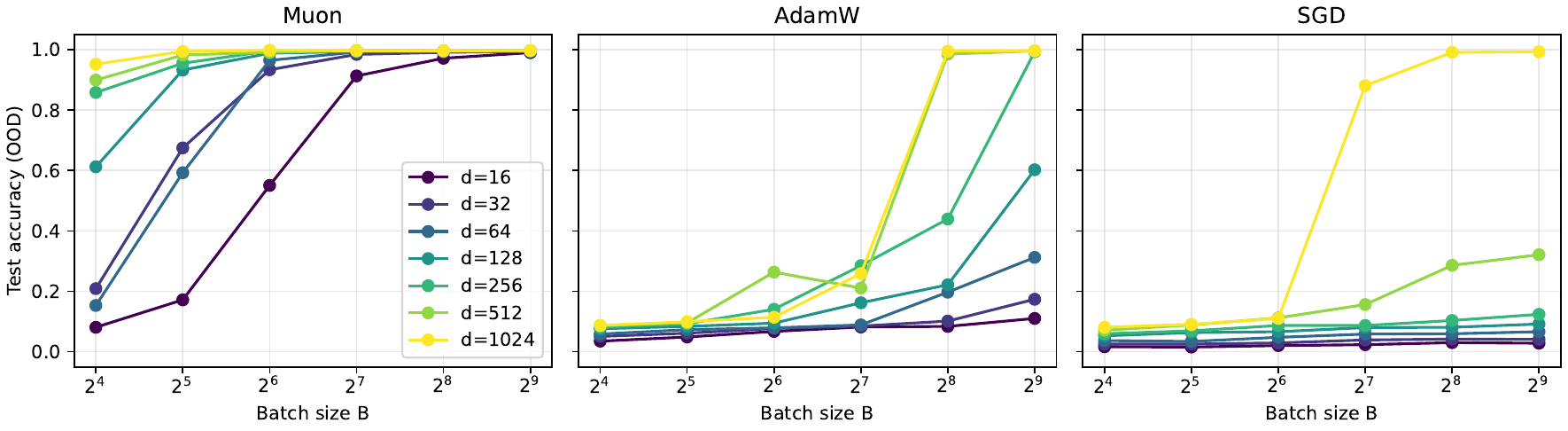}
    \vspace{-1mm}
    \subcaption{OOD accuracy against~$B$ ($\alpha_o = 1.5$).}
    \label{fig:batchsize-ood}
  \end{subfigure}

  \vspace{0.4em}

  % Row 2: Fixed token budget, n_seq = 16384
%\begin{subfigure}{\textwidth}
%\includegraphics[width=0.32\linewidth]{figures/ihead/acc_vs_batchsize_fixedtoks_nseq16384_muon.pdf}
%\hfill
%\includegraphics[width=0.32\linewidth]{figures/ihead/acc_vs_batchsize_fixedtoks_nseq16384_adamw.pdf}
%\hfill
%\includegraphics[width=0.32\linewidth]{figures/ihead/acc_vs_batchsize_fixedtoks_nseq16384_sgd.pdf}
%\subcaption{Fixed token budget $N_{\mathrm{seq}} = 16384$ ($\alpha_o = 1.1$).}
%\end{subfigure}

%\vspace{0.5em}

  % Row 3: Fixed token budget, n_seq = 65536
%\begin{subfigure}{\textwidth}
%\includegraphics[width=0.32\linewidth]{figures/ihead/acc_vs_batchsize_fixedtoks_nseq65536_muon.pdf}
%\hfill
%\includegraphics[width=0.32\linewidth]{figures/ihead/acc_vs_batchsize_fixedtoks_nseq65536_adamw.pdf}
%\hfill
%\includegraphics[width=0.32\linewidth]{figures/ihead/acc_vs_batchsize_fixedtoks_nseq65536_sgd.pdf}
%\subcaption{Fixed token budget $N_{\mathrm{seq}} = 65536$ ($\alpha_o = 1.1$).}
%\end{subfigure}

  \begin{subfigure}{\textwidth}
%\includegraphics[width=0.325\linewidth]{figures/ihead/wo1acc_vs_batchsize_muon.pdf}
   %\hfill
   %\includegraphics[width=0.325\linewidth]{figures/ihead/wo1acc_vs_batchsize_adamw.pdf}
   %\hfill
   %\includegraphics[width=0.325\linewidth]{figures/ihead/wo1acc_vs_batchsize_sgd.pdf}
   \includegraphics[width=0.99\linewidth]{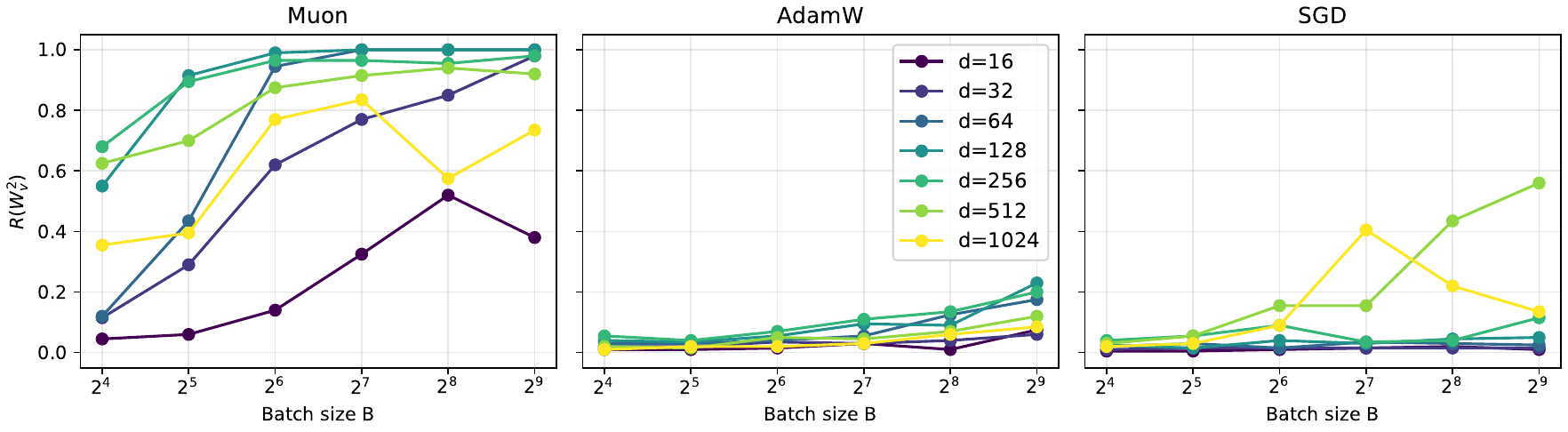}
    \vspace{-1mm}
   \caption{%
     Memory recall accuracy $R(\bW_V^2)$ against~$B$.
   }
   \label{fig:batchsize-r}
  \end{subfigure}
%\vspace{-6.6mm}
  \caption{\small
    OOD and memory recall accuracy as a function of batch size $B$, for Muon, AdamW, and SGD (columns left to right), with different curves per model dimension, at iteration $1024$. For Figure~\ref{fig:batchsize-r}, we use a two-layer transformer with no feed-forward layers to avoid redundancies between the value matrix and the subsequent MLP layer. For each $(B, \text{dim})$ pair, the learning rate is chosen to maximize accuracy.
  }
  \label{fig:batchsize}
  \vspace{-2.mm} 
\end{figure}

\paragraph{Evaluation metrics.}
To assess the capacity and robustness of these optimizers to different power laws, we evaluate using \textbf{out-of-distribution (OOD) accuracy} on the next-token predictions for the second output tokens, on a batch of out-of-distribution data generated with uniform~$p^{(t)},p^{(o)}$. This relates to capacity in the sense that it measures what fraction of the $N \times N$ pairs $(q, o)$ the model is able to recall in-context, even though the triggers and outputs seen during training are power-law distributed. We also evaluate the \textbf{memory recall accuracy} for the value matrix at the second layer, denoted~$\bW_V^{(2)}$, which is expected to map input token embeddings~$e_v$ to output token embeddings~$u_v$ for all~$v \in \cV$, as described in~\citet{bietti2023birth}. Concretely, we compute
\begin{align*}
R(\bW_V^{(2)}) = \textstyle\frac{1}{|\cV|} \sum_{v \in \cV} \1\{\argmax_{v'\in\cV} u_{v'}^\top \bW_V^{(2)} e_v = v\}.
\end{align*}

\paragraph{Results.}
Figure~\ref{fig:acc_vs_dim} reports the in-distribution (ID) and OOD accuracy of the three optimizers as the model dimension varies. Muon consistently outperforms SGD and AdamW across dimensions, training iterations, and both evaluation metrics. In Figures~\ref{fig:batchsize-ood} and~\ref{fig:batchsize-r}, we vary the batch size while fixing the number of training steps, and measure the resulting OOD accuracy and the recall accuracy of the value matrix, respectively. We again observe that Muon has the best performance; however, Muon attains near-perfect accuracy at a smaller batch size compared to SGD and AdamW, thus saturating more quickly. This discrepancy with our theoretical analysis of the critical batch size is likely due to the different task and optimizer setups, so that the information-theoretic rate is not the main bottleneck. We leave a quantitative investigation of this gap to future work.

Finally, Figure~\ref{fig:acc_vs_dim_alpha} examines how model performance changes with the power-law exponents of the output and trigger distributions, $\alpha_o$ and $\alpha_t$, while all other distributions are kept uniform. Larger values of $\alpha$ correspond to faster decay in item frequencies and therefore make learning more difficult. However, our discussion of signal amplification in Section~\ref{sec:multi_step} suggests that Muon should be more robust to this effect. Consistent with this intuition, we observe that Muon and AdamW perform similarly when $\alpha=0$, but as $\alpha$ increases, Muon remains remarkably robust whereas AdamW degrades much more sharply. This behavior also qualitatively aligns with our one-step analysis: as $\alpha$ grows larger, Muon still retains $\Omega(d)$ capacity, but the SGD recovery rate will collapse to zero.

\begin{figure}[h]
  \centering
% \vspace{-2mm}
  \begin{subfigure}{\textwidth}
    %\includegraphics[width=0.325\linewidth]{figures/ihead/acc_vs_dim_alphaoutput_muon.pdf}
    %\hfill
    %\includegraphics[width=0.325\linewidth]{figures/ihead/acc_vs_dim_alphaoutput_adamw.pdf}
    %\hfill
    %\includegraphics[width=0.325\linewidth]{figures/ihead/acc_vs_dim_alphaoutput_sgd.pdf}

    \includegraphics[width=0.99\linewidth]{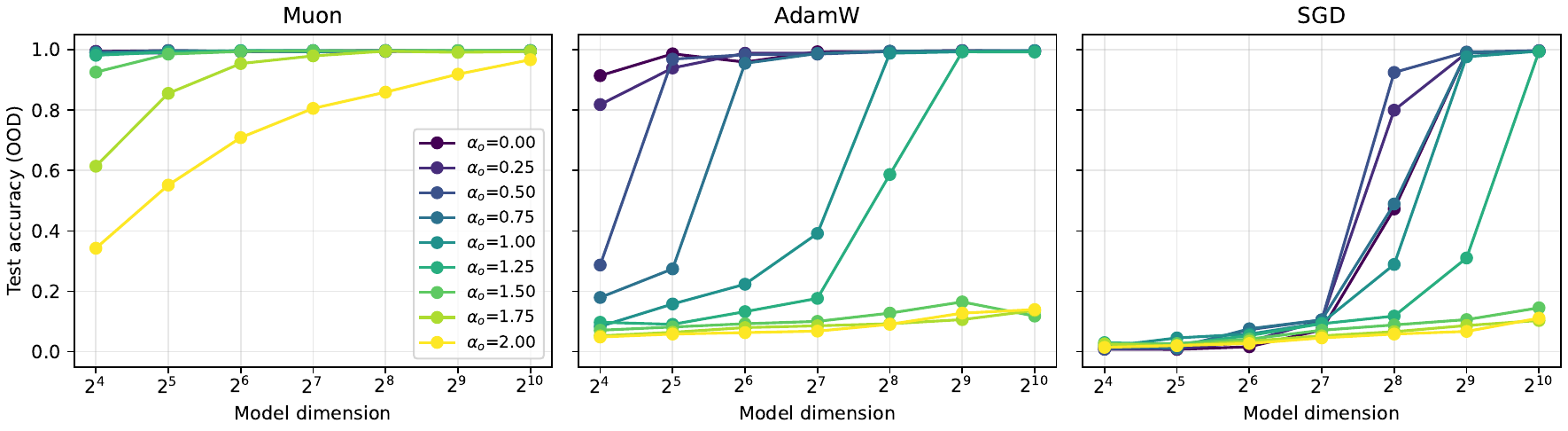}
    %\vspace{-0.5mm}
    \subcaption{Varying output token distribution~$\alpha_o$, with $\alpha_t = 0$.}
  \end{subfigure}

  \vspace{0.4em}

  \begin{subfigure}{\textwidth}
    %\includegraphics[width=0.325\linewidth]{figures/ihead/acc_vs_dim_alphatrigger_muon.pdf}
    %\hfill
    %\includegraphics[width=0.325\linewidth]{figures/ihead/acc_vs_dim_alphatrigger_adamw.pdf}
    %\hfill
    %\includegraphics[width=0.325\linewidth]{figures/ihead/acc_vs_dim_alphatrigger_sgd.pdf}

    \includegraphics[width=0.99\linewidth]{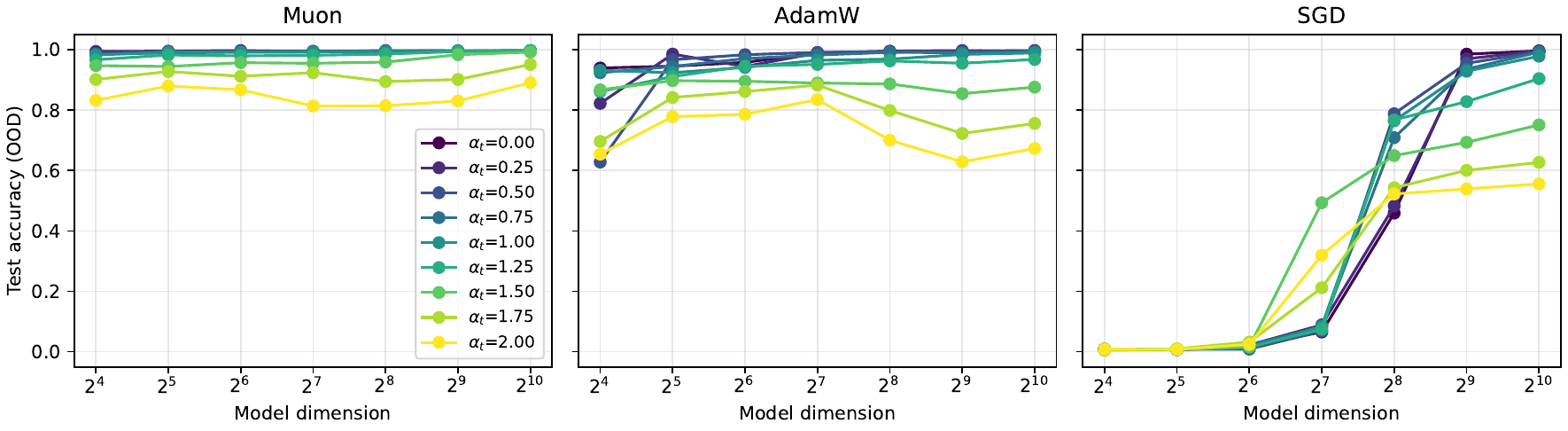}
    %\vspace{-0.5mm}
    \subcaption{Varying trigger distribution exponent~$\alpha_t$, with $\alpha_o = 0$.}
  \end{subfigure}
% \vspace{-6.6mm}
  \caption{\small
    OOD accuracy as a function of model dimension for Muon, AdamW, and SGD (columns left to right), with batch size $256$ at iteration $512$. 
    Each curve corresponds to a different power-law exponent for (a) the output distribution~$\alpha_0$; (b) the trigger distribution~$\alpha_t$, with $\alpha = 0$ being the uniform
    distribution and larger $\alpha$ concentrating probability mass on fewer tokens (where we expect adaptive optimizers to be beneficial).
    For each $(\text{dim}, \alpha)$ pair, the learning rate is chosen to
    maximize OOD accuracy.
  }
  \label{fig:acc_vs_dim_alpha}
\end{figure}

\section{Conclusion}

In this work, we study the storage capacities of Muon, SGD, and Newton's method in a linear associative memory model with random embeddings and power-law item frequencies. We sharply characterize the one-step recovery rate of each optimizer and show that Muon recovers substantially more items than SGD, even matching the Newton update. This gap also implies a much larger critical batch size for Muon compared to SGD, helping explain why its advantage is most visible in large-batch training. We further show that Muon’s gains are concentrated early in training, whereas over longer horizons its recovery dynamics become comparable to those of SGD. Overall, our analysis opens the door towards establishing scaling laws for spectral optimizers in realistic language modeling settings. Promising directions for future work include studying anisotropic embeddings, the role of momentum \citep{ferbach2025dimension}, and more complex language-modeling tasks such as compositional and multi-hop reasoning \citep{sanford2024transformers,wang2025learning}.

\bigskip

\section*{Acknowledgements}
The authors thank Elliot Paquette and Yue M.~Lu for discussion and feedback. JK, EN, and JDL acknowledge support of a Google Research Award, NSF IIS 2107304,  NSF CCF 2539753, NSF CAREER Award 2540142, and NSF CCF 2019844.

\bigskip

% \clearpage
\bibliography{main}

\clearpage

{\hypersetup{linkcolor=black}
\setlength{\parskip}{4pt}
\tableofcontents
}
\allowdisplaybreaks

\clearpage
\appendix

\section{Proof of Theorem~\ref{thm:main}}\label{sec:a}

The negative gradient at initialization is
\begin{align*}
\bG_0 = -\nabla_\bW L(\bW_0;\cB) = \sum_{i\in[N]} q_i(u_i - \bar u)v_i^\top
\end{align*}
where $\bar{u} = \frac1N \sum_{i\in[N]} u_i$ is a centering term. It will suffice to study the \emph{uncentered} gradient~$\bG$ and logits~$\gamma_{ij}$, given as
\begin{align*}
\gamma_{ij} := u_j^\top h_\lam(\bG)v_i, \quad \bG :=
\sum_{i\in[N]} q_iu_iv_i^\top.
\end{align*}
In Sections~\ref{sec:frechet}-\ref{sec:signal}, we lower bound the signal terms~$\gamma_{ii}$. In Section~\ref{sec:interaction}, we upper bound the magnitude of the interaction terms~$\gamma_{ij}$ for $i\ne j$. Finally in Section~\ref{sec:putting}, we conclude the proof of Theorem~\ref{thm:main}.

\subsection{Fr\'{e}chet derivative computations}\label{sec:frechet}

In this subsection, we work with general smooth nondecreasing functions $h:\RR_{\ge 0}\to\RR_{\ge 0}$. Let the leave-one-out gradient be $\bG_{-i} := \bG - q_i u_i v_i^\top$ and define the function
\begin{align*}
\phi(q) = u_i^\top h(\bG_{-i} + qu_iv_i^\top)v_i, \quad q\ge 0.
\end{align*}
%Crucially, the vector $v_i$ is independent of $\bG_{-i}$.
We aim to control the signal $\gamma_{ii} = \phi(q_i)$ via Taylor expansion. We will utilize the Daleckii--Krein formula for the Fr\'{e}chet derivative of matrix functions.

\begin{proposition}[Daleckii--Krein formula]\label{prop:daleckii}
Let $\bM$, $\bE$ be real symmetric matrices and $f$ be a $(2d-1)$-times continuously differentiable function. Denote by $Df(\bM) [\bE]$ the Fr\'{e}chet derivative of $f$ w.r.t.~$\bM$ in the~$\bE$ direction. Let $\bM = \bP \bLambda \bP^\top$ be the eigendecomposition of~$\bM$ with $\bLambda = \diag(\lam_1,\cdots,\lam_d)$. Let~$f^{(1)}$ be the first divided difference of~$f$,
\begin{align*}
f^{(1)}(x,y) = \begin{cases}
        \frac{f(x) - f(y)}{x-y} & x \ne y,\\
        f'(x) & x=y,
    \end{cases}
\end{align*}
and set $\bT_{ij} := f^{(1)}(\lambda_i, \lambda_j)$ for $1\le i,j\le d$. Then
\begin{align*}
    \bP^\top Df(\bM)[\bE] \bP = (\bP^\top \bE \bP) \circ \bT.
\end{align*}
Furthermore, let $D^2f(\bM)[\bE, \bE]$ be the second Fr\'{e}chet derivative of $f$ w.r.t. $\bM$ in the $\bE,\bE$ directions. Let $f^{(2)}$ be the second divided difference of~$f$,
\begin{align*}
f^{(2)}(x, y, z) = \begin{cases}
        \frac{f^{(1)}(x, z) - f^{(1)}(y, z)}{x - y} & x\ne y,\\
        \partial_xf^{(1)}(x,z) & x=y.
    \end{cases}
\end{align*}
    Then
    \begin{align*}
        (\bP^\top D^2f(\bM)[\bE, \bE] \bP)_{ij} = \sum_{k=1}^d f^{(2)}(\lambda_i, \lambda_j, \lambda_k)(\bP^\top \bE \bP)_{ik}(\bP^\top \bE \bP)_{jk}.
    \end{align*}
    \end{proposition}

\begin{proof}
We only prove the formula when $f$ is a polynomial for illustrative purposes; the full proof is given in Theorem 3.11 and Corollary 3.12 of \citet{higham2008functions}. By linearity, it suffices to consider the case where $f(z)=z^n$, $n\in\NN$ is a monomial. The Fr\'{e}chet derivative of~$f$ in the $\bE$ direction is
\begin{align}\label{eq:frechet-mono}
Df(\bM)(\bE) = \sum_{k=1}^n \bM^{k-1}\bE\bM^{n-k}
\end{align}
and the first divided difference is
\begin{align*}
\bT_{ij} = f^{(1)}(\lam_i,\lam_j) = \sum_{k=1}^n \lam_i^{k-1} \lam_j^{n-k}.
\end{align*}
Therefore we directly check that
\begin{align*}
(\bP^\top Df(\bM)(\bE)\bP)_{ij} &= \sum_{k=1}^n e_i^\top \bLambda^{k-1}\bP^\top \bE\bP \bLambda^{n-k} e_j\\
&= \sum_{k=1}^n \lam_i^{k-1}(\bP^\top \bE\bP)_{ij} \lam_j^{n-k} = (\bP^\top \bE\bP \circ\bT)_{ij}.
\end{align*}
%The second derivative is checked similarly.
For the second derivative, differentiating Eq.~\eqref{eq:frechet-mono} again gives
\begin{align*}
D^2f(\bM)[\bE,\bE]
= 2\sum_{1\le \ell < k \le n} \bM^{\ell-1}\bE\bM^{k-\ell-1}\bE\bM^{n-k}
\end{align*}
and
\begin{align*}
f^{(2)}(x,y,z) = \frac{f^{(1)}(x,z)-f^{(1)}(y,z)}{x-y} = \sum_{1\le \ell < k \le n} x^{\ell-1}y^{k-\ell-1}z^{n-k}.
\end{align*}
Hence
\begin{align*}
(\bP^\top D^2f(\bM)[\bE,\bE]\bP)_{ij}
&= 2\sum_{1\le \ell < k \le n} e_i^\top \bLambda^{\ell-1}\bP^\top \bE\bP \bLambda^{k-\ell-1}\bP^\top \bE\bP \bLambda^{n-k} e_j\\
&= 2\sum_{m=1}^d \sum_{1\le \ell < k \le n}
\lam_i^{\ell-1}\lam_m^{k-\ell-1}\lam_j^{n-k}
(\bP^\top \bE\bP)_{im}(\bP^\top \bE\bP)_{mj}\\
&= 2\sum_{m=1}^d f^{(2)}(\lam_i,\lam_m,\lam_j)
(\bP^\top \bE\bP)_{im}(\bP^\top \bE\bP)_{mj}\\
&= 2\sum_{m=1}^d f^{(2)}(\lam_i,\lam_j,\lam_m)
(\bP^\top \bE\bP)_{im}(\bP^\top \bE\bP)_{jm},
\end{align*}
where in the last step we used the symmetry of $f^{(2)}$ and $\bE$.
\end{proof}

The following two lemmas use the Daleckii--Krein formula to compute the first and second derivatives of~$\phi$.

\begin{lemma}\label{lem:phi-prime}
Let the SVD of the leave-one-out matrix be $\bG_{-i} = \bA\bS\bB^\top$ with singular values $\bS = \diag(s_1, \dots, s_d)$ in decreasing order and denote $a = \bA^\top u_i$, $b = \bB^\top v_i$. Then it holds that $\phi'(q) \ge 0$ for all $q$ and
\begin{align*}
    \phi'(0) &= \frac14 \sum_{k \neq \ell}\qty( \frac{h(s_k) + h(s_\ell)}{s_k + s_\ell} (a_kb_\ell - a_\ell b_k)^2 + \frac{h(s_k) - h(s_\ell)}{s_k - s_\ell} (a_kb_\ell + a_\ell b_k)^2)\\
    &\qquad + \sum_k h'(s_k) a_k^2 b_k^2.
\end{align*}
\end{lemma}

Note that if $s_k=s_\ell$, the ratio $\frac{h(s_k) - h(s_\ell)}{s_k - s_\ell}$ is to be interpreted as the first divided difference $h^{(1)}(s_k,s_k) = h'(s_k)$; if $s_k=s_\ell=0$, $\frac{h(s_k) + h(s_\ell)}{s_k + s_\ell}$ is to be interpreted as the continuous limit~$h'(0)$.

\begin{proof}
For notational convenience, we define $u = u_i$, $v = v_i$ and
\begin{align*}
\bG_0 = \bG_{-i}, \quad \bG_q = \bG_0+quv^\top,
\end{align*}
so that $\phi(q) = u^\top h(\bG_q) v$. Let the auxiliary function $\xi(z) = h(\sqrt{z})/\sqrt{z}$ for $z>0$; for our stabilized version of Muon, $\xi(z) = 1/\sqrt{z+\lam^2}$. We also define the SVD of $\bG_q$ and related quantities
\begin{align*}
\bG_q &= \bA_q\bS_q\bB_q^\top,\\
\bY_q &= h (\bG_q) = \bA_q h (\bS_q)\bB_q^\top,\\
\bT_q &= \bG_q^\top\bG_q,\\
\bR_q &= \xi(\bT_q) = \bB_q h (\bS_q)\bS_q^{-1}\bB_q^\top.
\end{align*}
Finally, we redefine $a = \bA_q^\top u, b = \bB_q^\top v$, here dependent on $q$.

First, observe that $\bY_q = \bG_q\bR_q$. Using dot notation, differentiating w.r.t. the variable $q$ yields $\dot \bG_q=uv^\top$ and
\begin{align*}
\dot \bY_q &= \dot\bG_q \bR_q + \bG_q \dot \bR_q = uv^\top \bR_q + \bG_q \dot \bR_q.
\end{align*}
Therefore
\begin{align}
    \phi'(q) &= u^\top \dot \bY_q v\notag\\
    &= \norm{u}_2^2\cdot v^\top \bR_q v + u^\top \bG_q\dot \bR_q v \notag \\
    &= \norm{a}_2^2\cdot b^\top \xi(\bS_q^2) b + a^\top \bS_q \bB_q^\top \dot \bR_q \bB_q b.\label{eq:phi-prime-formula}
\end{align}
    Next, we have $\dot \bR_q = D\xi(\bT_q)[\dot \bT_q]$ where
    \begin{align*}
        \dot \bT_q = \dot \bG_q^\top \bG_q + \bG_q^\top \dot \bG_q = vu^\top \bG_q + \bG_q^\top uv^\top
    \end{align*}
and so
\begin{align*}
\bB_q^\top \dot \bT_q \bB_q = \bS_q ab^\top + ba^\top \bS_q.
\end{align*}
Since $\bT_q$ is diagonalized as $\bT_q = \bB_q\bS_q^2\bB_q^\top$, we compute via the Daleckii--Krein formula (Proposition~\ref{prop:daleckii}),
\begin{align*}
    (\bB_q^\top \dot \bR_q \bB_q)_{k\ell} = (\bB_q^\top \dot \bT_q \bB_q)_{k\ell}\cdot \xi^{(1)}(s_k^2, s_\ell^2) = (s_ka_kb_\ell + s_\ell a_\ell b_k) \xi^{(1)}(s_k^2, s_\ell^2)
\end{align*}
where $\bS_q=\diag(s_1,\cdots,s_d)$. We may assume all $s_k$ are distinct; the general case follows from continuity. Plugging into~\eqref{eq:phi-prime-formula}, we obtain:
\begin{align*}
&\phi'(q)\\
&= \norm{a}_2^2 \sum_{\ell} b_\ell^2 \xi(s_\ell^2) + \sum_{k,\ell} a_k s_k (\bB_q^\top \dot \bR_q \bB_q)_{k\ell} b_\ell\\
&= \norm{a}_2^2 \sum_{\ell} b_\ell^2 \xi(s_\ell^2) + \sum_{k,\ell} a_k s_k b_\ell(s_ka_kb_\ell + s_\ell a_\ell b_k) \xi^{(1)}(s_k^2,s_\ell^2)\\
&= \norm{a}_2^2 \sum_{\ell} b_\ell^2 \xi(s_\ell^2) + 2\sum_k a_k^2 b_k^2 s_k^2 \xi'(s_k^2)  +\sum_{k \neq \ell} a_k s_k b_\ell(s_ka_kb_\ell + s_\ell a_\ell b_k) \frac{\xi(s_k^2) - \xi(s_\ell^2)}{s_k^2 - s_\ell^2}\\
    &= \sum_{k,\ell} a_k^2 b_\ell^2 \frac{h(s_\ell)}{s_\ell} + \sum_{k\neq \ell} a_k s_k b_\ell(s_ka_kb_\ell + s_\ell a_\ell b_k) \frac{h(s_k)s_\ell - h(s_\ell)s_k}{s_ks_\ell(s_k^2 - s_\ell^2)} \\
    &\qquad + \sum_k a_k^2 b_k^2s_k^2\qty(\frac{h'(s_k)}{s_k^2} - \frac{h(s_k)}{s_k^3})\\
    &= \sum_{k\neq \ell} a_k^2 b_\ell^2\qty(\frac{h(s_\ell)}{s_\ell} + \frac{h(s_k)s_ks_\ell - h(s_\ell)s_k^2}{s_\ell(s_k^2-s_\ell^2)}) + a_k a_\ell b_k b_\ell\qty(\frac{h(s_k)s_\ell - h(s_\ell)s_k}{s_k^2 - s_\ell^2}) \\
    &\qquad + \sum_k a_k^2 b_k^2 h'(s_k)\\
        &= \sum_{k\neq \ell} a_k^2 b_\ell^2\qty(\frac{h(s_k)s_k - h(s_\ell)s_\ell}{s_k^2-s_\ell^2}) + a_\ell b_k b_\ell\qty(\frac{h(s_k)s_\ell - h(s_\ell)s_k}{s_k^2 - s_\ell^2}) \\
    &\qquad + \sum_k a_k^2 b_k^2 h'(s_k) \\
&= \frac12 \sum_{k\ne\ell} (a_k^2 b_\ell^2 + a_\ell^2 b_k^2) \qty(\frac{h(s_k) - h(s_\ell)}{s_k - s_\ell} + \frac{h(s_k) + h(s_\ell)}{s_k + s_\ell}) \\
&\qquad + \frac12 \sum_{k\ne\ell} \qty(a_k a_\ell b_k b_\ell + a_\ell a_k b_\ell b_k) \qty(\frac{h(s_k) - h(s_\ell)}{s_k - s_\ell} - \frac{h(s_k) + h(s_\ell)}{s_k + s_\ell}) \\
&\qquad + \sum_k a_k^2 b_k^2 h'(s_k) \\
        &= \frac14\sum_{k\neq \ell} (a_kb_\ell - a_\ell b_k)^2 \frac{h(s_k) + h(s_\ell)}{s_k + s_\ell} + (a_kb_\ell + a_\ell b_k)^2 \frac{h(s_k) - h(s_\ell)}{s_k - s_\ell}\\
        &\qquad + \sum_k a_k^2 b_k^2 h'(s_k).
    \end{align*}
We conclude that since $h$ is increasing, $\phi'(q) \ge 0$, and moreover taking $q=0$ gives the desired formula for $\phi'(0)$.
\end{proof}

%\begin{lemma}\label{lem:second-explicit}
%In the setting of Lemma~\ref{lem:phi-prime},
%\begin{align*}
%\phi''(0) &= 6\|a\|_2^2 \sum_{j,k} s_k a_k b_k b_j^2 \cdot \xi^{(1)}(s_k^2,s_j^2)\\ &\qquad + 2\sum_{j,k,\ell} s_k a_k b_j (s_k a_k b_\ell+s_\ell a_\ell b_k)(s_j a_j b_\ell+s_\ell a_\ell b_j) \cdot\xi^{(2)}(s_k^2,s_j^2,s_\ell^2).
%\end{align*}
%\end{lemma}

\begin{lemma}\label{lem:phi-prime-prime}
For $h(z) = h_\lam(z)=\frac{z}{\sqrt{z^2+\lam^2}}$, it holds that $\sup_{q\in[0,1]}\abs{\phi''(q)} \lesssim \lam^{-2}$ with probability $1-e^{-\Omega(d)}$.
\end{lemma}

\begin{proof}[Proof of Lemma \ref{lem:phi-prime-prime}]
Recall the definitions of $\bG_q,\bA_q,\bS_q,\bB_q,\bY_q,\bT_q,\bR_q,a,b,\xi$ from the proof of Lemma~\ref{lem:phi-prime}. Differentiating $\bY_q$ twice with respect to $q$, we obtain
\begin{align}
\ddot \bY_q &= 2\dot \bG_q \dot \bR_q + \bG_q \ddot \bR_q \notag\\
&= 2u b^\top \bB_q^\top \dot \bR_q + \bA_q \bS_q \bB_q^\top \ddot \bR_q \label{eq:ddot-by}
\end{align}
since $\dot \bG_q=uv^\top$ and $\ddot \bG_q=0$. The derivatives of $\bR_q$ are given as
\begin{align}
    \dot \bR_q &= D\xi(\bT_q)[\dot \bT_q],\label{eq:dot-br}\\
    \ddot \bR_q &= D\xi(\bT_q)[\ddot \bT_q] + D^2\xi(\bT_q)[\dot \bT_q,\dot \bT_q].\label{eq:ddot-br}
\end{align}
For Eq.~\eqref{eq:dot-br}, we showed in the proof of Lemma~\ref{lem:phi-prime} that
\begin{align*}
(\bB_q^\top \dot \bT_q \bB_q)_{ij} &= s_i a_i b_j + s_j a_j b_i,\\
(\bB_q^\top \dot \bR_q \bB_q)_{ij} &= (s_i a_i b_j + s_j a_j b_i)\cdot \xi^{(1)}(s_i^2,s_j^2).
\end{align*}
For the first term in Eq.~\eqref{eq:ddot-br}, differentiating $\bT_q=\bG_q^\top \bG_q$ twice gives
\begin{align*}
    \ddot \bT_q = \ddot \bG_q^\top \bG_q + 2\dot \bG_q^\top \dot \bG_q + \bG_q^\top \ddot \bG_q = 2\dot \bG_q^\top \dot \bG_q = 2\norm{u}_2^2 vv^\top
\end{align*}
and so $\bB_q^\top \ddot \bT_q \bB_q = 2\norm{a}_2^2 bb^\top$. Therefore by the Daleckii--Krein formula,
\begin{align*}
    \qty(\bB_q^\top D\xi(\bT_q)[\ddot \bT_q]\bB_q)_{ij}
    = (\bB_q^\top \ddot \bT_q \bB_q)_{ij}\cdot \xi^{(1)}(s_i^2,s_j^2)
    = 2\norm{a}_2^2 b_ib_j\cdot \xi^{(1)}(s_i^2,s_j^2).
\end{align*}
For the second term in Eq.~\eqref{eq:ddot-br}, again by the Daleckii--Krein formula and using the explicit form of $\bB_q^\top \dot \bT_q \bB_q$,
\begin{align*}
    &\qty(\bB_q^\top D^2\xi(\bT_q)[\dot \bT_q,\dot \bT_q]\bB_q)_{ij}\\
    &= 2\sum_{k=1}^d \xi^{(2)}(s_i^2,s_j^2,s_k^2)(\bB_q^\top \dot \bT_q \bB_q)_{ik}(\bB_q^\top \dot \bT_q \bB_q)_{jk}\\
    &= 2\sum_{k=1}^d \xi^{(2)}(s_i^2,s_j^2,s_k^2)(s_i a_i b_k + s_k a_k b_i)(s_j a_j b_k + s_k a_k b_j).
\end{align*}
Plugging into Eq.~\eqref{eq:ddot-br} and Eq.~\eqref{eq:ddot-by} yields
\begin{align*}
\phi''(q)
&= 6\|a\|_2^2 \sum_{j,k=1}^d
s_k a_k b_k b_j^2 \cdot \xi^{(1)}(s_k^2,s_j^2)\\
&\qquad + 2\sum_{j,k,\ell=1}^d
s_k a_k b_j (s_k a_k b_\ell+s_\ell a_\ell b_k)(s_j a_j b_\ell+s_\ell a_\ell b_j) \cdot\xi^{(2)}(s_k^2,s_j^2,s_\ell^2).
\end{align*}

We proceed to bound $\nnorm{\ddot \bY_q}_{\op}$. We will use that
\begin{align*}
\sum_{k=1}^d a_k^2 = \norm{a}_2^2 = \norm{u}_2^2, \quad \sum_{k=1}^d b_k^2 = \norm{b}_2^2 = \norm{v}_2^2, \quad \sum_{k=1}^d |a_kb_k| \le \norm{a}_2\norm{b}_2
\end{align*}
are all $\Theta(1)$ uniformly over $q$ with probability $1-e^{-\Omega(d)}$.

Right-multiplying Eq.~\eqref{eq:ddot-by} by $\bB_q$, we have
\begin{align*}
    \nnorm{\ddot \bY_q}_{\op}
    &= \nnorm{\ddot \bY_q \bB_q}_{\op}\\
    &\le 2\nnorm{u b^\top \bB_q^\top \dot \bR_q \bB_q}_{\op} + \nnorm{\bA_q \bS_q \bB_q^\top \ddot \bR_q \bB_q}_{\op}\\
    &\lesssim \nnorm{\bB_q^\top \dot \bR_q \bB_q}_{\op} + \nnorm{\bS_q \bB_q^\top \ddot \bR_q \bB_q}_{\op}\\
    &\le \nnorm{\bB_q^\top \dot \bR_q \bB_q}_{\F} + \nnorm{\bS_q \bB_q^\top \ddot \bR_q \bB_q}_{\F}.
\end{align*}

For the first term,
\begin{align*}
    \nnorm{\bB_q^\top \dot \bR_q \bB_q}_{\F}^2
    &\le \sum_{i,j}(s_i a_i b_j + s_j a_j b_i)^2\cdot \xi^{(1)}(s_i^2,s_j^2)^2\\
    &\le 4\Big(\sup_{i,j} s_i\abs{\xi^{(1)}(s_i^2,s_j^2)}\Big)^2 \sum_{i,j} a_i^2 b_j^2\\
    &\lesssim \Big(\sup_{i,j} s_i\abs{\xi^{(1)}(s_i^2,s_j^2)}\Big)^2.
\end{align*}
For the second term, decompose $\ddot \bR_q$ as in Eq.~\eqref{eq:ddot-br}. First,
\begin{align*}
    \abs{s_i\qty(\bB_q^\top D\xi(\bT_q)[\ddot \bT_q]\bB_q)_{ij}}
    \lesssim \Big(\sup_{i,j} s_i\abs{\xi^{(1)}(s_i^2,s_j^2)}\Big)\abs{b_i b_j}.
\end{align*}
Next, we have from the triangle inequality,
\begin{align*}
&\abs{s_i\qty(\bB_q^\top D^2\xi(\bT_q)[\dot \bT_q,\dot \bT_q]\bB_q)_{ij}}\\
&\lesssim \sum_{k=1}^d \abs{\xi^{(2)}(s_i^2,s_j^2,s_k^2)}\cdot
    \abs{s_i(s_i a_i b_k + s_k a_k b_i)(s_j a_j b_k + s_k a_k b_j)}\\
&\le \Big(\sup_{i,j,k} s_i^2 s_j\abs{\xi^{(2)}(s_i^2,s_j^2,s_k^2)}\Big)\abs{a_i a_j}\sum_{k=1}^d b_k^2\\
&\qquad + \Big(\sup_{i,j,k} s_i^2 s_k\abs{\xi^{(2)}(s_i^2,s_j^2,s_k^2)}\Big)\abs{a_i b_j}\sum_{k=1}^d \abs{a_k b_k}\\
&\qquad + \Big(\sup_{i,j,k} s_i s_j s_k\abs{\xi^{(2)}(s_i^2,s_j^2,s_k^2)}\Big)\abs{a_j b_i}\sum_{k=1}^d \abs{a_k b_k}\\
&\qquad + \Big(\sup_{i,j,k} s_i s_k^2\abs{\xi^{(2)}(s_i^2,s_j^2,s_k^2)}\Big)\abs{b_i b_j}\sum_{k=1}^d a_k^2\\
&\lesssim \Big(\sup_{i,j,k}  s_i(s_i+s_k)(s_j+s_k)\abs{\xi^{(2)}(s_i^2,s_j^2,s_k^2)}\Big) \qty(\abs{a_i a_j} + \abs{a_i b_j} + \abs{b_i a_j} + \abs{b_i b_j}).
\end{align*}
Squaring and summing over $i,j$ gives
\begin{align*}
    \nnorm{\bS_q \bB_q^\top \ddot \bR_q \bB_q}_{\F}
    \lesssim
    \sup_{i,j} s_i\abs{\xi^{(1)}(s_i^2,s_j^2)}
    + \sup_{i,j,k} s_i(s_i+s_k)(s_j+s_k)\abs{\xi^{(2)}(s_i^2,s_j^2,s_k^2)}.
\end{align*}
Altogether, we have
\begin{align*}
    \nnorm{\ddot \bY_q}_{\op}
    \lesssim
    \sup_{i,j} s_i\abs{\xi^{(1)}(s_i^2,s_j^2)}
    + \sup_{i,j,k} s_i(s_i+s_k)(s_j+s_k)\abs{\xi^{(2)}(s_i^2,s_j^2,s_k^2)}.
\end{align*}
Finally, we evaluate this bound with our choice of $h_\lam(z)=\frac{z}{\sqrt{z^2+\lam^2}}$, which corresponds to $\xi(z)=\frac{1}{\sqrt{z+\lam^2}}$. Computing the first divided difference directly gives
\begin{align*}
\xi^{(1)}(x,y) = \frac{\xi(x)-\xi(y)}{x-y} = -\frac{1}{\sqrt{x+\lam^2}\sqrt{y+\lam^2}(\sqrt{x+\lam^2}+\sqrt{y+\lam^2})},
\end{align*}
which is valid when $x=y$ by continuity, and plugging in $s_i^2,s_j^2$ yields
\begin{align*}
    \abs{s_i\xi^{(1)}(s_i^2,s_j^2)}
    &= \frac{1}{\sqrt{s_i^2+\lam^2}\sqrt{\smash[b]{s_j^2}+\lam^2}} \times \frac{s_i}{\sqrt{s_i^2+\lam^2}+\sqrt{\smash[b]{s_j^2}+\lam^2}}
    \le \lam^{-2}
\end{align*}
for all $i,j$. Also, for the second divided difference, we obtain
\begin{align*}
    &\xi^{(2)}(x,y,z)\\
    &= \frac{\xi^1(x,z)-\xi^{(1)}(y,z)}{x-y} \\
    &= \frac{\sqrt{x+\lam^2}+\sqrt{y+\lam^2}+\sqrt{z+\lam^2}}{(\sqrt{x+\lam^2}+\sqrt{y+\lam^2})(\sqrt{x+\lam^2}+\sqrt{z+\lam^2})(\sqrt{y+\lam^2}+\sqrt{z+\lam^2})}\\
    &\qquad \times \frac{1}{\sqrt{x+\lam^2}\sqrt{y+\lam^2}\sqrt{z+\lam^2}}.
\end{align*}
It follows that for all $i,j,k$,
\begin{align*}
&s_i(s_i+s_k)(s_j+s_k)\abs{\xi^{(2)}(s_i^2,s_j^2,s_k^2)}\\
&\le \frac{s_i}{\sqrt{s_i^2+\lam^2}} \frac{s_i+s_k}{\sqrt{s_i^2+\lam^2}+\sqrt{\smash[b]{s_k^2}+\lam^2}} \frac{s_j+s_k}{\sqrt{\smash[b]{s_j^2}+\lam^2}+\sqrt{\smash[b]{s_k^2}+\lam^2}} \\
&\qquad\times \frac{\sqrt{s_i^2+\lam^2} + \sqrt{\smash[b]{s_j^2}+\lam^2} + \sqrt{\smash[b]{s_k^2}+\lam^2}}{\sqrt{\smash[b]{s_j^2}+\lam^2} \sqrt{\smash[b]{s_k^2}+\lam^2}(\sqrt{s_i^2+\lam^2} + \sqrt{\smash[b]{s_j^2}+\lam^2})} \\
&\le \frac{1}{\sqrt{\smash[b]{s_j^2}+\lam^2} \sqrt{\smash[b]{s_k^2}+\lam^2}} + \frac{1}{\sqrt{\smash[b]{s_j^2}+\lam^2}(\sqrt{s_i^2+\lam^2} + \sqrt{\smash[b]{s_j^2}+\lam^2})} \\
&\le \frac32 \lam^{-2}.
\end{align*}
Therefore for this choice of $\xi$, we conclude that $\nnorm{\ddot \bY_q}_{\op}\lesssim \lam^{-2}$ and thus
\begin{align*}
\abs{\phi''(q)} = |u^\top\ddot\bY_q v| \lesssim \lam^{-2}
\end{align*}
for all $q\in[0,1]$.
\end{proof}

\subsection{Minibatch concentration}\label{sec:minibatch}

Here, we collect concentration inequalities for the minibatch frequencies $q$ which will be needed in later sections.

\begin{lemma}\label{lem:spectrum-new}
Define the weighted covariance matrix
\begin{align*}
\bM:=\sum_{j\in[N]} q^2_j u_ju_j^\top.
\end{align*}
It holds with probability $1-O(d^{-M})$ over sampling of $q$ that
\begin{align*}
    \lambda_{d/2}(\bM) \lesssim \begin{cases}
        d^{-2\alpha}(\log d)^2 & B\gtrsim d^\alpha,\\
        0 & B\lesssim d^\alpha.
    \end{cases}
\end{align*}
\end{lemma}
Hereafter, we use $M$ to denote any sufficiently large constant exponent that allows for union bounding over (say) $N^2=\poly(d)$ items, and often omit the qualifying high probability statements.

%For multi-step, set all modulated terms to max, so that we indeed have deterministic `fully classified' barrier $d_1=i^*$ and the rest is capped by $q_i^{(t)}$ which are independent of $u_i$s. Tail argument is also fine since the sparsity pattern only depends on $q_i^{(t)}$.

\begin{proof}
Let~$N$ be the number of examples satisfying $i\ge\frac{d}{4}$ in a minibatch of size~$B$. $N$ is distributed as $\Bin(B,\rho)$ where $\rho = \sum_{i \ge d/4} p_i\asymp d^{1-\alpha}$. From the multiplicative Chernoff bound, it holds that for all $\eps>0$,
\begin{align*}
\Pr\qty(N\ge (1+\eps)B\rho) \le \exp\qty(-\frac{\eps^2 B\rho}{2+\eps}).
\end{align*}
If $B\lesssim d^\alpha$ so that $B\rho\le \frac{d}{8}$, by taking $\eps = \frac{d}{4B\rho} - 1$ we have
\begin{align*}
\Pr(N\ge \frac{d}{4}) \le \exp\qty(-\frac{(d-4B\rho)^2}{4(d+4B\rho)}) \le e^{-\Omega(d)}.
\end{align*}
Hence with high probability, $N<\frac{d}{4}$ so that the total number of nonzero $q_i$ is less than $\frac{d}{2}$. It follows that $\rank(\bM)<\frac{d}{2}$ and so $\lambda_{d/2}(\bM)=0$.

Now suppose $B\gtrsim d^\alpha$. Choose a positive integer $K\asymp \frac1d B^{1/\alpha}$ and define the sets $I_k := \{(k-1)d+\frac{d}{2}, \cdots, kd+\frac{d}{2}-1\}$ for $k\ge 1$. Consider the decomposition
\begin{align*}
\bM = \underbrace{\sum_{i=1}^{d/2-1} q_i^2 u_iu_i^\top}_{=:\bM_0} + \sum_{k\in[K]} \underbrace{\sum_{i\in I_k} q_i^2u_iu_i^\top}_{=:\bM_k} + \underbrace{\sum_{i= (K+1/2)d}^N q_i^2u_iu_i^\top}_{=:\bM_{\tail}}.
\end{align*}
Since $\rank(\bM_0)<\frac{d}{2}$, we have $\lambda_{d/2}(\bM_0)=0$. By Weyl's inequality,
\begin{align}\label{eq:m-decomp}
\lambda_{d/2}(\bM) \le \lambda_{d/2}(\bM_0) + \norm{\sum_{k\in[K]} \bM_k + \bM_{\tail}}_{\op} \le \sum_{k\in[K]} \norm{\bM_k}_{\op} + \norm{\bM_{\tail}}_{\op}.
\end{align}
We first control the bulk sum. Since $|I_k|\le d$, it follows from \citet[Theorem 4.6.1]{vershynin2018high} that 
\begin{align*}
\norm{\sum_{i\in I_k} u_iu_i^\top}_{\op} = O(1) \quad\implies\quad \norm{\bM_k}_{\op} \lesssim \max_{i\in I_k} q_i^2
\end{align*}
with probability $1-e^{-\Omega(d)}$. To bound this quantity, set $\bar{p}_k := \max_{i\in I_k} p_j \asymp (k-\frac12)^{-\alpha}d^{-\alpha}$ and note that $p_i \ge 3^{-\alpha} \bar{p}_k$ for all $i\in I_k$. By the Chernoff bound for $Bq_i\sim\Bin(B,p_i)$,
\begin{align}\label{eq:chernoff-indiv}
\Pr\qty(q_i \ge (1+\eps)p_i) \le \exp\qty(-\frac{Bp_i\eps^2}{2+\eps})
\end{align}
and so union bounding over $i\in I_k$, we have
\begin{align*}
\Pr\qty(\max_{i\in I_k} q_i \ge (1+\eps) \bar{p}_k) &\le d \exp\qty(-\frac{3^{-\alpha}B\bar{p}_k\eps^2}{2+\eps}) \lesssim \frac{1}{d^M}
\end{align*}
by taking $\eps \gtrsim \frac{\log d}{B\bar{p}_k} \vee \sqrt{\frac{\log d}{B\bar{p}_k}}$. Hence for all $k\in[K]$ we have
\begin{align*}
\max_{i\in I_k} q_i \lesssim (1+\eps)\bar{p}_k \lesssim \bar{p}_k + \frac{\log d}{B}.
\end{align*}

For the tail sum, we exploit the sparsity of the frequencies~$q_i$. Define the set of indices
\begin{align*}
I_{\tail} := \left\{i: \qty(K+\frac12)d\le i\le N, \; q_i>0 \right\}.
\end{align*}
$N_{\tail} := |I_{\tail}|$ is distributed as $\Bin(B,\rho_{\tail})$ where $\rho_{\tail} = \sum_{i\ge (K+1/2)d}p_i \asymp B^{(1-\alpha)/\alpha}$, so that $N_{\tail} \asymp B\rho_{\tail}\asymp B^{1/\alpha}$ with probability $1-e^{-\Omega(d)}$. Moreover for each $i\in I_{\tail}$, it holds that $p_i \lesssim (Kd)^{-\alpha} \asymp 1/B$ and so
\begin{align*}
\Pr\qty(Bq_i\ge r) \le \binom{B}{r} p_i^r \le \qty(\frac{eBp_i}{r})^r = d^{-\omega(1)}
\end{align*}
by taking $r\asymp\log d$, hence $q_i\lesssim \frac{\log d}{B}$. It follows from \citet[Remark 4.7.3]{vershynin2018high} that
\begin{align*}
\norm{\bM_{\tail}}_{\op} &\le \max_{i\in I_{\tail}} q_i^2\cdot \norm{\sum_{i\in I_{\tail}} u_iu_i^\top}_{\op}\\
&\lesssim \qty(\frac{\log d}{B})^2 \frac{N_{\tail}}{d} \qty(1+ \sqrt{\frac{d}{N_{\tail}}} + \frac{d}{N_{\tail}})\\
&\lesssim \qty(\frac{\log d}{B})^2 \frac{B^{1/\alpha}}{d}
\end{align*}
since $B\gtrsim d^\alpha$. We conclude from Eq.~\eqref{eq:m-decomp}:
\begin{align*}
\lambda_{d/2}(\bM) \lesssim \sum_{k\in[K]} \qty(\bar{p}_k + \frac{\log d}{B})^2 + \qty(\frac{\log d}{B})^2 \frac{B^{1/\alpha}}{d} \lesssim d^{-2\alpha} (\log d)^2,
\end{align*}
where we have used that $\sum_{k\ge 1}\bar{p}_k^2 \asymp \sum_{k\ge 1} (kd)^{-2\alpha} \asymp d^{-2\alpha}$.
\end{proof}

\begin{lemma}[concentration of tail frequencies]\label{lem:mini-truncated}
Let $\cB$ be a randomly sampled minibatch with empirical frequencies $q$. Let $r$ be any integer and denote $q_{>r} = (q_{r+1},\cdots,q_N)$. Then it holds with probability $1-O(d^{-M})$ that
\begin{align*}
\norm{q_{>r}}_\infty &\lesssim r^{-\alpha} +\frac{\log d}{B}
\end{align*}
and
\begin{align*}
\norm{q_{>r}}_2 &\lesssim \begin{cases}
r^{1/2-\alpha}\log d & B\gtrsim r^\alpha, \\ \displaystyle
\sqrt{\frac{r^{1-\alpha}}{B}} \log d & B\lesssim r^\alpha.
\end{cases}
\end{align*}
\end{lemma}

\begin{proof}
We first control $\norm{q_{>r}}_\infty$. Let the index $i>r$ so that $p_i\lesssim r^{-\alpha}$. Recalling the Chernoff bound Eq.~\eqref{eq:chernoff-indiv} for $Bq_i\sim\Bin(B,p_i)$, we may choose
\begin{align*}
\eps \gtrsim 1+\frac{\log d}{Bp_i}
\end{align*}
such that
\begin{align*}
\Pr\qty(q_i \ge (1+\eps)p_i) \le \exp\qty(-\frac{Bp_i\eps^2}{2+\eps}) \le \exp\qty(-\frac{Bp_i\eps}{3}) \le \frac{1}{d^M},
\end{align*}
which implies
\begin{align*}
q_i\lesssim \qty(1+\frac{\log d}{Bp_i})p_i \lesssim r^{-\alpha} + \frac{\log d}{B}
\end{align*}
for all $r<i\le N$.

For $\norm{q_{>r}}_2$, we repeat the analysis from the proof of Lemma~\ref{lem:spectrum-new}. If $B\gtrsim r^\alpha$, thresholding at~$B^{1/\alpha}$ gives $q_i \lesssim p_i + \frac{\log d}{B}$ for items with $r<i\le B^{1/\alpha}$, and $q_i\lesssim \frac{\log d}{B}$ for the $N_{\tail}\asymp B^{1/\alpha}$ items with $i>B^{1/\alpha}$. Combining, we obtain
\begin{align*}
\norm{q_{>r}}_2^2 &\lesssim \sum_{i=r+1}^{B^{1/\alpha}} \qty(p_i + \frac{\log d}{B})^2 + N_{\tail} \qty(\frac{\log d}{B})^2 \\
&\lesssim r^{1-2\alpha} + B^{1/\alpha-2}(\log d)^2 \\
&\lesssim r^{1-2\alpha}(\log d)^2.
\end{align*}
Finally, if $B\lesssim r^\alpha$, we may treat all items $i>r$ as in the tail, so that $N_{\tail}\sim \Bin(B,\rho_{\tail})$ with $\rho_{\tail} = \sum_{i>r} p_i \asymp r^{1-\alpha}$ and $N_{\tail} \asymp B\rho_{\tail} \asymp Br^{1-\alpha}$. Hence
\begin{align*}
\norm{q_{>r}}_2^2 &\lesssim N_{\tail} \qty(\frac{\log d}{B})^2 \lesssim \frac{r^{1-\alpha}}{B}(\log d)^2,
\end{align*}
as was to be shown.
\end{proof}

As a corollary, we prove the following lemma which will be used in Section~\ref{sec:interaction}.

\begin{lemma}\label{lem:lam}
Let $r\asymp \frac{d}{(\log d)^2}$. There exists
\begin{align}\label{eq:lam}
\lam \asymp \max\left\{\frac{(\log d)^{2\alpha+2}}{d^\alpha}, \frac{(\log d)^2}{B} \right\}
\end{align}
such that the event
\begin{align}
\cE_q &\,:\, \max\left\{\norm{q_{>r}}_\infty, \frac{\norm{q_{>r}}_2}{\sqrt{d}}\right\} \le \lam\sqrt{\frac{r}{d}}\, \label{eq:event-q}
\end{align}
satisfies $\Pr(\cE_q) \ge 1-O(d^{-M})$.
\end{lemma}

\begin{proof}
By Lemma~\ref{lem:mini-truncated}, we have that when $B\gtrsim r^\alpha$,
\begin{align*}
\max\left\{\norm{q_{>r}}_\infty, \frac{\norm{q_{>r}}_2}{\sqrt{d}}\right\} &\lesssim r^{-\alpha} +\frac{\log d}{B} + \sqrt{\frac{r}{d}} \cdot r^{-\alpha}\log d \lesssim d^{-\alpha}(\log d)^{2\alpha+1}
\end{align*}
and when $B\lesssim r^\alpha$,
\begin{align*}
\max\left\{\norm{q_{>r}}_\infty, \frac{\norm{q_{>r}}_2}{\sqrt{d}}\right\} &\lesssim r^{-\alpha} +\frac{\log d}{B} + \sqrt{\frac{r^{1-\alpha}}{Bd}}\log d \lesssim \frac{\log d}{B}.
\end{align*}
Therefore by choosing $\lam$ as in Eq.~\eqref{eq:lam} with an appropriate proportionality constant, we can ensure that $\cE_q$ occurs with probability $1-O(d^{-M})$.
\end{proof}

\subsection{Lower bounding the signal}\label{sec:signal}

We now consider the stabilized sign map $h_\lam(z) = \frac{z}{\sqrt{z^2+\lam^2}}$. Recall from Lemma~\ref{lem:phi-prime} that
\begin{align*}
    \phi'(0) &= \frac14 \sum_{k \neq \ell}\qty( \frac{h_\lam(s_k) + h_\lam(s_\ell)}{s_k + s_\ell} (a_kb_\ell - a_\ell b_k)^2 + \frac{h_\lam(s_k) - h_\lam(s_\ell)}{s_k - s_\ell} (a_kb_\ell + a_\ell b_k)^2)\\
    &\qquad + \sum_k h_\lam'(s_k) a_k^2 b_k^2.
\end{align*}

Since~$\frac{h_\lam(s_k) - h_\lam(s_\ell)}{s_k - s_\ell}$ is always positive by the mean value theorem and $h_\lam(z)/z$ is decreasing, we may lower bound~$\phi'(0)$ as
\begin{align}
    \phi'(0) &\ge \frac14 \sum_{k \neq \ell} \frac{h_\lam(s_k) + h_\lam(s_\ell)}{s_k + s_\ell}(a_kb_\ell - a_\ell b_k)^2 \notag\\
    &\ge \frac12 \sum_{d/2\le k<\ell} \frac{h_\lam(s_k) + h_\lam(s_\ell)}{s_k + s_\ell}(a_kb_\ell - a_\ell b_k)^2\notag \\
    &\ge \frac{h_\lam(s_{d/2})}{2s_{d/2}} \sum_{d/2\le k<\ell} (a_kb_\ell - a_\ell b_k)^2.\label{eq:abab}
\end{align}

Note that~$a,b$ are Gaussian conditioned on~$\bG_{-i}$. Let $a',b'$ be the vectors consisting of the last $\frac{d}{2}$ coordinates of $a,b$, respectively. It holds that $a',b'\sim \cN(0, \frac1d\bI_{d/2})$ and
\begin{align}\label{eq:halfsum}
\sum_{d/2\le k<\ell} (a_kb_\ell - a_\ell b_k)^2 = \norm{a'}^2\norm{b'}^2 - \langle a',b'\rangle^2.
\end{align}
Standard concentration bounds give $\norm{a'}^2, \norm{b'}^2 =\Theta(1)$ while $\langle a',b'\rangle^2 \lesssim \frac{\log d}{d}$ with probability $1-O(d^{-M})$, hence Eq.~\eqref{eq:halfsum} is lower bounded by a constant.

It thus suffices to control the `bulk' singular value $s_{d/2}$. By the lemma below, this can be reduced to controlling the bulk eigenvalues of the weighted covariance matrix
\begin{align}\label{eq:marugame}
\bM:=\sum_{j\in[N]} q^2_j u_ju_j^\top.
\end{align}

\begin{lemma}\label{lem:symmetrize}
It holds that $s_k(\bG_{-i}) \lesssim \lambda_k(\bM)^{1/2}$ for all $i, k$ with probability $1-e^{-\Omega(d)}$.
\end{lemma}
\begin{proof}
Let $\bM_{-i}:=\sum_{j\ne i} q^2_j u_ju_j^\top$. The $k$th column of $\bG_{-i}$ is $\sum_{j \neq i} q_j u_j v_{jk}$, which for each~$k$ is an i.i.d. sample from $\cN(0, \frac1d\bM_{-i})$ conditioned on $u_1,\cdots,u_N$. Therefore
\begin{align}\label{eq:z-rep}
\bG_{-i} \deq \frac{1}{\sqrt{d}} \bM_{-i}^{1/2}\bZ, \quad\text{where}\quad \bZ_{k\ell} \sim \cN(0,1) \;\;\text{i.i.d.}
\end{align}
It holds that \citep[Eq.~2.3]{rudelson2010nonasymptotictheoryrandommatrices}
\begin{align*}
\Pr \qty(\norm{\bZ}_{\op} \le 3\sqrt{d}) \ge 1-2e^{-d/2}
\end{align*}
and thus
    \begin{align*}
        s_k(\bG_{-i}) \le \frac{1}{\sqrt{d}} s_k\qty(\bM_{-i}^{1/2}) \norm{\bZ}_{\op} \lesssim \lambda_k(\bM_{-i})^{1/2} \le \lambda_k(\bM)^{1/2}
    \end{align*}
since $\lambda_k(\cdot)$ respects Loewner order.
\end{proof}

Then by Lemma~\ref{lem:spectrum-new}, we have $\lam_{d/2}(\bM) \lesssim d^{-2\alpha}(\log d)^2$ and so $s_{d/2}(\bG_{-i}) \lesssim d^{-\alpha}\log d \lesssim\lam$. We conclude from Eq.~\eqref{eq:abab},
\begin{align*}
\phi'(0) \ge \frac{h_\lam(s_{d/2})}{2s_{d/2}} \sum_{k<l\le d/2} (a_kb_l + a_lb_k)^2 \gtrsim \frac{1}{\sqrt{\smash[b]{s_{d/2}^2}+\lam^2}} \gtrsim \frac{1}{\lam}.
\end{align*}
We have also shown that $\sup_{q\in[0,1]}\abs{\phi''(q)} \lesssim \lam^{-2}$ in Lemma~\ref{lem:phi-prime-prime}. In addition, since $u_i,v_i$ are independent of $\bG_{-i}$ and $\nnorm{h_\lam(\bG_{-i})}_\op \le 1$ from Proposition~\ref{prop:operator-lipschitz}, a standard concentration bound for subexponential sums \citep[Lemma 2.8.6 and Corollary 2.9.2]{vershynin2018high} gives that with probability $1-O(d^{-M})$,
\begin{align*}
\abs{\phi(0)} = \abs{u_i^\top h_\lam(\bG_{-i}) v_i} \lesssim \sqrt{\frac{\log d}{d}}.
\end{align*}
Since $\phi$ is increasing by Lemma~\ref{lem:phi-prime}, we can therefore Taylor expand $\phi$ to obtain
\begin{align*}
    \phi(q) \ge \phi(t) \ge \phi(0) + t\phi'(0) - \frac12 t^2 \sup_{0\le s\le t} |\phi''(s)| \gtrsim \phi(0) + \frac{t}{\lam} - \frac{t^2}{\lam^2}.
\end{align*}
Finally, taking the supremum over $t\in[0,q]$ gives
\begin{align}\label{eq:signal-guarantee}
\gamma_{ii} = \phi(q_i) \gtrsim \min\left\{\frac{q_i}{\lam}, 1\right\} - O\qty(\sqrt{\frac{\log d}{d}}).
\end{align}

\subsection{Putting things together}\label{sec:putting}

In Section~\ref{sec:interaction}, we analyze the interaction terms and show that under $\cE_q$ (Proposition~\ref{prop:interaction}),
\begin{align*}
|\gamma_{ij}| \lesssim \frac{(\log d)^3}{\sqrt{d}}, \quad\forall i\ne j.
\end{align*}
Combining with Eq.~\eqref{eq:signal-guarantee}, the uncentered logit gap is thus lower bounded as
\begin{align*}
\gamma_{ii}- \max_{j\ne i}\gamma_{ij} \gtrsim \min\left\{\frac{q_i}{\lam},1\right\} - O\qty(\frac{(\log d)^3}{\sqrt{d}}).
\end{align*}
We now show that centering does not affect the computation. The mean vector is distributed as $\bar{u} \sim\cN(0,\frac{1}{Nd}\bI_d)$ so that $\nnorm{\bar{u}}_2 \lesssim 1/\sqrt{N}$, and moreover $\nnorm{u_i}_2, \nnorm{v_i}_2\lesssim 1$ for all $i\in[N]$ with probability $1-e^{-\Omega(d)}$. It follows that
\begin{align}\label{eq:gradient-appx}
\nnorm{\bG_0 - \bG}_\op \le \sum_{i\in[N]} q_i\nnorm{\bar{u}}_2 \nnorm{v_i}_2 \lesssim \frac{1}{\sqrt{N}}.
\end{align}
By Proposition~\ref{prop:operator-lipschitz}, for all $i,j$,
\begin{align*}
|u_j^\top h_\lam(\bG_0) v_i - u_j^\top h_\lam(\bG) v_i| &\lesssim \nnorm{h_\lam(\bG_0) - h_\lam(\bG)}_\op \\
&\le \frac{1}{\lam}\nnorm{\bG_0 - \bG}_\op \lesssim \frac{d^\alpha}{\sqrt{N}} \lesssim \frac{1}{\sqrt{d}}.
\end{align*}
Thus we also have
\begin{align*}
u_i^\top h_\lam(\bG_0)v_i - \max_{j\ne i} u_j^\top h_\lam(\bG_0)v_i \gtrsim \min\left\{\frac{q_i}{\lam},1\right\} - O\qty(\frac{(\log d)^3}{\sqrt{d}}).
\end{align*}
We conclude that item~$i$ will be recovered (regardless of the scaling $\eta$) if
\begin{align*}
q_i \gtrsim \frac{(\log d)^3}{\sqrt{d}} \lam \asymp \max\left\{\frac{(\log d)^{2\alpha+5}}{d^{\alpha+1/2}}, \frac{(\log d)^5}{B\sqrt{d}} \right\}.
\end{align*}
In the population regime ($B=\infty$), taking $q_i=p_i\asymp i^{-\alpha}$, we hence recover all items up to
\begin{align*}
i\le i^\star \asymp d^{1+\frac{1}{2\alpha}} (\log d)^{-2-\frac{5}{\alpha}}.
\end{align*}
If $q_i$ are obtained from a minibatch of size $B$, we have from the Chernoff lower bound that $\Pr(q_i \le \frac12 p_i) \le \exp(-\frac12 Bp_i) \le d^{-M}$ for $i$ such that $p_i\gtrsim B^{-1}\log d$, which also ensures
\begin{align*}
q_i \ge \frac{p_i}{2} \gtrsim \max\left\{\frac{(\log d)^{2\alpha+5}}{d^{\alpha+1/2}}, \frac{(\log d)^5}{B\sqrt{d}} \right\}
\end{align*}
for all $i\lesssim i^\star$. Therefore with probability $1-O(d^{-M})$, we recover all items up to
\begin{align}\label{eq:final-recovery}
i \lesssim \min\left\{i^\star, \qty(\frac{B}{\log d})^{1/\alpha}\right\}.
\end{align}

\subsection{Proof of Corollary~\ref{cor:muon-loss}}

Let $\hp_1 := \hp_{\bW_1}$ be the predicted score under $\bW_1$. By choosing $\eta\asymp (\log d)^{-4}\sqrt{d}$, we can guarantee a logit gap of
\begin{align*}
u_i^\top \bW_1 v_i - \max_{j\ne i} u_j^\top \bW_1 v_i \gtrsim \eta \qty(\min\left\{\frac{q_i}{\lam},1\right\} - \frac{(\log d)^3}{\sqrt{d}}) \gtrsim (\log d)^2
\end{align*}
for all items~$i$ satisfying Eq.~\eqref{eq:final-recovery} (up to an additional polylog factor), which implies that $\hp_1(i\mid i) = 1 - d^{-\omega(1)}$. We denote these items as $i\le i'$. For all other items, it holds that
\begin{align*}
u_i^\top \bW_1 v_i - \max_{j\ne i} u_j^\top \bW_1 v_i \gtrsim \eta \qty(- \frac{(\log d)^3}{\sqrt{d}}) \gtrsim -\frac{1}{\log d},
\end{align*}
and so $\hp_1(i\mid i) \ge \frac{1-o(1)}{N}$ and $\hp_1(j\mid i) \le \frac{1+o(1)}{N}$ for all $j\ne i$. It follows that
\begin{align*}
L(\bW_1) = \mathbb{E}_{i \sim p} [-\log \hp_1(i \mid i)] &\lesssim d^{-\omega(1)} + \sum_{i>i'} p_i \log N \\
&\lesssim d^{-\omega(1)} + (i')^{1-\alpha} \log d \\
& = \widetilde{O} \qty(\max\left\{d^{\frac12+\frac{1}{2\alpha}-\alpha}, B^{\frac{1}{\alpha} -1}\right\}).
\end{align*}

\clearpage
\section{Analysis of Interaction Terms}\label{sec:interaction}

\subsection{Overview}\label{sec:overview}

In this section, we show the following result for the interaction terms.

\begin{proposition}\label{prop:interaction}
Fix a threshold $r\asymp \frac{d}{(\log d)^2}$. Under the event
\begin{align*}
\cE_q &\,:\, \max\left\{\norm{q_{>r}}_\infty, \frac{\norm{q_{>r}}_2}{\sqrt{d}}\right\} \le \lam\sqrt{\frac{r}{d}}\, ,
\end{align*}
it holds with probability $1-d^{-\omega(1)}$ that for all pairs $i\ne j$ of distinct indices,
\begin{align*}
|\gamma_{ij}| &\lesssim \frac{(\log d)^3}{\sqrt{d}}.
\end{align*}
\end{proposition}
We have verified $\cE_q$ occurs with high probability by a judicious choice of $\lam$ in Lemma~\ref{lem:lam}, and assume this for fixed $q$ throughout the section. When either $q_i$ or $q_j\ll\lam$, the interaction terms can be bounded by a simple operator Lipschitz concentration argument, which we provide in Section~\ref{sec:gl}. The main challenge arises when controlling the leading $r\times r$ block, where any operator norm bound fails to capture the correct scale. The analysis for these `large' interactions requires a much more involved perturbative approach, and will be developed throughout Sections~\ref{sec:large-setup}-\ref{sec:large-hyper}. For the readers' convenience, we provide a sketch of the argument here.

Gather the top $r\asymp \frac{d}{(\log d)^2}$ items into $\bU=[u_1 \;\cdots\; u_r]$, $\bV = [v_1 \;\cdots\; v_r]$ and $\bQ = \diag(q_1,\cdots,q_r)$. We need to bound the off-diagonal entries of $\bK := \bU^\top h_\lam(\bG)\bV$, where the gradient $\bG$ is split into
\begin{align*}
\bG = \bU\bQ\bV^\top + \bZ, \quad \bZ :=\sum_{\ell=r+1}^N q_\ell u_\ell v_\ell^\top.
\end{align*}
In the limiting regime $r/d\to 0$, we can ensure that the Gram matrices $\bG_u = \bU^\top\bU$, $\bG_v = \bV^\top\bV$ are approximately identity (Section~\ref{sec:large-setup}). Utilizing equivariance of~$h_\lam$ and isotropicity of the tail~$\bZ$ given~$\bU,\bV$, we rewrite~$\bK$ in this near-orthonormal basis as
\begin{align*}
\bK \deq \bmat{\bG_u^{1/2}&0} h_\lam\qty(\bmat{\bG_u^{1/2} \bQ \bG_v^{1/2} &\\&0}+\bZ) \bmat{\bG_v^{1/2}\\0},
\end{align*}
which is a perturbation of the top $r\times r$ block of $h_\lam(\bQ)$. We then invoke the resolvent representation $\bX^{-1/2} = \frac{1}{\pi}\int_0^\infty s^{-1/2} (\bX + s\bI_d)^{-1} \rd s$ to get rid of the inverse square root in~$h_\lam$, and expand all fractional powers and inverses in terms of the error matrices $\bE_u = \bG_u-\bI_r$, $\bE_v = \bG_v-\bI_r$ and $\tilde{\bZ} = \lam^{-1}\bZ$ (Section~\ref{sec:large-resolvent}). This yields the expression (omitting series truncations, which are controlled in Section~\ref{sec:large-truncation})
\begin{align}\label{eq:kk}
\bK = \frac{1}{\pi} \int_0^\infty s^{-1/2} \bmat{\bI_r&0} \bH \bD_s^{-1/2} \sum_{k\ge 0} \qty(\bD_s^{-1/2}\bDelta \bD_s^{-1/2})^k \bD_s^{-1/2} \bmat{\bI_r\\0} \rd s,
\end{align}
where $\bD_s =\begin{bsmallmatrix}\bQ^2&\\&0 \end{bsmallmatrix} +(\lam^2+s)\bI_d$ is diagonal dependent on~$s$ and $\bH,\bDelta$ are perturbations, e.g., the expansion for~$\bH$ is
\begin{align*}
\bH &= \bmat{\bQ+\bE_u\bQ & \\ & 0} + \lam \sum_{k,\ell\ge 0} \binom{\frac12}{k}\binom{-\frac12}{\ell} \bmat{\bE_u &\\& 0}^k \tilde\bZ \bmat{\bE_v &\\& 0}^\ell.
\end{align*}
We further expand Eq.~\eqref{eq:kk} entrywise over all summed factors in $\bH,\bDelta$ (recorded as symbols $\mu,\nu$) and also over all valid index paths $\iota$, into products~$\bT_\iota^{\mu,\nu}$ of entries of $\bE_u,\bE_v,\tilde{\bZ}$. Integrating out~$s$ in the coefficients gives the complete expansion~$\bK_{ij} = \sum \theta_\iota^{\mu,\nu} \bT_\iota^{\mu,\nu}$. Along the way, we prove two crucial results: (1)~all integrated coefficients~$|\theta_\iota^{\mu,\nu}| \le 1$ (Section~\ref{sec:large-full}); and (2)~every pair of monomials $\bT_\iota^{\mu,\nu},\bT_{\iota'}^{\mu',\nu'}$ are nonnegatively correlated (Section~\ref{sec:large-path}). This lets us strip away all coefficients to construct an \emph{isotropic} perturbation~$\hat\bK_{ij} := \sum \bT_\iota^{\mu,\nu}$ which upper bounds~$\bK_{ij}$:
\begin{align*}
%\bK = \sum_{\iota,\mu,\nu} \theta_\iota^{\mu,\nu} \bT_\iota^{\mu,\nu} \quad\Longrightarrow\quad
\E[\bK_{ij}^2] = \sum \theta_\iota^{\mu,\nu}\theta_{\iota'}^{\mu',\nu'} \E[\bT_\iota^{\mu,\nu} \bT_{\iota'}^{\mu',\nu'}] \le \sum \E[\bT_\iota^{\mu,\nu} \bT_{\iota'}^{\mu',\nu'}] = \E[\hat\bK_{ij}^2].
\end{align*}
This new object~$\hat\bK$ is essentially equivalent to removing all scalar coefficients of $\bH,\bDelta$ and factors of~$\bD_s$ in the computation of Eq.~\eqref{eq:kk}. Importantly, unlike~$\bK$, the off-diagonal entries of $\hat\bK$ are now distributionally invariant. Furthermore, its higher moments can be controlled (after sorting by degree) using standard moment methods, i.e., Gaussian hypercontractivity and decay estimates for $\bE_u,\bE_v,\tilde{\bZ}$ (Section~\ref{sec:large-hyper}). This finally yields the desired upper bound
\begin{align*}
|\gamma_{ij}| = |\bK_{ij}| \lesssim \frac{(\log d)^3}{\sqrt{d}}.
\end{align*}

\subsection{Setup and norm estimates}\label{sec:large-setup}

We use the notation
\begin{align*}
\bU &= \bmat{u_1 \;\cdots\; u_r}, \bV = \bmat{v_1 \;\cdots\; v_r} \in\RR^{d\times r},\quad \bQ = \diag(q_1,\cdots,q_r)
\end{align*}
and $\bK := \bU^\top h_\lam(\bG)\bV$, so that
\begin{align*}
\bG = \sum_{\ell=1}^N q_\ell u_\ell v_\ell^\top = \bU\bQ\bV^\top + \underbrace{\sum_{\ell=r+1}^N q_\ell u_\ell v_\ell^\top}_{=:\bZ}.
\end{align*}
Our goal is to prove Proposition~\ref{prop:interaction} for the case $i,j\le r$, which corresponds to bounding the off-diagonal entries of $\bK$. Set
\begin{align*}
\bG_u &= \bU^\top\bU,\quad \bE_u = \bG_u - \bI_r, \\
\bG_v &= \bV^\top\bV,\quad \bE_v = \bG_v - \bI_r.
\end{align*}
We require the following decay estimates.

\begin{lemma}[decay estimate for $\bE_u,\bE_v$]\label{lem:decay-e}
There exists a constant $C>0$ such that\begin{align*}
\Pr\qty(\norm{\bE_u}_\op > C\max\left\{ \frac{\sqrt{r}+t}{\sqrt{d}}, \qty(\frac{\sqrt{r}+t}{\sqrt{d}})^2\right\}) \le 2e^{-t^2}, \quad\forall t\ge 0
\end{align*}
and similarly for $\bE_v$. In particular, it holds with probability $1-e^{-\Omega(r)}$ that
\begin{align*}
\norm{\bE_u}_\op, \norm{\bE_v}_\op \lesssim \sqrt{\frac{r}{d}}.
\end{align*}
\end{lemma}

\begin{proof}
See Theorem 4.6.1 of \citet{vershynin2018high}.
\end{proof}

\begin{lemma}[decay estimate for $\bZ$]\label{lem:decay-z}
Denote $q_{>r} = (q_{r+1},\cdots,q_N)\in [0,1]^{N-r}$. There exist constants $C,t_0>0$ such that
\begin{align*}
\Pr\qty(\norm{\bZ}_\op > \max\left\{\norm{q_{>r}}_\infty, \frac{\norm{q_{>r}}_2}{\sqrt{d}}\right\}t) \le e^{Cd(t_0-t)}, \quad\forall t\ge t_0.
\end{align*}
In particular, it holds with probability $1-e^{-\Omega(d)}$ that
\begin{align*}
\norm{\bZ}_\op \lesssim \max\left\{\norm{q_{>r}}_\infty, \frac{\norm{q_{>r}}_2}{\sqrt{d}}\right\}.
\end{align*}
\end{lemma}

\begin{proof}
For fixed vectors $x,y\in\SS^{d-1}$, $\sqrt{d}x^\top u_i$ and $\sqrt{d}y^\top v_i$ are each $\cN(0,1)$ so that $\xi_i: = d(x^\top u_i) (v_i^\top y)$ is subexponential with $\nnorm{\xi_i}_{\psi_1} = O(1)$. Then
\begin{align*}
x^\top\bZ y = \sum_{i=r+1}^N q_i (x^\top u_i)(v_i^\top y) = \frac1d \sum_{i=r+1}^N q_i \xi_i
\end{align*}
satisfies
\begin{align*}
\Pr\qty(|x^\top\bZ y| \ge \tau) \le 2\exp\qty(-C\min\left\{\frac{d^2\tau^2}{\norm{q_{>r}}_2^2}, \frac{d\tau}{\norm{q_{>r}}_\infty}\right\})
\end{align*}
by the subexponential Bernstein inequality. Taking
\begin{align*}
\tau = \max\left\{\norm{q_{>r}}_\infty, \frac{\norm{q_{>r}}_2}{\sqrt{d}}\right\} t
\end{align*}
for some $t>0$, it follows that $\Pr(|x^\top\bZ y| \ge t) \le 2e^{-Cd(t\wedge t^2)}$. Now choose a $1/4$-net $\cM$ of $\SS^{d-1}$ with size $|\cM| \le 9^d$. It holds that
\begin{align*}
\norm{\bZ}_\op = \sup_{x,y\in\SS^{d-1}} |x^\top \bZ y| \le \sup_{x,y\in\cM} |x^\top \bZ y| + \frac{1}{2} \norm{\bZ}_\op
\end{align*}
and so union bounding over $\cM$,
\begin{align*}
\Pr\qty(\norm{\bZ}_\op > \max\left\{\norm{q_{>r}}_\infty, \frac{\norm{q_{>r}}_2}{\sqrt{d}}\right\}t) \le 2\cdot 9^d \cdot e^{-Cd(t\wedge t^2)} \le e^{Cd(t_0-t)}
\end{align*}
for constants $C,t_0$. The last claim follows by taking $t=2t_0$.
\end{proof}
Now define the decay factor $\rho$ as
\begin{align}\label{eq:rhodef}
\rho\asymp\sqrt{\frac{r}{d}} \asymp \frac{1}{\log d}
\end{align}
and the event $\cE_\op$ as
\begin{align}
\cE_\op &\,:\, \max\left\{\norm{\bE_u}_\op, \norm{\bE_v}_\op, \frac{\norm{\bZ}_\op}{\lam}\right\} \le \rho. \label{eq:decay}
\end{align}
From Lemma~\ref{lem:decay-e} and Lemma~\ref{lem:decay-z}, under the event $\cE_q$, we can choose the proportionality constant in Eq.~\eqref{eq:rhodef} so that $\Pr(\cE_\op)\ge 1-e^{-\Omega(r)}$. We note that the truncated series expansions in Section~\ref{sec:large-resolvent}-\ref{sec:large-truncation} are valid conditional on $\cE_\op$, however once we algebraically reduce to the appropriate quantities, we do not condition on~$\cE_\op$ for the moment computations in Section~\ref{sec:large-full}-\ref{sec:large-hyper}.

Under $\cE_\op$, we also have the following bounds:

\begin{lemma}[series expansion for $\bG_u,\bG_v$]\label{lem:series-g}
Under the event $\cE_\op$, it holds for all $K\ge 0$,
\begin{align*}
\norm{\bG_u^{1/2} - \bar\bG_u^{1/2}}_\op &\le \rho^{K+1}, \quad \bar\bG_u :=\qty(\sum_{k=0}^K \binom{\frac12}{k} \bE_u^k)^2,\\
\norm{\bG_v^{-1/2} - \bar\bG_v^{-1/2}}_\op &\le \rho^{K+1}, \quad \bar\bG_v := \qty(\sum_{k=0}^K \binom{\frac12}{k} \bE_v^k)^{-2}, \\
\norm{\bG_v^{-1} - \check\bG_v^{-1}}_\op &\le \rho^{K+1}, \quad \check\bG_v := \qty(\sum_{k=0}^K (-\bE_v)^k)^{-1}.
\end{align*}
\end{lemma}

\begin{proof}
Note that for all $k\ge 0$,
\begin{align}\label{eq:binom-bounded}
\abs{\binom{\frac12}{k}} = \frac{(2k-3)!!}{(2k)!!} \le 1, \quad \abs{\binom{-\frac12}{k}} = \frac{(2k-1)!!}{(2k)!!} \le 1.
\end{align}
Then by \citet[Theorem 4.8]{higham2008functions},
\begin{align*}
&\norm{(\bI_r+\bE_u)^{1/2} - \sum_{k=0}^K \binom{\frac12}{k} \bE_u^k}_\op\\
&\le \binom{\frac12}{K+1} \max_{0\le t\le 1} \norm{\bE_u^{K+1}(\bI_r+t\bE_u)^{-K-1/2}}_\op \le \rho^{K+1},
\end{align*}
and similarly for the two expansions involving $\bE_v$.
\end{proof}

\subsection{Block resolvent integral representation}\label{sec:large-resolvent}

Set $\bO_u = \bU\bG_u^{-1/2}$, $\bO_v = \bV\bG_v^{-1/2}$ so that $\bO_\gamma^\top\bO_\gamma = \bI_r$ for $\gamma\in\{u,v\}$, and let $\bP_\gamma\in\RR^{d\times d}$ be an orthonormal completion of $\bO_\gamma$. Also define
\begin{align*}
\bC := \bmat{\bG_u^{1/2} \bQ \bG_v^{1/2} &\\&0}.
\end{align*}
We omit all non-diagonal zero blocks for brevity. Conditioned on $\bU,\bV$, it holds that $\bP_u\bZ\bP_v^\top \deq \bZ$, and so
\begin{align*}
\bK &= \bU^\top h_\lam(\bG)\bV \\
&= \bU^\top h_\lam(\bU\bQ\bV^\top+\bZ)\bV \\
&= \bG_u^{1/2}\bO_u^\top h_\lam \qty(\bO_u\bG_u^{1/2} \bQ \bG_v^{1/2}\bO_v^\top + \bZ)  \bO_v\bG_v^{1/2} \\
&= \bG_u^{1/2}\bO_u^\top h_\lam \qty(\bP_u \bmat{\bG_u^{1/2} \bQ \bG_v^{1/2} &\\&0} \bP_v^\top + \bZ)  \bO_v\bG_v^{1/2} \\
&\deq \bmat{\bG_u^{1/2}&0} h_\lam\qty(\bC+\bZ) \bmat{\bG_v^{1/2}\\0} \\
&= \bmat{\bG_u^{1/2}&0} \qty(\bC +\bZ) \qty((\bC +\bZ)^\top (\bC +\bZ) +\lam^2 \bI_d)^{-1/2} \bmat{\bG_v^{1/2}\\0}.
\end{align*}
We invoke the following resolvent integral representation
\begin{align*}
    \bX^{-1/2} = \frac{1}{\pi}\int_0^\infty s^{-1/2} (\bX + s\bI_d)^{-1} \rd s.
\end{align*}
Applying to $\bX = (\bC +\bZ)^\top (\bC +\bZ) +\lam^2 \bI_d$, we have
\begin{align}
\bK &= \frac{1}{\pi}\int_0^\infty s^{-1/2} \bmat{\bG_u^{1/2}&0} \qty(\bC +\bZ) (\bX + s\bI_d)^{-1} \bmat{\bG_v^{1/2}\\0} \rd s \notag\\
&= \frac{1}{\pi}\int_0^\infty s^{-1/2} \qty(\bmat{\bG_u\bQ \bG_v^{1/2}&0} + \bmat{\bG_u^{1/2}&0}\bZ) (\bX + s\bI_d)^{-1} \bmat{\bG_v^{1/2}\\0} \rd s \notag\\
&= \frac{1}{\pi}\int_0^\infty s^{-1/2} \qty(\bmat{\bG_u\bQ&0} + \bmat{\bG_u^{1/2}&0}\bZ \bmat{\bG_v^{-1/2}&\\&\bI_{d-r}}) \label{eq:resolvent-head}\\
&\qquad\times \qty(\bmat{\bG_v^{-1/2}&\\&\bI_{d-r}} (\bX+s\bI_d) \bmat{\bG_v^{-1/2}&\\&\bI_{d-r}})^{-1} \bmat{\bI_r\\0} \rd s. \label{eq:resolvent-init}
\end{align}
Let us further define $\beta_s := \sqrt{\lambda^2 + s}$, $\tilde\bZ := \lam^{-1}\bZ$ and denote the zero-padded matrix
\begin{align*}
\db{\bA} := \bmat{\bA&\\&0} \in\RR^{d\times d}, \quad\bA\in\RR^{r\times r}.
\end{align*}
Note that $\db{\bA}^k=\db{\bA^k}$ for $k\ge 1$ but $\db{\bA}^0=\bI_d \ne \db{\bA^0}$.

Expanding Eq.~\eqref{eq:resolvent-head} via Lemma~\ref{lem:series-g}, we have
\begin{align*}
&\bmat{\bG_u\bQ&0} + \bmat{\bG_u^{1/2}&0}\bZ \bmat{\bG_v^{-1/2}&\\&\bI_{d-r}} \\
&= \bmat{\bI_r&0} \qty(\bmat{\bQ+\bE_u\bQ&\\&0} + \bmat{\bG_u^{1/2}&\\&\bI_{d-r}} \bZ \bmat{\bG_v^{-1/2}&\\&\bI_{d-r}}) \\
&= \bmat{\bI_r&0} (\bH + \bR_h),
\end{align*}
where
\begin{align}
\bH &:= \bmat{\bQ+\bE_u\bQ&\\&0} + \bmat{\bar\bG_u^{1/2}&\\&\bI_{d-r}} \bZ \bmat{\bar\bG_v^{-1/2}&\\&\bI_{d-r}} \notag\\
&= \db{\bQ}+\db{\bE_u\bQ} + \lam \sum_{k,\ell=0}^K \binom{\frac12}{k}\binom{-\frac12}{\ell} \db{\bE_u}^k \tilde\bZ \db{\bE_v}^\ell \label{eq:def-bh}
\end{align}
and $\bR_h$ is the error term due to applying the series truncation in Lemma~\ref{lem:series-g}. We control truncation errors in Lemma~\ref{lem:truncation} below. Next, from
\begin{align*}
\bmat{\bG_v^{-1/2}&\\&\bI_{d-r}}\bC^\top = \db{\bQ\bG_u^{1/2}} = \db{\bQ} \bmat{\bG_u^{1/2}&\\&\bI_{d-r}},
\end{align*}
the term in the inverse can be expressed as
\begin{align*}
&\bmat{\bG_v^{-1/2}&\\&\bI_{d-r}} (\bX+s\bI_d) \bmat{\bG_v^{-1/2}&\\&\bI_{d-r}} \\
&=\bmat{\bG_v^{-1/2}&\\&\bI_{d-r}} (\bC^\top\bC + \bC^\top\bZ + \bZ^\top\bC  + \bZ^\top\bZ+\beta_s^2 \bI_d) \bmat{\bG_v^{-1/2}&\\&\bI_{d-r}} \\
&= \bmat{\bQ\bG_u\bQ&\\&0} + \db{\bQ} \bmat{\bG_u^{1/2}&\\&\bI_{d-r}} \bZ \bmat{\bG_v^{-1/2}&\\&\bI_{d-r}} \\
&\qquad + \bmat{\bG_v^{-1/2}&\\&\bI_{d-r}} \bZ^\top \bmat{\bG_u^{1/2}&\\&\bI_{d-r}} \db{\bQ}\\
&\qquad + \bmat{\bG_v^{-1/2}&\\&\bI_{d-r}} \bZ^\top\bZ \bmat{\bG_v^{-1/2}&\\&\bI_{d-r}} + \beta_s^2 \bmat{\bG_v^{-1}&\\&\bI_{d-r}}\\
&= \db{\bQ^2} +\beta_s^2 \bI_d + \db{\bQ\bE_u\bQ} + \beta_s^2 \sum_{k=1}^K \db{-\bE_v}^k \\
&\qquad + \sum_{k,\ell=0}^K \binom{\frac12}{k}\binom{-\frac12}{\ell} \qty(\db{\bQ} \db{\bE_u}^k \bZ \db{\bE_v}^\ell + \db{\bE_v}^\ell \bZ^\top \db{\bE_u}^k \db{\bQ})\\
&\qquad + \sum_{k,\ell=0}^K \binom{-\frac12}{k}\binom{-\frac12}{\ell} \db{\bE_v}^k \bZ^\top\bZ \db{\bE_v}^\ell \\
& \qquad + \bR_{\delta,s},
\end{align*}
where the error $\bR_{\delta,s}$ (here dependent on $s$) is also controlled in Lemma~\ref{lem:truncation}. Hence,
\begin{align*}
\bmat{\bG_v^{-1/2}&\\&\bI_{d-r}} (\bX+s\bI_d) \bmat{\bG_v^{-1/2}&\\&\bI_{d-r}} = \bD_s + \bDelta + \bR_{\delta,s}
\end{align*}
where
\begin{align*}
\bD_s = \diag(d_{1,s},\cdots,d_{d,s}) := \db{\bQ^2} +\beta_s^2 \bI_d
\end{align*}
is diagonal positive-definite and
\begin{align}
\bDelta &:= \db{\bQ\bE_u\bQ} + \beta_s^2 \sum_{k=1}^K \db{-\bE_v}^k \notag\\
&\qquad + \lam\sum_{k,\ell=0}^K \binom{\frac12}{k}\binom{-\frac12}{\ell} \qty(\db{\bQ} \db{\bE_u}^k \tilde\bZ \db{\bE_v}^\ell + \db{\bE_v}^\ell \tilde\bZ^\top \db{\bE_u}^k \db{\bQ}) \notag\\
&\qquad + \lam^2\sum_{k,\ell=0}^K \binom{-\frac12}{k}\binom{-\frac12}{\ell} \db{\bE_v}^k \tilde\bZ^\top\tilde\bZ \db{\bE_v}^\ell. \label{eq:def-bdelta}
\end{align}
Plugging back into Eq.~\eqref{eq:resolvent-init}, we obtain the expression
\begin{align*}
\bK &= \frac{1}{\pi} \int_0^\infty s^{-1/2} \bmat{\bI_r&0} (\bH+\bR_h)(\bD_s+\bDelta+\bR_{\delta,s})^{-1} \bmat{\bI_r\\0} \rd s.
\end{align*}
Compare to the quantity obtained by ignoring the truncation errors $\bR_h,\bR_{\delta,s}$ and expanding the inverse using the (again truncated) Neumann series:
\begin{align}
\tilde\bK = \frac{1}{\pi} \int_0^\infty s^{-1/2} \bmat{\bI_r&0} \underbrace{\bH \bD_s^{-1/2} \sum_{k=0}^K \qty(-\bD_s^{-1/2}\bDelta \bD_s^{-1/2})^k \bD_s^{-1/2}}_{=:\bPsi_s(\bH,\bDelta)} \bmat{\bI_r\\0} \rd s. \label{eq:g-hd-expansion}
\end{align}
We justify this expansion in Lemma~\ref{lem:neumann} and show $\nnorm{\bK-\tilde\bK}_\op =d^{-\omega(1)}$ in Lemma~\ref{lem:truncation}. Hence it suffices to bound the off-diagonal entries of $\tilde\bK$.

\subsection{Truncation error bounds}\label{sec:large-truncation}

Here, we show that the errors from truncating the series for $\bG_u^{1/2},\bG_v^{-1/2},\bG_v^{-1}$ and the Neumann series in Eq.~\eqref{eq:g-hd-expansion} are ignorable.

\begin{lemma}[Neumann series stability]\label{lem:neumann}
Under the event $\cE_\op$, there exists a constant~$C$ such that for every $s\ge 0$,
\begin{align*}
\norm{\bD_s^{-1/2}\bDelta\bD_s^{-1/2}}_\op \le C\rho.
\end{align*}
\end{lemma}

\begin{proof}
Define the diagonal matrices
\begin{align*}
\bA_s := (\bQ^2+\beta_s^2\bI_r)^{-1/2}\bQ, \quad \bB_s := \beta_s (\bQ^2+\beta_s^2\bI_r)^{-1/2}
\end{align*}
so that $\norm{\bA_s}_\op, \norm{\bB_s}_\op\le 1$. We bound each of the four terms in \eqref{eq:def-bdelta} separately. For the first term,
\begin{align*}
\norm{\bD_s^{-1/2}\db{\bQ\bE_u\bQ}\bD_s^{-1/2}}_\op = \norm{\bA_s\bE_u\bA_s}_\op \le \rho.
\end{align*}
For the second term, noting that the sum starts from $k=1$,
\begin{align*}
\norm{\bD_s^{-1/2} \qty(\beta_s^2 \sum_{k=1}^K \db{-\bE_v}^k) \bD_s^{-1/2}}_\op \le \sum_{k=1}^K \norm{\bB_s\bE_v\bB_s}_\op^k \le \sum_{k=1}^K \rho^k \le \frac{\rho}{1-\rho}.
\end{align*}
For the third term, we have
\begin{align*}
&\norm{\bD_s^{-1/2}\db{\bQ}\db{\bE_u}^k}_\op = \norm{\bA_s\bE_u^k}_\op \le\rho^k,\\
&\norm{\db{\bE_v}^\ell \bD_s^{-1/2}}_\op = \beta_s^{-1} \norm{\db{\bE_v}^\ell \bB_s}_\op \le \beta_s^{-1} \rho^\ell.
\end{align*}
Then noting that $\lam\le\beta_s$,
\begin{align*}
&\norm{\bD_s^{-1/2}\qty(\lam\sum_{k,\ell=0}^K \binom{\frac12}{k}\binom{-\frac12}{\ell} \db{\bQ} \db{\bE_u}^k \tilde\bZ \db{\bE_v}^\ell)\bD_s^{-1/2}}_\op \\
&\le \lam \sum_{k,\ell=0}^K \norm{\bD_s^{-1/2}\db{\bQ}\db{\bE_u}^k \tilde\bZ \db{\bE_v}^\ell \bD_s^{-1/2}}_\op \\
&\le \lam\beta_s^{-1} \sum_{k,\ell=0}^K \rho^{k+\ell+1} \le \frac{\rho}{(1-\rho)^2}
\end{align*}
and similarly for the transposed term. Finally for the fourth term,
\begin{align*}
&\norm{\bD_s^{-1/2}\qty(\lam^2\sum_{k,\ell=0}^K \binom{-\frac12}{k}\binom{-\frac12}{\ell} \db{\bE_v}^k \tilde\bZ^\top\tilde\bZ \db{\bE_v}^\ell)\bD_s^{-1/2}}_\op \\
&\le \lam^2\sum_{k,\ell=0}^K \norm{\bD_s^{-1/2} \db{\bE_v}^k \tilde\bZ^\top\tilde\bZ \db{\bE_v}^\ell \bD_s^{-1/2}}_\op \\
&\le \lam^2\beta_s^{-2} \sum_{k,\ell=0}^K \rho^{k+\ell+2} \le \frac{\rho^2}{(1-\rho)^2}.
\end{align*}
Combining the errors concludes the proof.
\end{proof}

\begin{lemma}[truncation error bound]\label{lem:truncation}
Suppose the decay factor satisfies $\rho\lesssim \frac{1}{\log d}$ and the truncation threshold $K\gtrsim\log d$. Under the event $\cE_\op$, it holds that
\begin{align*}
\nnorm{\bK-\tilde\bK}_\op =d^{-\omega(1)}.
\end{align*}
\end{lemma}

\begin{proof}
First we control the errors $\bR_h,\bR_{\delta,s}$. For $\bR_h$, we have that
\begin{align*}
\bR_h &= \bmat{\bG_u^{1/2}&\\&\bI_{d-r}} \bZ \bmat{\bG_v^{-1/2}&\\&\bI_{d-r}} - \bmat{\bar\bG_u^{1/2}&\\&\bI_{d-r}} \bZ \bmat{\bar\bG_v^{-1/2}&\\&\bI_{d-r}} \\
&= \bmat{\bG_u^{1/2}&\\&\bI_{d-r}} \bZ \db{\bG_v^{-1/2} - \bar\bG_v^{-1/2}} + \db{\bG_u^{1/2} - \bar\bG_u^{1/2}} \bZ \bmat{\bar\bG_v^{-1/2}&\\&\bI_{d-r}}.
\end{align*}
Then by Lemma~\ref{lem:series-g},
\begin{align*}
\nnorm{\bR_h}_\op &\le \norm{\bG_u^{1/2}\bZ}_\op \norm{\bG_v^{-1/2} - \bar\bG_v^{-1/2}}_\op + \norm{\bG_u^{1/2} - \bar\bG_u^{1/2}}_\op \norm{\bZ\bar\bG_v^{-1/2}}_\op \\
&\le \qty(2\sqrt{1+\rho} +\rho^{K+1}) \lam\rho \cdot \rho^{K+1} = d^{-\omega(1)}.
\end{align*}
For $\bR_{\delta,s}$, we have that
\begin{align*}
\bR_{\delta,s} &= \beta_s^2 \db{\bG_v^{-1} - \check\bG_v^{-1}}\\
&\qquad + \db{\bQ} \qty(\bmat{\bG_u^{1/2}&\\&\bI_{d-r}} \bZ \bmat{\bG_v^{-1/2}&\\&\bI_{d-r}} - \bmat{\bar\bG_u^{1/2}&\\&\bI_{d-r}} \bZ \bmat{\bar\bG_v^{-1/2}&\\&\bI_{d-r}})\\
&\qquad + \qty(\bmat{\bG_v^{-1/2}&\\&\bI_{d-r}} \bZ^\top \bmat{\bG_u^{1/2}&\\&\bI_{d-r}} - \bmat{\bar\bG_v^{-1/2}&\\&\bI_{d-r}} \bZ^\top \bmat{\bar\bG_u^{1/2}&\\&\bI_{d-r}}) \db{\bQ}\\
&\qquad + \bmat{\bG_v^{-1/2}&\\&\bI_{d-r}} \bZ^\top\bZ \bmat{\bG_v^{-1/2}&\\&\bI_{d-r}} - \bmat{\bar\bG_v^{-1/2}&\\&\bI_{d-r}} \bZ^\top\bZ \bmat{\bar\bG_v^{-1/2}&\\&\bI_{d-r}}
\end{align*}
and we can similarly bound, using that $\nnorm{\bQ}_\op \le 1$ and $\lam = \poly(d^{-1})$,
\begin{align*}
\nnorm{\bR_{\delta,s}}_\op = \beta_s^2\cdot d^{-\omega(1)}.
\end{align*}
It follows that
\begin{align}\label{eq:jaffa-cakes}
\norm{\bD_s^{-1/2}\bR_{\delta,s}\bD_s^{-1/2}}_\op = d^{-\omega(1)}
\end{align}
uniformly over $s$. Now define
\begin{align*}
\bN_1 &= \qty(\bI_d + \bD_s^{-1/2}\bDelta\bD_s^{-1/2} + \bD_s^{-1/2}\bR_{\delta,s}\bD_s^{-1/2})^{-1}, \\
\bN_2 &= \qty(\bI_d + \bD_s^{-1/2}\bDelta\bD_s^{-1/2})^{-1}, \\
\bN_3 &= \sum_{k=0}^K \qty(-\bD_s^{-1/2}\bDelta\bD_s^{-1/2})^k.
\end{align*}
We have that $\nnorm{\bN_1}_\op, \nnorm{\bN_2}_\op = 1+o(1)$ due to Lemma~\ref{lem:neumann} and Eq.~\eqref{eq:jaffa-cakes}. Moreover,
\begin{align*}
\norm{\bN_1 - \bN_2}_\op &= \norm{\bN_2(\bN_2^{-1}-\bN_1^{-1})\bN_1}_\op \lesssim \norm{\bD_s^{-1/2}\bR_{\delta,s}\bD_s^{-1/2}}_\op = d^{-\omega(1)}
\end{align*}
and by the Neumann series,
\begin{align*}
\norm{\bN_2-\bN_3}_\op \le \sum_{k=K+1}^\infty (C\rho)^k = d^{-\omega(1)}.
\end{align*}
Thus from
\begin{align*}
\bK &= \frac{1}{\pi} \int_0^\infty s^{-1/2} \bmat{\bI_r&0} (\bH+\bR_h) \bD_s^{-1/2}\bN_1\bD_s^{-1/2} \bmat{\bI_r\\0} \rd s, \\
\tilde\bK &= \frac{1}{\pi} \int_0^\infty s^{-1/2} \bmat{\bI_r&0} \bH \bD_s^{-1/2} \bN_3 \bD_s^{-1/2} \bmat{\bI_r\\0} \rd s,
\end{align*}
and $\norm{\bH}_\op \lesssim 1$, it follows that
\begin{align*}
&\nnorm{\bK-\tilde\bK}_\op\\
&\le \frac{1}{\pi} \int_0^\infty s^{-1/2} \norm{(\bH+\bR_h) \bD_s^{-1/2}\bN_1\bD_s^{-1/2} - \bH \bD_s^{-1/2} \bN_3 \bD_s^{-1/2}}_\op \rd s \\
&\le \frac{1}{\pi} \int_0^\infty s^{-1/2} \norm{\bR_h \bD_s^{-1/2}\bN_1\bD_s^{-1/2}}_\op \rd s\\
&\qquad + \frac{1}{\pi} \int_0^\infty s^{-1/2} \norm{\bH \bD_s^{-1/2}(\bN_1-\bN_3)\bD_s^{-1/2}}_\op \rd s \\
&\le \frac{1}{\pi} \int_0^\infty s^{-1/2} \beta_s^{-2}\rd s \cdot \qty(\norm{\bR_h}_\op\norm{\bN_1}_\op + \norm{\bH}_\op \nnorm{\bN_1-\bN_3}_\op)\\
&= \lam^{-1} d^{-\omega(1)} = d^{-\omega(1)},
\end{align*}
as was to be shown.
\end{proof}

\subsection{Complete perturbative expansion}\label{sec:large-full}

We will now fully multiply out $\bPsi_s(\bH,\bDelta)$ in Eq.~\eqref{eq:g-hd-expansion} by plugging in Eq.~\eqref{eq:def-bh} and Eq.~\eqref{eq:def-bdelta} into each instance of $\bH,\bDelta$ and further expanding all matrix products entrywise. Since there are many different types of terms, we will keep track of all terms and coefficients by introducing \emph{symbols} $\mu\in\cS_\mu$ and $\nu\in\cS_\nu$ for $\bH$ and $\bDelta$, respectively.

For the rest of the section, we set
\begin{align*}
\tq = (\tq_1,\cdots,\tq_d) := (q_1,\cdots,q_r,0,\cdots,0)
\end{align*}
so that $\db{\bQ} = \diag(\tq)$. We will also denote by $\cI^m$ the set of length $m$ index sequences or \emph{paths} $\iota=(i_1,\cdots,i_m)\in [d]^{m}$, and by $\cI_{ij}^m$ the set of augmented paths $\iota=(i_0,\cdots,i_{m+1})\in [d]^{m+2}$ with the restriction that $i_0=i$ and $i_{m+1}=j$.

For $\bH$, let
\begin{align*}
\cS_\mu := \{1,2\} \cup \left\{(3,k,\ell,\iota): 0\le k,\ell\le K, \; \iota\in\cI^{k+\ell}\right\}.
\end{align*}
From Eq.~\eqref{eq:def-bh}, we can decompose
\begin{align}\label{eq:symbol-bh}
\bH_{ij} = \sum_{\mu\in\cS_\mu} a_{ij}^\mu\bH_{ij}^\mu
\end{align}
where
\begin{enumerate}
    \item $\mu=1$: $(a_{ij}^1, \bH_{ij}^1) = (\tq_j, \db{\bI_r}_{ij})$
    \item $\mu=2$: $(a_{ij}^2, \bH_{ij}^2) = (\tq_j, \db{\bE_u}_{ij})$
    \item $\mu=(3,k,\ell,\iota)$: recalling $\iota=(i_1,\cdots,i_{k+\ell})\in [d]^{k+\ell}$,
    \begin{align*}
    (a_{ij}^\mu, \bH_{ij}^\mu) &= \qty(\binom{\frac12}{k}\binom{-\frac12}{\ell} \lam, \prod_{m=1}^k\db{\bE_u}_{i_{m-1}i_{m}} \tilde\bZ_{i_ki_{k+1}}\prod_{m=1}^\ell \db{\bE_v}_{i_{k+m}i_{k+m+1}}),
    \end{align*}
    here with the convention that $i_0=i$, $i_{m+1}=j$ depending on the pair $(i,j)$ being expanded.
\end{enumerate}

For $\bDelta$, let
\begin{align*}
\cS_\nu := \{1\} &\cup \left\{(2,k,\iota): 1\le k\le K, \;\iota\in \cI^{k-1}\right\} \\
&\cup \left\{(3,k,\ell,\iota): 0\le k,\ell\le K, \;\iota\in\cI^{k+\ell}\right\} \\
&\cup \left\{(4,k,\ell,\iota): 0\le k,\ell\le K, \;\iota\in \cI^{k+\ell}\right\} \\
&\cup \left\{(5,k,\ell,\iota): 0\le k,\ell\le K, \;\iota\in \cI^{k+\ell+1}\right\}.
\end{align*}
From Eq.~\eqref{eq:def-bdelta}, we can decompose $\bDelta_{ij}$ with coefficients in the following bilinear form:
\begin{align}\label{eq:symbol-bdelta}
\bDelta_{ij} = \sum_{\nu\in\cS_\nu} b_i^\nu\bDelta_{ij}^\nu c_j^\nu
\end{align}
where
\begin{enumerate}
    \item $\nu=1$:
    \begin{align*}
    (b_i^1,c_j^1,\bDelta_{ij}^1) = (\tq_i,\tq_j,\db{\bE_u}_{ij})
    \end{align*}
    \item $\nu=(2,k,\iota)$:
    \begin{align*}
    (b_i^\nu,c_j^\nu,\bDelta_{ij}^\nu) = \qty((-1)^k\beta_s,\beta_s,\prod_{m=1}^k \db{\bE_u}_{i_{m-1}i_m})
    \end{align*}
    \item $\nu=(3,k,\ell,\iota)$:
    \begin{align*}
    (b_i^\nu,c_j^\nu,\bDelta_{ij}^\nu) = \qty(\binom{\frac12}{k}\tq_i,\binom{-\frac12}{\ell}\lam, \prod_{m=1}^k\db{\bE_u}_{i_{m-1}i_{m}} \tilde\bZ_{i_ki_{k+1}}\prod_{m=1}^\ell \db{\bE_v}_{i_{k+m}i_{k+m+1}})
    \end{align*}
    \item $\nu=(4,k,\ell,\iota)$:
    \begin{align*}
    (b_i^\nu,c_j^\nu,\bDelta_{ij}^\nu) = \qty(\binom{-\frac12}{\ell}\lam, \binom{\frac12}{k}\tq_j, \prod_{m=1}^\ell \db{\bE_v}_{i_{m-1}i_{m}} \tilde\bZ_{i_{\ell+1}i_\ell}\prod_{m=1}^k \db{\bE_u}_{i_{\ell+m}i_{\ell+m+1}})
    \end{align*}
    \item $\nu=(5,k,\ell,\iota)$:
    \begin{align*}
    &(b_i^\nu,c_j^\nu,\bDelta_{ij}^\nu) = \Bigg(\binom{-\frac12}{k}\lam, \binom{-\frac12}{\ell}\lam, \prod_{m=1}^k\db{\bE_v}_{i_{m-1}i_{m}} \tilde\bZ_{i_{k+1}i_k} \tilde\bZ_{i_{k+1}i_{k+2}}\prod_{m=2}^{\ell+1} \db{\bE_v}_{i_{k+m}i_{k+m+1}}\Bigg).
    \end{align*}
\end{enumerate}
Observe that every $\bH_{ij}^\mu$ and $\bDelta_{ij}^\nu$ are purely products of entries of $\db{\bE_u},\db{\bE_v}$ or $\tilde\bZ$, without any numerical coefficients.

Returning to Eq.~\eqref{eq:g-hd-expansion}, fix a pair of indices $i,j \in[r]$, so that
\begin{align}\label{eq:to-integrate-out}
\tilde\bK_{ij} = \frac{1}{\pi} \int_0^\infty s^{-1/2} \bPsi_s(\bH,\bDelta)_{ij} \rd s.
\end{align}
By expanding each power $\big(-\bD_s^{-1/2}\bDelta \bD_s^{-1/2}\big)^k$ along paths $\iota\in\cI_{ij}^k$ and plugging in Eq.~\eqref{eq:symbol-bh} and Eq.~\eqref{eq:symbol-bdelta}, we obtain
\begin{align}
\bPsi_s(\bH,\bDelta)_{ij}
&= \sum_{k=0}^K \left[ \bH \bD_s^{-1/2} \qty(-\bD_s^{-1/2}\bDelta \bD_s^{-1/2})^k \bD_s^{-1/2} \right]_{ij} \notag \\
&= \sum_{k=0}^K (-1)^k \sum_{\iota\in\cI_{ij}^k} \frac{\bH_{i_0i_1}}{\sqrt{d_{i_1,s}}} \qty(\prod_{\ell=1}^k \frac{\bDelta_{i_\ell i_{\ell+1}}}{\sqrt{d_{i_\ell,s}d_{i_{\ell+1},s}}}) \frac{1}{\sqrt{d_{i_{k+1},s}}} \notag \\
&= \sum_{k=0}^K (-1)^k \sum_{\iota\in\cI_{ij}^k} \sum_{\mu\in\cS_\mu} \frac{a_{i_0i_1}^\mu \bH_{i_0i_1}^\mu}{\sqrt{d_{i_1,s}}} \qty(\prod_{\ell=1}^k \sum_{\nu_\ell\in\cS_\nu} \frac{b_{i_\ell}^{\nu_\ell} \bDelta_{i_\ell i_{\ell+1}}^{\nu_\ell} c_{i_{\ell+1}}^{\nu_\ell}}{\sqrt{d_{i_\ell,s}d_{i_{\ell+1},s}}} ) \frac{1}{\sqrt{d_{i_{k+1},s}}} \notag \\
&= \sum_{k=0}^K (-1)^k \sum_{\iota\in\cI_{ij}^k} \sum_{\mu\in\cS_\mu} \sum_{\nu\in\cS_\nu^k} \zeta_{\iota,s}^{\mu,\nu} \bT_\iota^{\mu,\nu}, \label{eq:zeta-expansion}
\end{align} 
where we have defined for each $\mu\in\cS_\mu$ and $\nu=(\nu_1,\cdots,\nu_k)\in\cS_\nu^k$ ($k$ being implicit),
\begin{align*}
\zeta_{\iota,s}^{\mu,\nu} := \frac{a_{i_0i_1}^\mu}{\sqrt{d_{i_1,s}}} \qty(\prod_{\ell=1}^k \frac{b_{i_\ell}^{\nu_\ell} c_{i_{\ell+1}}^{\nu_\ell}}{\sqrt{d_{i_\ell,s}d_{i_{\ell+1},s}}}) \frac{1}{\sqrt{d_{i_{k+1},s}}}
\end{align*}
and
\begin{align*}
\bT_\iota^{\mu,\nu} := \bH_{i_0i_1}^\mu \prod_{\ell=1}^k \bDelta_{i_\ell i_{\ell+1}}^{\nu_\ell}.
\end{align*}
Note that $\bT_\iota^{\mu,\nu}$ is also a product of a number of entries of $\db{\bE_u},\db{\bE_v}$ or $\tilde\bZ$. We denote this number by the \emph{degree} $n_\iota^{\mu,\nu}$; the degree of $\bT_\iota^{\mu,\nu}$ as a polynomial of Gaussians $u_{ij},v_{ij}$ is $2n_\iota^{\mu,\nu}$ (however, this polynomial is nonhomogeneous due to the presence of the $-\bI_r$ terms in $\bE_u,\bE_v$). From the definition of $\bH,\bDelta$, we can check that
\begin{align}\label{eq:degree-inequality}
0\le n_\iota^{\mu,\nu} \le (2K+1)+K(2K+2) \le CK^2.
\end{align}
The coefficients $\zeta_{\iota,s}^{\mu,\nu}$ further satisfy the following uniform bound.

\begin{lemma}\label{lem:zeta-uniform}
For all $\iota\in\cI^k$ and symbols $\mu\in\cS_\mu$, $\nu\in\cS_\nu^k$, there exists an index $m=m(\iota,\mu,\nu)\in\{1,\cdots,k+1\}$ such that for all $s\ge 0$,
\begin{align*}
|\zeta_{\iota,s}^{\mu,\nu}| \le \frac{\tq_{i_m}\vee\lam}{d_{i_m,s}}.
\end{align*}
\end{lemma}

\begin{proof}
Denote the projection of the symbols $\mu$ and $\nu$ to the integer-valued first coordinate as $\pi(\mu)\in\{1,2,3\}$ and $\pi(\nu)\in\{1,2,3,4,5\}$, respectively.

First suppose $k=0$. When $\pi(\mu)\in\{1,2\}$, we have $a_{ij}^\mu = \tq_j$. When $\pi(\mu)=3$, we have $|a_{ij}^\mu| \le \lam$. Hence
\begin{align*}
|\zeta_{\iota,s}^{\mu,\nu}| = \frac{|a_{ij}^\mu|}{d_{j,s}} \le \frac{\tq_j \vee \lam}{\tq_j^2+\beta_s^2}.
\end{align*}

Now let $k\ge 1$. We first claim that for all $i,j$ and all symbols $\mu\in\cS_\mu$, $\nu,\nu'\in\cS_\nu$, it holds that $|a_{ij}^\mu b_j^\nu| \le d_{j,s}$ and $|b_j^\nu c_j^{\nu'}| \le d_{j,s}$. Indeed, for each $i,j$,
\begin{align*}
a_{ij}^\mu &\in \{\tq_j\} \cup \left\{ \binom{\frac12}{k}\binom{-\frac12}{\ell} \lam : k,\ell\le K\right\}, \\
b_j^\nu, c_j^{\nu'} &\in \left\{\tq_j, \beta_s, -\beta_s\right\} \cup \left\{\binom{\frac12}{k} \tq_j: k \le K\right\} \cup \left\{\binom{-\frac12}{\ell}\lam : \ell\le K\right\}.
\end{align*}
By Eq.~\eqref{eq:binom-bounded} and $\lam\le\sqrt{\lam^2+s} = \beta_s$, we have
\begin{align}\label{eq:ab-bc-claim}
|a_{ij}^\mu b_j^\nu|,\, |b_j^\nu c_j^{\nu'}| \le (\tq_j\vee \beta_s)^2 \le \tq_j^2 + \beta_s^2 = d_{j,s}
\end{align}
as claimed. Now rewrite
\begin{align*}
\zeta_{\iota,s}^{\mu,\nu} &= \frac{a_{i_0i_1}^\mu}{\sqrt{d_{i_1,s}}} \qty(\prod_{\ell=1}^k \frac{b_{i_\ell}^{\nu_\ell} c_{i_{\ell+1}}^{\nu_\ell}}{\sqrt{d_{i_\ell,s}d_{i_{\ell+1},s}}}) \frac{1}{\sqrt{d_{i_{k+1},s}}} \\
&= \frac{a_{i_0i_1}^\mu b_{i_1}^{\nu_1}}{d_{i_1,s}} \qty(\prod_{\ell=2}^{k} \frac{c_{i_\ell}^{\nu_{\ell-1}} b_{i_\ell}^{\nu_\ell}}{d_{i_\ell,s}}) \frac{c_{i_{k+1}}^{\nu_k}}{d_{i_{k+1},s}}
\end{align*}
to consolidate the denominators. We divide into the following cases.
\begin{enumerate}
\item $\pi(\nu_k)\ne 2$: we have
\begin{align*}
c_{i_{k+1}}^{\nu_k} = c_j^{\nu_k} \in \left\{\binom{\frac12}{k} \tq_j: k \le K\right\} \cup \left\{\binom{-\frac12}{\ell}\lam : \ell\le K\right\} 
\end{align*}
so that $|c_j^{\nu_k}| \le \tq_j \vee \lam$. Thus by Eq.~\eqref{eq:ab-bc-claim},
\begin{align*}
|\zeta_{\iota,s}^{\mu,\nu}| = \frac{|a_{i_0i_1}^\mu b_{i_1}^{\nu_1}|}{d_{i_1,s}} \prod_{\ell=2}^{k} \frac{|c_{i_\ell}^{\nu_{\ell-1}} b_{i_\ell}^{\nu_\ell}|}{d_{i_\ell,s}} \cdot \frac{|c_j^{\nu_k}|}{d_{j,s}} \le \frac{|c_j^{\nu_k}|}{d_{j,s}} \le \frac{\tq_j \vee \lam}{d_{j,s}}.
\end{align*}
%\item $\pi(\nu_k) = 2$ and $\pi(\mu)=3$: we have $c_j^{\nu_k} = \beta_s$ and $|a_{i_0i_1}^\mu| \le \lam$. Then swapping the positions of these two terms,
%\begin{align*}
%|\zeta_{\iota,s}^{\mu,\nu}| \le \frac{|c_j^{\nu_k} b_{i_1}^{\nu_1}|}{d_{i_1,s}} \cdot \frac{|a_{i_0i_1}^\mu|}{d_{j,s}} \le \frac{\beta_s(\tq_{i_1}\vee\beta_s)}{d_{i_1,s}} \cdot \frac{\lam}{d_{j,s}} \le \frac{\lam}{d_{j,s}}.
%\end{align*}

\item $\pi(\nu_k) = 2$ and $\pi(\nu_1)\notin \{1,3\}$: we have $c_j^{\nu_k} = \beta_s$ and $|b_{i_1}^{\nu_1}| \le \beta_s$, as well as $|a_{i_0i_1}^\mu| \le \tq_{i_1}\vee\lam$ regardless of $\mu$. Then
\begin{align*}
|\zeta_{\iota,s}^{\mu,\nu}| \le \frac{|a_{i_0i_1}^\mu|}{d_{i_1,s}} \cdot \frac{|b_{i_1}^{\nu_1} c_j^{\nu_k}|}{d_{j,s}} \le \frac{\tq_{i_1}\vee\lam}{d_{i_1,s}} \cdot \frac{\beta_s^2}{d_{j,s}} \le \frac{\tq_{i_1}\vee\lam}{d_{i_1,s}}.
\end{align*}

\item $\pi(\nu_k) = 2$ and $\pi(\nu_1)\in \{1,3\}$: let $m\in\{2,\cdots,k\}$ be the smallest index such that $\pi(\nu_m) \notin \{1,3\}$, so $c_j^{\nu_k} = \beta_s$ and $|b_{i_m}^{\nu_m}| \le \beta_s$. Since $\pi(\nu_{m-1})\in\{1,3\}$, we also have either $|c_{i_m}^{\nu_{m-1}}| \le \tq_{i_m}$ or $|c_{i_m}^{\nu_{m-1}}| \le \lam$. Then
\begin{align*}
|\zeta_{\iota,s}^{\mu,\nu}| \le \frac{|c_{i_m}^{\nu_{m-1}}|}{d_{i_m,s}} \cdot \frac{|b_{i_m}^{\nu_m} c_j^{\nu_k}|}{d_{j,s}} \le \frac{\tq_{i_m}\vee\lam}{d_{i_m,s}} \cdot \frac{\beta_s^2}{d_{j,s}} \le \frac{\tq_{i_m}\vee\lam}{d_{i_m,s}}.
\end{align*}
\end{enumerate}
This concludes the proof of the lemma.
\end{proof}

\subsection{Positive path correlation and graded recombination}\label{sec:large-path}

Substituting Eq.~\eqref{eq:zeta-expansion} and integrating out~$s$ in Eq.~\eqref{eq:to-integrate-out} thus gives
\begin{align}\label{eq:four-sigma}
\tilde\bK_{ij} = \sum_{k=0}^K (-1)^k \sum_{\iota\in\cI_{ij}^k} \sum_{\mu\in\cS_\mu} \sum_{\nu\in\cS_\nu^k} \theta_\iota^{\mu,\nu} \bT_\iota^{\mu,\nu}
\end{align}
where the coefficients are given as
\begin{align*}
\theta_\iota^{\mu,\nu} = \frac{1}{\pi} \int_0^\infty s^{-1/2} \zeta_{\iota,s}^{\mu,\nu} \rd s.
\end{align*}
Importantly, $\theta_\iota^{\mu,\nu}$ are uniformly bounded: by Lemma~\ref{lem:zeta-uniform}, there exists an index $m$ such that
\begin{align*}
|\theta_\iota^{\mu,\nu}| &\le \frac{1}{\pi} \int_0^\infty s^{-1/2} \cdot \frac{\tq_{i_m}\vee\lam}{d_{i_m,s}} \rd s = \frac{\tq_{i_m}\vee\lam}{\sqrt{\tq_{i_m}^2 + \lam^2}} \le 1.
\end{align*}
With Eq.~\eqref{eq:degree-inequality} in mind, we further introduce a gradation in Eq.~\eqref{eq:four-sigma} according to degree; this is necessary to correctly apply Gaussian hypercontractivity later.
\begin{align}\label{eq:stratification}
\tilde\bK_{ij} = \sum_{n=0}^{CK^2} \bK_{ij:n}, \quad \bK_{ij:n} := \sum_{k=0}^K (-1)^k \sum_{\iota\in\cI_{ij}^k} \sum_{\mu\in\cS_\mu} \sum_{\nu\in\cS_\nu^k} \theta_\iota^{\mu,\nu} \bone_{\{n_\iota^{\mu,\nu}=n\}}\bT_\iota^{\mu,\nu}.
\end{align}

We now present a key insight which allows us to remove the coefficients $\theta_\iota^{\mu,\nu}$ when computing moments of $\bK_{ij:n}$.

\begin{lemma}[positive path correlation]\label{lem:positive-corr}
Let $k,k'\ge 0$. It holds for all paths $\iota\in\cI^k$, $\iota'\in\cI^{k'}$ and symbols $\mu,\mu'\in\cS_\mu$, $\nu\in \cS_\nu^k$, $\nu'\in\cS_\nu^{k'}$ that
\begin{align*}
\E\left[\bT_\iota^{\mu,\nu} \bT_{\iota'}^{\mu',\nu'}\right] \ge 0,
\end{align*}
where the expectation is taken over all $u_1,\cdots,u_N$ and $v_1,\cdots,v_N$.
\end{lemma}

\begin{proof}
$\bT_\iota^{\mu,\nu}, \bT_{\iota'}^{\mu',\nu'}$ are products of indices of $\db{\bE_u},\db{\bE_v}$ or $\tilde\bZ$, so we may write
\begin{align}\label{eq:tt}
\bT_\iota^{\mu,\nu} \bT_{\iota'}^{\mu',\nu'} = \prod_{(i,j)} \db{\bE_u}_{ij} \prod_{(i,j)} \db{\bE_v}_{ij} \prod_{(i,j)} \tilde\bZ_{ij}
\end{align}
where the products range over multisets of index pairs. We can remove the double brackets by restricting to $(i,j)\in [r]\times [r]$ for $\bE_u,\bE_v$, otherwise the product will be identically zero. Further expand each entry as
\begin{align*}
(\bE_u)_{ij} &= u_i^\top u_j - \delta_{ij} = \sum_{\ell=1}^d \qty(u_{i\ell}u_{j\ell} - \frac{\delta_{ij}}{d}), \\
(\bE_v)_{ij} &= v_i^\top v_j - \delta_{ij} = \sum_{\ell=1}^d \qty(v_{i\ell}v_{j\ell} - \frac{\delta_{ij}}{d}), \\
\tilde\bZ_{ij} &= \lam^{-1} \sum_{\ell=r+1}^N q_\ell u_{\ell i} v_{\ell j},
\end{align*}
then Eq.~\eqref{eq:tt} decomposes into a sum of terms with positive coefficients of the form
\begin{align*}
\prod_\gamma \qty(u_\gamma^2-\frac1d) \prod_\gamma \qty(v_\gamma^2-\frac1d) \prod_\gamma u_\gamma \prod_\gamma v_\gamma
\end{align*}
where $\gamma\in [N]\times [d]$ denote index pairs. Rescale $\tu_\gamma = \sqrt{d}u_\gamma$ so that $\tu_\gamma$ is i.i.d. $\cN(0,1)$, then by symmetry it suffices to show
\begin{align*}
\bY = \prod_{\gamma\in\cA} (\tu_\gamma^2-1) \prod_{\gamma\in\cB} \tu_\gamma
\end{align*}
has nonnegative expectation for arbitrary multisets $\cA,\cB$. Denote multiset union by $\sqcup$. By Isserlis' theorem,
\begin{align}
\E[\bY] &= \sum_{m\ge 0} (-1)^m \sum_{\substack{\cA' \subseteq \cA\\ |\cA\setminus \cA'|=m}} \E\left[\prod_{\gamma\in\cA'\sqcup\cA'\sqcup\cB} \tu_\gamma\right]\notag\\
&= \sum_{m\ge 0} (-1)^m \sum_{\substack{\cA' \subseteq \cA\\ |\cA\setminus \cA'|=m}} \cP(\cA'\sqcup\cA'\sqcup\cB) \label{eq:isserlis}
\end{align}
where $\cP(\cC)$ counts the number of ways to partition $\cC$ into pairs of equal index pairs. Then by inclusion–exclusion, Eq.~\eqref{eq:isserlis} exactly counts the number of ways to partition $\cA\sqcup\cA\sqcup\cB$ into pairs which do not contain any of the $(\gamma,\gamma)$ pairs arising from each of the $\tu_\gamma^2-1$ factors, as fixing $m$ such pairs in $\cA\sqcup\cA$ yields a subset of `free' index pairs $\cA'\sqcup\cA'$ where $|\cA\setminus \cA'|=m$. Hence $\E[\bY]$ is a count and thus nonnegative.
\end{proof}

To utilize this result, define the `coefficientless' recombined version $\hat{\bK}$ of $\tilde\bK$ and its gradation~$\hat\bK_{:n}$ analogously to Eq.~\eqref{eq:stratification},
\begin{align*}
\hat\bK := \sum_{n=0}^{CK^2} \hat\bK_{:n}, \quad [\hat\bK_{:n}]_{ij} = \hat\bK_{ij:n} := \sum_{k=0}^K \sum_{\iota\in\cI_{ij}^k} \sum_{\mu\in\cS_\mu} \sum_{\nu\in\cS_\nu^k} \bone_{\{n_\iota^{\mu,\nu}=n\}}\bT_\iota^{\mu,\nu}.
\end{align*}
Then we can bound using Lemma~\ref{lem:positive-corr} and $|\theta_\iota^{\mu,\nu}|,|\theta_{\iota'}^{\mu',\nu'}| \le 1$,
\begin{align}
\E\left[\bK_{ij:n}^2\right] &= \sum_{k,\iota,\mu,\nu} \sum_{k',\iota',\mu',\nu'} (-1)^{k+k'}\theta_\iota^{\mu,\nu} \theta_{\iota'}^{\mu',\nu'} \bone_{\{n_\iota^{\mu,\nu}=n\}} \bone_{\{n_{\iota'}^{\mu',\nu'}=n\}} \E\left[\bT_\iota^{\mu,\nu} \bT_{\iota'}^{\mu',\nu'} \right] \notag \\
&\le \sum_{k,\iota,\mu,\nu} \sum_{k',\iota',\mu',\nu'} \bone_{\{n_\iota^{\mu,\nu}=n\}} \bone_{\{n_{\iota'}^{\mu',\nu'}=n\}} \E\left[\bT_\iota^{\mu,\nu} \bT_{\iota'}^{\mu',\nu'} \right] \notag \\
%&= \E\left[\qty(\sum_{k,\iota,\mu,\nu} \bH_{i_0i_1}^\mu \prod_{\ell=1}^k \bDelta_{i_\ell i_{\ell+1}}^{\nu_\ell})^2\right] \notag \\
%&= \E\left[\qty(\sum_{k=0}^K [\hat\bH \smash[t]{\hat{\bDelta}}^{k}]_{ij})^2\right] \notag \\
&= \E\left[\hat\bK_{ij:n}^2\right]. \label{eq:g-to-hatg}
\end{align}
Furthermore, define the `coefficientless' versions of $\bH,\bDelta$ as
\begin{align}
\hat\bH &:= \db{\bI_r} + \db{\bE_u} + \sum_{k,\ell=0}^K \db{\bE_u}^k \tilde\bZ \db{\bE_v}^\ell,\label{eq:def-hat-bh}\\
\hat\bDelta &:= \db{\bE_u} + \sum_{k=1}^K \db{\bE_v}^k + \sum_{k,\ell=0}^K (\db{\bE_u}^k \tilde\bZ \db{\bE_v}^\ell + \db{\bE_v}^\ell \tilde\bZ^\top \db{\bE_u}^k) \notag\\
&\qquad + \sum_{k,\ell=0}^K \db{\bE_v}^k \tilde\bZ^\top\tilde\bZ \db{\bE_v}^\ell. \label{eq:def-hat-bdelta}
\end{align}
These correspond to removing precisely the coefficients $a_{ij}^\mu$ (resp. $b_i^\nu,c_j^\nu$) in the entrywise decompositions Eq.~\eqref{eq:symbol-bh}, Eq.~\eqref{eq:symbol-bdelta}, yielding the relations
\begin{align*}
\hat\bH_{ij} = \sum_{\mu\in\cS_\mu} \bH_{ij}^\mu, \quad \hat\bDelta_{ij} = \sum_{\nu\in\cS_\nu} \bDelta_{ij}^\nu
\end{align*}
and
\begin{align*}
\hat\bK_{ij} = \sum_{k,\iota,\mu,\nu} \bT_\iota^{\mu,\nu} = \sum_{k,\iota,\mu,\nu} \bH_{i_0i_1}^\mu \prod_{\ell=1}^k \bDelta_{i_\ell i_{\ell+1}}^{\nu_\ell} = \left[\sum_{k=0}^K \hat\bH \smash[t]{\hat{\bDelta}}^{k} \right]_{ij}.
\end{align*}
Thus we obtain the recombined expression
\begin{align}\label{eq:def-hat-bg}
\hat\bK = \sum_{k=0}^K \hat\bH \smash[t]{\hat{\bDelta}}^{k}.
\end{align}
Since an $n$-fold product of matrices expands entrywise into a sum of $n$-fold products of entries, $\hat\bK_{:n}$ is precisely the grading of $\hat\bK$ according to (polynomial) degree. In particular, we may express $\hat{\bK}_{:n} = \bF_n(\db{\bE_u},\db{\bE_v},\tilde\bZ)$ for some homogeneous matrix polynomial~$\bF_n$ of degree $n$.

Next, let $\bPi$ be any $r\times r$ permutation matrix and let
\begin{align*}
\bPi_+ := \bmat{\bPi&\\& \bI_{d-r}}.
\end{align*}
By symmetry, $(\bU,\bV) \deq (\bU\bPi, \bV\bPi)$ and independently
\begin{align*}
(u_{r+1}, \cdots, u_N, v_{r+1}, \cdots, v_N) \deq (\bPi_+^\top u_{r+1}, \cdots, \bPi_+^\top u_N, \bPi_+^\top v_{r+1}, \cdots, \bPi_+^\top v_N),
\end{align*}
which implies
\begin{align*}
\qty(\db{\bE_u},\db{\bE_v},\tilde\bZ) &\deq \qty(\db{\bPi^\top\bU^\top\bU\bPi - \bI_r},  \db{\bPi^\top\bV^\top\bV\bPi - \bI_r}, \bPi_+^\top\tilde\bZ\bPi_+) \\
&= \qty(\bPi_+^\top\db{\bE_u}\bPi_+, \bPi_+^\top\db{\bE_v}\bPi_+, \bPi_+^\top\tilde\bZ\bPi_+).
\end{align*}
Then for each $n$ it holds that
\begin{align*}
\hat\bK_{:n} &= \bF_n \qty(\db{\bE_u},\db{\bE_v},\tilde\bZ)\\
&\deq \bF_n \qty(\bPi_+^\top\db{\bE_u}\bPi_+, \bPi_+^\top\db{\bE_v}\bPi_+, \bPi_+^\top\tilde\bZ\bPi_+) \\
&=\bPi_+^\top \bF_n \qty(\db{\bE_u},\db{\bE_v},\tilde\bZ) \bPi_+\\
&= \bPi_+^\top \hat\bK_{:n} \bPi_+
\end{align*}
since $\bF_n$ is a matrix polynomial, therefore $\hat{\bK}_{:n}$ is also distributionally invariant under the permutation $\bPi_+$. In particular, the second moment of $\hat{\bK}_{ij:n}$ is equal for any pair of distinct indices $i,j\le r$, and so
\begin{align}
\E\left[\hat{\bK}_{ij:n}^2\right] &= \frac{1}{r(r-1)} \E\left[\sum_{i,j\le r, i\ne j} \hat{\bK}_{ij:n}^2\right] \notag\\
&\le \frac{1}{r(r-1)} \E\left[\nnorm{\hat\bK_{:n}}_{\F}^2\right] \le \frac{d}{r(r-1)} \E\left[\nnorm{\hat\bK_{:n}}_\op^2\right]. \label{eq:gn-frob-bound}
\end{align}

\subsection{Graded tail bounds and hypercontractivity}\label{sec:large-hyper}

We proceed to bound each $\hat\bK_{:n}$. We remark that we only need to control products up to at most polylogarithmic degree since $n\le CK^2 \lesssim (\log d)^2$, otherwise the expectation would suffer superexponential blowup in $d$. In addition, $\hat\bK_{:0} = \db{\bI_r}$ is diagonal (as is $\bK_{:0}$) and does not affect Eq.~\eqref{eq:gn-frob-bound}, so we only consider $n\ge 1$.

Expanding all products in Eq.~\eqref{eq:def-hat-bg}, the number of summed monomials in the expression $\hat{\bK}_{:n} = \bF_n(\db{\bE_u},\db{\bE_v},\tilde\bZ)$ can be upper bounded as follows. Each monomial is an $n$-fold product of $\db{\bE_u},\db{\bE_v},\tilde\bZ$ which we write as a length~$n$ sequence; there are at most~$3^n$ possible sequences. This is further partitioned into~$k+1$ consecutive subsequences which simultaneously determine the power $k\ge 0$ of~$\hat\bDelta$, and which factor of~$\hat\bH$ or~$\hat\bDelta$ each subsequence originated from. Since all terms in Eq.~\eqref{eq:def-hat-bh} are distinct, and all terms in Eq.~\eqref{eq:def-hat-bdelta} are also distinct, this information uniquely specifies each term in $\hat{\bK}_{:n}$. As a partition can be specified by choosing a subset of points in the sequence as break points, the total number of such partitioned sequences is at most $3^n\times 2^n=6^n$.

The discussion thus far implies that
\begin{align*}
\nnorm{\hat\bK_{:n}}_\op \le 6^n \max\left\{\nnorm{\bE_u}_\op, \nnorm{\bE_v}_\op, \nnorm{\tilde\bZ}_\op\right\}^n
\end{align*}
and so
\begin{align}\label{eq:avocado}
\E\left[\nnorm{\hat\bK_{:n}}_{\op}^2\right] & \le 6^{2n} \qty(\E\left[\nnorm{\bE_u}_\op^{2n}\right] + \E\left[\nnorm{\bE_v}_\op^{2n}\right] + \E\left[\nnorm{\tilde\bZ}_\op^{2n}\right]).
\end{align}
We now bound each moment in turn.

For $\bE_u$ and $\bE_v$, recall from Lemma~\ref{lem:decay-e} that
\begin{align*}
\Pr\qty(\norm{\bE_u}_\op > C\max\left\{ \frac{\sqrt{r}+t}{\sqrt{d}}, \qty(\frac{\sqrt{r}+t}{\sqrt{d}})^2\right\}) \le 2e^{-t^2}.
\end{align*}
% Some algebraic manipulation gives \begin{align*} \Pr(\norm{E_U}_\op > s) \le \begin{cases} 1 & 0\le s< C\sqrt{r/d} \\ 2\exp\qty(-\qty(s\sqrt{d}/C-\sqrt{r})^2) & C\sqrt{r/d}\le s<C \\ 2\exp\qty(-\qty(\sqrt{ds}/C-\sqrt{r})^2) & s\ge C. \end{cases} \end{align*}
Applying the tail integral formula and integrating by parts, we have that
\begin{align*}
\E\left[{\norm{\bE_u}_\op^{2n}} \right] &= \int_0^\infty 2ns^{2n-1}\Pr(\norm{\bE_u}_\op>s)\rd s \\
&\le \qty(C\sqrt{\frac{r}{d}})^{2n} + \int_0^{\sqrt{d}-\sqrt{r}} 2n\qty(C\frac{\sqrt{r}+t}{\sqrt{d}})^{2n-1} 2e^{-t^2} \cdot\frac{C}{\sqrt{d}} \rd t \\
&\qquad + \int_{\sqrt{d}-\sqrt{r}}^\infty 2n\qty(C\qty(\frac{\sqrt{r}+t}{\sqrt{d}})^2)^{2n-1} 2e^{-t^2} \cdot\frac{C}{\sqrt{d}} \rd t.
\end{align*}
The second term is bounded, using the inequality $(a+b)^n \le 2^{n-1}(a^n+b^n)$, as
\begin{align*}
&\frac{4nC}{\sqrt{d}} \int_0^{\sqrt{d}-\sqrt{r}} \qty(C\frac{\sqrt{r}+t}{\sqrt{d}})^{2n-1} e^{-t^2} \rd t \\
&\le \frac{2^{2n}nC}{\sqrt{d}} \qty(C\sqrt{\frac{r}{d}})^{2n-1} \int_0^\infty e^{-t^2}\rd t + \frac{2^{2n}nC}{\sqrt{d}} \int_0^\infty \qty(\frac{Ct}{\sqrt{d}})^{2n-1} e^{-t^2} \rd t \\
&\lesssim \frac{2^{2n}nC}{\sqrt{d}} \qty(C\sqrt{\frac{r}{d}})^{2n-1} + \frac{2^{2n}nC}{\sqrt{d}} \qty(\frac{C}{\sqrt{d}})^{2n-1} \Gamma(n) \\
%&\lesssim \frac{2^{2n}nC}{\sqrt{d}} \qty(C\sqrt{\frac{r}{d}})^{2n-1} + \frac{2^{2n}nC}{\sqrt{d}} \qty(C\sqrt{\frac{n}{d}})^{2n-1} \\
&\lesssim \qty(2C\sqrt{\frac{r}{d}})^{2n-1},
\end{align*}
where we have used that $\Gamma(n)\lesssim n^{n-1/2} \ll \sqrt{r}^{2n-1}$ and $n\lesssim (\log d)^2$.

Similarly, the third term is bounded as
\begin{align*}
&\frac{4nC}{\sqrt{d}}\int_{\sqrt{d}-\sqrt{r}}^\infty \qty(C\qty(\frac{\sqrt{r}+t}{\sqrt{d}})^2)^{2n-1} e^{-t^2}\rd t \\
%&\lesssim \frac{2^{4n}nC}{\sqrt{d}} \qty(C\sqrt{\frac{r}{d}})^{4n-2} + \frac{2^{4n}nC}{\sqrt{d}} \int_0^\infty \qty(\frac{Ct}{\sqrt{d}})^{4n-2} e^{-t^2} dt \\
&\lesssim \frac{2^{4n}nC}{\sqrt{d}} \qty(C\sqrt{\frac{r}{d}})^{4n-2} + \frac{2^{4n}nC}{\sqrt{d}} \qty(\frac{C}{\sqrt{d}})^{4n-2} \Gamma\qty(2n+\frac12) \\
&\lesssim \qty(2C\sqrt{\frac{r}{d}})^{4n-2}.
\end{align*}
We thus have
\begin{align*}
\E\left[{\norm{\bE_u}_\op^{2n}} \right] = \E\left[{\norm{\bE_v}_\op^{2n}} \right] \lesssim \qty(2C\sqrt{\frac{r}{d}})^{2n-1}.
\end{align*}
For $\tilde\bZ$, we have from Lemma~\ref{lem:decay-z} and Eq.~\eqref{eq:event-q} that
\begin{align*}
\Pr\qty(\nnorm{\tilde\bZ}_\op > t\sqrt{\frac{r}{d}}) \le e^{Cd(t_0-t)}, \quad\forall t\ge t_0
\end{align*}
for constants $C,t_0$. Then, substituting $s = Cd(t-t_0)$,
\begin{align*}
\E\left[{\nnorm{\tilde\bZ}_\op^{2n}} \right] &\le \qty(t_0\sqrt{\frac{r}{d}})^{2n} + 2n\qty(\frac{r}{d})^n \int_{t_0}^\infty t^{2n-1} e^{Cd(t_0-t)}\rd t \\
&\le \qty(t_0\sqrt{\frac{r}{d}})^{2n} + \frac{2^{2n-1} nr^n}{Cd^{n+1}} \int_0^\infty \qty(\qty(\frac{s}{Cd})^{2n-1} + t_0^{2n-1}) e^{-s} \rd s \\
&= \qty(t_0\sqrt{\frac{r}{d}})^{2n} + \frac{2^{2n-1} nr^n}{Cd^{n+1}} \qty(\frac{\Gamma(2n)}{(Cd)^{2n-1}} + t_0^{2n-1}) \\
&\lesssim \qty(2t_0\sqrt{\frac{r}{d}})^{2n}.
\end{align*}
Recalling that $\rho\asymp\sqrt{r/d}$, we have shown that Eq.~\eqref{eq:avocado} is bounded as $(C\rho)^{2n-1}$ for some constant $C$. Combining Eq.~\eqref{eq:g-to-hatg} and Eq.~\eqref{eq:gn-frob-bound}, it follows that
\begin{align*}
\E\left[\bK_{ij:n}^2\right] \le \E\left[\hat{\bK}_{ij:n}^2\right] \le \frac{d}{r(r-1)} \E\left[\nnorm{\hat\bK_{:n}}_\op^2\right] \lesssim \frac{d}{r^2} (C\rho)^{2n-1}.
\end{align*}

Now fix an integer $L\asymp\log d$ such that $C\rho L \le \frac12$. Observe that each $\bK_{ij:n}$ is a multilinear polynomial of degree at most~$2n$ in the entries $u_{k\ell},v_{k\ell}$, thus by Gaussian hypercontractivity,
\begin{align*}
\E\left[\bK_{ij:n}^L \right]^{1/L} \le (L-1)^n \E\left[\bK_{ij:n}^2\right]^{1/2} \lesssim \frac{\sqrt{dL}}{r}(C\rho L)^{n-1/2} =: \frac{t}{\sqrt{L}}.
\end{align*}
By Markov's inequality,
\begin{align*}
\Pr(|\bK_{ij:n}| > t) \le t^{-L} \E\left[\bK_{ij:n}^L \right] \lesssim L^{-L/2} = d^{-\omega(1)}.
\end{align*}
Therefore, union bounding over all $1\le i,j\le d$ with $i\ne j$ and $n\lesssim (\log d)^2$, we conclude:
\begin{align*}
|\tilde\bK_{ij}| \le\sum_{n=1}^{CK^2} |\bK_{ij:n}| \lesssim \sum_{n=1}^{CK^2} \frac{L\sqrt{d}}{r}(C\rho L)^{n-1/2} \lesssim \frac{(\log d)^3}{\sqrt{d}}
\end{align*}
and hence
\begin{align}\label{eq:interaction-1}
|\bK_{ij}| \lesssim \frac{(\log d)^3}{\sqrt{d}}
\end{align}
with probability $1-d^{-\omega(1)}$.

\subsection{Lipschitz concentration for tail logits}\label{sec:gl}

We now bound the magnitude of the interactions $\gamma_{ij}$ when either $q_i$ or $q_j\ll\lam$, which is true when $\max\{i,j\}>r$ under $\cE_q$. Here, we only provide the argument for when $q_j$ is small. We first show that $h_\lam$ is $\lam^{-1}$-Lipschitz w.r.t. operator norm.

\begin{comment}
\begin{proposition}[Frobenius Lipschitz bound]\label{prop:frob-lipschitz}
For any generalized matrix function associated to a Lipschitz function $h:\RR_{\ge 0}\to\RR$ and arbitrary $\bA,\bB\in\RR^{d\times d}$, it holds that
\begin{align*}
\nnorm{h(\bA)-h(\bB)}_{\F} \le\nnorm{h}_{\Lip}\nnorm{\bA-\bB}_{\F}.
\end{align*}
\end{proposition}

\begin{proof}
This follows from the more general result due to \citet{kittaneh85lipschitz}.
\end{proof}
\end{comment}

\begin{proposition}[operator Lipschitz bound]\label{prop:operator-lipschitz}
For $\lam>0$, $h_\lam(z)=\frac{z}{\sqrt{z^2+\lam^2}}$ and arbitrary $\bA,\bB\in\RR^{d\times d}$, it holds that $\nnorm{h_\lam(\bA)}_\op\le 1$ and
\begin{align}\label{eq:operator-lipschitz}
\nnorm{h_\lam(\bA)-h_\lam(\bB)}_{\op} \le\lam^{-1}\nnorm{\bA-\bB}_{\op}.
\end{align}
\end{proposition}

We remark that in general, matrix functions do not inherit the Lipschitz constant of the underlying scalar function in operator norm (although this is true in Frobenius norm; see \citet{kittaneh85lipschitz}). For this particular result, we rely on a uniform integral representation of $h_\lam$.

\begin{proof}
The first claim $\nnorm{h_\lam(\bA)}_\op\le 1$ holds since the range of $h_\lam$ is contained in $[-1,1]$. We now prove the main claim. We first show Eq.~\eqref{eq:operator-lipschitz} for symmetric $\bA,\bB$; note that since $h_\lam$ is odd, $h_\lam$ is equal to the usual functional calculus when applied to symmetric matrices. Consider the integral representation
\begin{align*}
h_\lam(t)=\frac{2}{\pi}\int_0^\infty \frac{t}{t^2+\delta_s^2}\rd s, \quad \delta_s := \sqrt{\lam^2+s^2}.
\end{align*}
For a real symmetric matrix $\bA$, define the map
\begin{align*}
h_{\lam,R}(\bA) := \frac{2}{\pi}\int_0^R \bA \qty(\bA^2+\delta_s^2 \bI_d)^{-1} \rd s
\end{align*}
so that $h_{\lam,R}(\bA) \to h_\lam(\bA)$ as $R\to\infty$. Note that
\begin{align*}
\qty(\bA+\mathrm{i}\delta_s \bI_d)^{-1} &= \qty(\bA-\mathrm{i}\delta_s \bI_d) \qty(\bA^2+\delta_s^2 \bI_d)^{-1},\\
\qty(\bA-\mathrm{i}\delta_s \bI_d)^{-1} &= \qty(\bA+\mathrm{i}\delta_s \bI_d) \qty(\bA^2+\delta_s^2 \bI_d)^{-1},
\end{align*}
so we may express
\begin{align*}
h_{\lam,R}(\bA) = \frac{1}{\pi}\int_0^R \qty(\bA+\mathrm{i}\delta_s \bI_d)^{-1} + \qty(\bA-\mathrm{i}\delta_s \bI_d)^{-1} \rd s.
\end{align*}
Denoting the spectrum of $\bA$ by $\sigma(\bA)$, it holds that
\begin{align*}
\norm{\qty(\bA\pm\mathrm{i}\delta_s \bI_d)^{-1}}_\op = \max_{\mu\in\sigma(\bA)} \frac{1}{\abs{\mu \pm\mathrm{i} \delta_s}} = \max_{\mu\in\sigma(\bA)} \frac{1}{\sqrt{\mu^2 + \delta_s^2}} \le \frac{1}{\delta_s}.
\end{align*}
Hence for all real symmetric $\bA,\bB$,
\begin{align*}
&\norm{\qty(\bA\pm\mathrm{i}\delta_s \bI_d)^{-1} - \qty(\bB\pm\mathrm{i}\delta_s \bI_d)^{-1}}_\op \\
&= \norm{\qty(\bA\pm\mathrm{i}\delta_s \bI_d)^{-1} (\bB-\bA) \qty(\bB\pm\mathrm{i}\delta_s \bI_d)^{-1}}_\op\\ &\le \frac{1}{\delta_s^2} \norm{\bA-\bB}_\op
\end{align*}
and so
\begin{align*}
\nnorm{h_{\lam,R}(\bA) - h_{\lam,R}(\bB)}_\op &\le \frac{1}{\pi} \int_0^R \frac{2}{\lam^2+s^2} \nnorm{\bA-\bB}_\op \rd s \le \frac{1}{\lam} \nnorm{\bA-\bB}_\op.
\end{align*}
Eq.~\eqref{eq:operator-lipschitz} follows by taking $R\to\infty$.

Now for general $\bA,\bB$, define the symmetric dilations
\begin{align*}
\tbA:=\bmat{0&\bA\\\bA^\top&0}\in\RR^{2d\times 2d},
\quad
\tbB:=\bmat{0&\bB\\\bB^\top&0}\in\RR^{2d\times 2d}.
\end{align*}
Let the SVD of $\bA$ be $\bA=\bU\bSigma \bV^\top$ and define the $2d\times 2d$ orthogonal matrix
\begin{align*}
\bO:=\frac{1}{\sqrt 2}\bmat{\bU&\bU\\\bV&-\bV}.
\end{align*}
Then by a simple computation, $\tbA$ can be diagonalized as
\begin{align*}
\tbA = \bO\bmat{\bSigma&0\\0&-\bSigma}\bO^\top
\end{align*}
so that
\begin{align*}
h_\lam(\tbA) = \bO\bmat{h_\lam(\bSigma)&0\\0&-h_\lam(\bSigma)} \bO^\top = \bmat{0&h_\lam(\bA)\\ h_\lam(\bA)^\top&0}.
\end{align*}
Since operator norm is preserved under dilation, we conclude:
\begin{align*}
\nnorm{h_\lam(\bA) - h_\lam(\bB)}_\op = \nnorm{h_\lam(\tbA) - h_\lam(\tbB)}_\op \le \frac{1}{\lam} \nnorm{\tbA-\tbB}_\op = \frac{1}{\lam} \nnorm{\bA-\bB}_\op.
\end{align*}
\end{proof}

We now show that (a truncated version of) each logit is a centered Lipschitz function of the pair $(u_j,v_j)$.

\begin{lemma}\label{lem:gl-iota}
Let $u,v\sim \cN(0,\bI_d/d)$ i.i.d. Define the maps
\begin{align*}
&F:(\RR^d)^2\to\RR, \quad F(u,v) := u^\top h_\lam(\bG_{-j} + q_juv^\top) v_i,\\
&\iota:\RR^d\to\RR^d, \quad \iota(u) = \frac{u}{1\vee \frac12\norm{u}_2}.
\end{align*}
Then the map $(u,v)\mapsto F(\iota(u),\iota(v))$ is centered and $(2+16\lam^{-1}q_j)$-Lipschitz.
\end{lemma}

\begin{proof}
Since $(u,v) \stackrel{d}{=} (-u,-v)$ and
\begin{align*}
F(\iota(-u),\iota(-v)) = F(-\iota(u),-\iota(v)) = -F(\iota(u),\iota(v)),
\end{align*}
we have $\E[F(\iota(u),\iota(v))]=0$ by symmetry. Also note that $\iota$ is a projection to an $L^2$-ball and thus $1$-Lipschitz. For $(u,v),(u',v')\in(\RR^d)^2$, let
\begin{align*}
H := h_\lam(\bG_{-j} + q_j\iota(u)\iota(v)^\top), \quad H' := h_\lam(\bG_{-j} + q_j\iota(u')\iota(v')^\top).
\end{align*}
Then by Proposition~\ref{prop:operator-lipschitz} and $\norm{\iota(u)}_2 \le 2$,
\begin{align*}
&\abs{F(\iota(u),\iota(v))-F(\iota(u'),\iota(v'))}\\
&\le \norm{\iota(u)-\iota(u')}_2 \norm{H}_\op \norm{\iota(v_i)}_2 + \norm{\iota(u')}_2 \norm{H-H'}_\op\norm{\iota(v_i)}_2\\
&\le 2\norm{u-u'}_2 + \frac{4q_j}{\lam} \norm{\iota(u)\iota(v)^\top - \iota(u')\iota(v')^\top}_\op\\
&\le \qty(2+\frac{8q_j}{\lam}) \norm{u-u'}_2 + \frac{8q_j}{\lam}\norm{v-v'}_2 \\
&\le \qty(2+\frac{16q_j}{\lam}) \norm{(u,v)-(u',v')}_2.
\end{align*}
This proves the assertion.
\end{proof}

By Lemma~\ref{lem:gl-iota} and concentration of Lipschitz functions of Gaussians \citep[Theorem 2.26]{wainwright2019}, it follows that
\begin{align*}
\Pr\qty(\abs{F(\iota(u),\iota(v))} \ge t) \le 2\exp\qty(-\frac{dt^2}{2(2+16\lam^{-1}q_j)^2})
\end{align*}
where the extra $d$ factor comes from the variance scaling of $u,v$. Moreover we have $\norm{u_k}_2,\norm{v_k}_2\le 2$ for all $k\in[N]$ with probability $1-e^{-\Omega(d)}$, so that $\iota(u_j)=u_j$, $\iota(v_j)=v_j$ and
\begin{align}\label{eq:gl-conclusion}
\abs{\gamma_{ij}} = \abs{F(u_j,v_j)} \lesssim \qty(1+\frac{q_j}{\lam}) \sqrt{\frac{\log d}{d}}.
\end{align}
Under $\cE_q$, we further have $q_j\le \nnorm{q_{>r}}_\infty < \lam$, hence $\abs{\gamma_{ij}} \lesssim \sqrt{\frac{\log d}{d}}$ if $j>r$. A similar argument applies when $i>r$. We remark that while Eq.~\eqref{eq:gl-conclusion} holds for all $i,j\in[N]$, we still need the more involved argument for the leading block since our guarantee for the signal in Eq.~\eqref{eq:signal-guarantee} is upper bounded by $\widetilde{O}(1)$.

\clearpage
\section{Proofs for SGD and Newton}

\subsection{Proof of Theorem~\ref{thm:gd}}

Item~$i$ is recovered by the SGD update $\bW_1^{\SGD} = \eta\bG_0$ iff
\begin{align}\label{eq:cei}
u_i^\top\bG_0 v_i > \max_{j\ne i} u_j^\top\bG_0 v_i.
\end{align}
The lower bound amounts to comparing the signal and noise magnitudes of the top~$d^\frac{1}{2\alpha}$ items. For the upper bound, we will show that items $i\gtrsim d^\frac{1}{2\alpha}$ are unlikely to be recovered due to the large random noise from the top $\Theta(\log d)$ competitors.

First note that
\begin{align}\label{eq:usual}
\max_{i\ne j}\left\{\abs{\langle u_i,u_j \rangle}, |\nnorm{u_i}_2^2-1|, \abs{\langle v_i,v_j \rangle}, |\nnorm{v_i}_2^2-1| \right\} \lesssim \sqrt{\frac{\log d}{d}}
\end{align}
with probability $1-O(d^{-M})$, due to the usual concentration bounds. The difference between the centered and uncentered logits can then be bounded as follows:

\begin{lemma}\label{lem:banana-bread}
It holds with probability $1-O(d^{-M})$ that
\begin{align*}
\max_{i,j\in[N]} |u_j^\top (\bG_0 - \bG) v_i| \lesssim \frac{\sqrt{\log d}}{d}.
\end{align*}
\end{lemma}
This improves upon the uniform control in Eq.~\eqref{eq:gradient-appx} as we can explicitly use the inner product structure of the logits in the SGD case.

\begin{proof}
Let $\bar{u}_{-i} := \frac1N \sum_{j\ne i}u_j$, then $\bar{u}_{-i} \sim\cN(0,\frac{N-1}{N^2d}\bI_d)$ and so $|\langle u_i, \bar{u}_{-i}\rangle| \lesssim \sqrt{\frac{\log d}{Nd}}$ for all $i\in[N]$ with probability $1-O(d^{-M})$. Hence,
\begin{align*}
|u_j^\top (\bG_0 - \bG) v_i| &= \abs{\sum_{k\in[N]} -q_k\langle u_i,\bar{u}\rangle \langle v_k,v_i\rangle}\\
&\le \sum_{k\in[N]} q_k |\langle u_i,\bar{u}_{-i}\rangle| |\langle v_k,v_i\rangle| +\frac1N \sum_{k\in[N]} q_k \nnorm{u_i}_2^2 |\langle v_k,v_i\rangle| \lesssim \sqrt{\frac{\log d}{Nd}} + \frac1N.
\end{align*}
The result follows by noting that $N\gtrsim d$.
\end{proof}

Furthermore, we have from Eq.~\eqref{eq:usual}
\begin{align*}
|u_i^\top\bG v_i - q_i| &= \abs{q_i \qty(\nnorm{u_i}_2^2 \nnorm{v_i}_2^2 -1) + \sum_{k\ne i} q_k \langle u_i,u_k\rangle \langle v_i,v_k\rangle} \\
&\lesssim q_i \sqrt{\frac{\log d}{d}} + \frac{\log d}{d} \sum_{k\ne i} q_k \lesssim \frac{\log d}{d}
\end{align*}
and
\begin{align*}
&|u_j^\top\bG v_i - q_j \langle v_i,v_j\rangle| \\
&= \abs{q_i \langle u_i,u_j\rangle \nnorm{v_i}_2^2 + q_j\qty(\nnorm{u_j}_2^2 -1) \langle v_i,v_j\rangle + \sum_{k\ne i,j} q_k \langle u_j,u_k\rangle \langle v_i,v_k\rangle} \\
&\lesssim q_i \sqrt{\frac{\log d}{d}} + q_j \frac{\log d}{d}+ \frac{\log d}{d} \sum_{k\ne i} q_k\lesssim \frac{\log d}{d}.
\end{align*}
Combining these bounds, we see that Eq.~\eqref{eq:cei} is implied by
\begin{align}\label{eq:terry}
q_i \gtrsim \sqrt{\frac{\log d}{d}} > \max_{j\ne i} q_j\langle v_i,v_j\rangle + O\qty(\frac{\log d}{d})
\end{align}
where the second bound follows from $q_j\le 1$ and Eq.~\eqref{eq:usual}. By the Chernoff bound, for items satisfying $p_i\gtrsim\frac{\log d}{B}$ it holds w.h.p. that $q_i\asymp p_i$:
\begin{align}\label{eq:starbucks}
\Pr(|q_i-p_i| \ge\frac{p_i}{2}) \le 2\exp(-\frac{Bp_i}{3}) \lesssim \frac{1}{d^M},
\end{align}
so the first inequality in Eq.~\eqref{eq:terry} holds if
\begin{align*}
p_i\gtrsim\sqrt{\frac{\log d}{d}} \quad\text{and}\quad p_i\gtrsim \frac{\log d}{B}.
\end{align*}
Therefore items $i\lesssim \min\{d^\frac{1}{2\alpha}(\log d)^{-\frac{1}{2\alpha}}, B^\frac{1}{\alpha} (\log d)^{-\frac{1}{\alpha}}\}$ are always recovered, proving the lower bound.

Conversely, Eq.~\eqref{eq:cei} implies
\begin{align}\label{eq:hazelnut}
q_i > \max_{j\ne i} q_j\langle v_i,v_j\rangle - O\qty(\frac{\log d}{d}).
\end{align}
First suppose that $B\gtrsim\sqrt{d} (\log d)^{\alpha+1}
$ and $p_i\gtrsim \frac{\log d}{B}$. 
For each $j\le L\log d$, it holds that $p_j \ge d^{-o(1)}$ so that the Chernoff bound Eq.~\eqref{eq:starbucks} holds for index $j$ as well. Then $j\le d^{-\frac{1}{2\alpha}} i$ so that $p_j \gtrsim \sqrt{d} p_i$ and so $q_j\gtrsim \sqrt{d} q_i$. We also have that $q_j \ge \frac12 p_j \gtrsim d^{-o(1)}$. Thus Eq.~\eqref{eq:hazelnut} further implies
\begin{align*}
\langle v_i,v_j\rangle \le \frac{1}{q_j}\qty( q_i + O\qty(\frac{\log d}{d})) \le \frac{C}{\sqrt{d}}
\end{align*}
for some constant $C$ (independent of $L$). If $p_i\lesssim \frac{\log d}{B}$, we instead use
\begin{align*}
\Pr\qty(Bq_i\ge m) \le \binom{B}{m} p_i^m \le \qty(\frac{eBp_i}{m})^m \lesssim \frac{1}{d^M}
\end{align*}
for sufficiently large $m\asymp \log d$, so $q_i \lesssim \frac{\log d}{B}$ and $q_j\gtrsim p_j\gtrsim (L\log d)^{-a}$ again implies
\begin{align}\label{eq:b-large}
\langle v_i,v_j\rangle \lesssim \frac{1}{q_j} \cdot \frac{\log d}{\min\{B,d\}} \le \frac{C}{\sqrt{d}}.
\end{align}
Now since $\sqrt{d}\langle v_i,v_j\rangle$ is i.i.d. distributed as $\cN(0,\nnorm{v_i}_2^2)$ conditioned on $v_i$ and $\nnorm{v_i}_2\ge \frac12$, this probability can be bounded as
\begin{align*}
\Pr(\max_{j\le L\log d}\langle v_i,v_j\rangle \le \frac{C}{\sqrt{d}}) \le \Pr(\cN(0,1)\le 2C)^{L\log d} \lesssim \frac{1}{d^M}
\end{align*}
by taking $L$ (and thus $B$ in Eq.~\eqref{eq:b-large}) sufficiently large. By a union bound, we conclude that no items $i\gtrsim d^\frac{1}{2\alpha}\log d$ can be recovered.

Finally, if $B\lesssim\sqrt{d}(\log d)^{\alpha+1}$, repeating the analysis for Lemma~\ref{lem:mini-truncated} shows that we sample at most $O(B^{1/\alpha})$ items such that $i>B^{1/\alpha}$. Aside from these items, $q_i=0$ and so Eq.~\eqref{eq:hazelnut} implies $\langle v_i,v_j\rangle \le d^{-1+o(1)}$, hence the same conclusion as above holds.

We have thus shown that items
\begin{align*}
i > i_{\SGD}^\star \asymp \min\left\{d^\frac{1}{2\alpha}(\log d)^{1+\frac{1}{\alpha}}, B^\frac{1}{\alpha}\right\} 
\end{align*}
are not recovered with high probability. It follows that $\hp_1(i\mid i)\le\frac12$ for these items, and hence the cross-entropy loss is lower bounded as
\begin{align*}
L(\bW_1^{\SGD}) = \mathbb{E}_{i \sim p} [-\log p_{\bW}(i \mid i)] \ge \sum_{i>i_{\SGD}^\star} p_i\log 2 \ge\widetilde{\Omega} \qty(\max\left\{ d^{\frac{1}{2\alpha}-\frac12}, B^{\frac{1}{\alpha}-1}\right\}),
\end{align*}
as was to be shown.

\subsection{Proof of Theorem~\ref{thm:newton}}

The Hessian of the cross-entropy loss $L(\bW;\cB)$ at initialization is computed as follows.

\begin{lemma}[Hessian at initialization]\label{lem:hessian}
Define
\begin{align*}
\bSigma_u:=\frac{1}{N} \sum_{i=1}^N u_i u_i^\top - \bar u \bar u^\top, \qquad \bM_v:=\sum_{i=1}^N q_i v_i v_i^\top.
\end{align*}
Then the Hessian $\cH = \nabla_{\bW}^2 L(\bW_0;\cB)$ of~$L$ at initialization is $\bM_v \otimes \bSigma_u$, that is $\cH[\bDelta]=\bSigma_u \bDelta \bM_v$ for every $\bDelta \in \RR^{d \times d}$.
\end{lemma}

\begin{proof}
We differentiate the gradient Eq.~\eqref{eq:gradient} in the direction $\bDelta$. Using that the Jacobian of the softmax map $\sigma$ is $D\sigma=\diag\sigma -\sigma\sigma^\top$,
\begin{align*}
D \hp_\bW(j \mid i)[\bDelta]
=
\sum_{k\in[N]} \hp_\bW(j \mid i)(\delta_{jk} - \hp_\bW(k \mid i)) u_k^\top \bDelta v_i
\end{align*}
and so
\begin{align*}
\nabla_{\bW}^2 L(\bW;\cB)[\bDelta]
=
\sum_{i\in[N]} q_i \sum_{j,k\in[N]}
u_j \hp_\bW(j \mid i)(\delta_{jk} - \hp_\bW(k \mid i)) u_k^\top \bDelta v_i v_i^\top.
\end{align*}
Since all logits are uniform at initialization, we obtain
\begin{align*}
\nabla_{\bW}^2 L(\bW_0;\cB)[\bDelta]
&=
\sum_{i\in[N]} q_i \sum_{j,k\in[N]}
u_j \qty(\frac{\delta_{jk}}{N} - \frac{1}{N^2}) u_k^\top \bDelta v_i v_i^\top = \bSigma_u \bDelta \bM_v.
\end{align*}
The Kronecker factorization follows from the identity $\operatorname{vec}(\bA \bDelta \bB)
=
(\bB^\top \otimes \bA)\operatorname{vec}(\bDelta)$.
\end{proof}

The following lemma shows that $B\gtrsim d^\alpha$ is needed for $\bM_v$ to be invertible, so that the inverse Hessian is well-behaved.

\begin{lemma}\label{lem:proportional}
The number of distinct items observed in a minibatch~$\cB$ of size~$B$ is $\Theta(B^{1/\alpha})$ w.h.p.
\end{lemma}

\begin{proof}
Let $D_{\cB}:= \sum_{i\ge 1} 1_{\{N_i \ge 1\}}$ denote the number of distinct items in~$\cB$, where~$N_i$ is the number of occurrences of item~$i$. Then
\begin{align*}
\E[D_{\cB}]
=
\sum_{i\ge 1} \Pr(N_i \ge 1)
=
\sum_{i\ge 1} (1-(1-p_i)^B).
\end{align*}
We split the sum at the threshold $i_\star \asymp B^{1/\alpha}$. For $i\lesssim i_\star$ we have $Bp_i \gtrsim 1$, so $1-(1-p_i)^B \asymp 1$.
For $i\gtrsim i_\star$ we have $Bp_i \lesssim 1$, so $1-(1-q_i)^B \asymp B i^{-\alpha}$. Therefore
\begin{align*}
\E[D_{\cB}] \asymp
i_\star + B i_\star^{1-\alpha} \asymp
B^{1/\alpha}.
\end{align*}
Moreover, the indicators $1_{\{N_i\ge 1\}}$ are negatively correlated, so $\Var(D_B)\le\E[D_B]\asymp B^{1/\alpha}$. Indeed for $i\ne j$,
\begin{align*}
&\Cov(1_{\{N_i\ge 1\}},1_{\{N_j\ge 1\}})\\
&= \Pr(N_i,N_j\ge 1) - \Pr(N_i\ge 1)\Pr(N_j\ge 1) \textcolor{white}{1^B}\\
&= 1-(1-p_i)^B-(1-p_j)^B+(1-p_i-p_j)^B - (1-(1-p_i)^B)(1-(1-p_j)^B)\\
&= (1-p_i-p_j)^B-(1-p_i)^B(1-p_j)^B\le 0.
\end{align*}
Hence by Chebyshev's inequality, for any $\eps>0$,
\begin{align*}
\Pr\left(\abs{D_{\cB}-\E[D_{\cB}]}\ge \eps B^{1/\alpha}\right) \lesssim B^{-1/\alpha}\to 0,
\end{align*}
and thus $D_{\cB} = \Theta(B^{1/\alpha})$ w.h.p.
\end{proof}

We will also make use of the Hanson-Wright inequality in the following sections:

\begin{lemma}[Hanson-Wright inequality]\label{lem:hw}
Let $\bA\in\RR^{d\times d}$ be fixed and $u\sim\cN(0,\frac1d\bI_d)$. There exists a universal constant $c$ such that for all~$t>0$,
\begin{align*}
\Pr(\abs{u^\top\bA u- \E[u^\top\bA u]} > t) \lesssim 2\exp(-c\min \left\{ \frac{d^2t^2}{\nnorm{\bA}_{\F}^2}, \frac{dt}{\nnorm{\bA}_\op}\right\}).
\end{align*}
\end{lemma}

\begin{proof}
See \citet[Theorem 6.2.2]{vershynin2018high}.
\end{proof}

We now proceed to the proof of Theorem~\ref{thm:newton}. Let
\begin{align*}
\bG = \sum_{i\in[N]} q_i u_i v_i^\top, \qquad \bM = \bM_v = \sum_{i\in[N]} q_i v_i v_i^\top
\end{align*}
and let $\bG_{-i},\bM_{-i}$ be the leave-one-out variants: $\bG=\bG_{-i}+q_iu_iv_i^\top$ and $\bM=\bM_{-i} + q_iv_iv_i^\top$. By Lemma~\ref{lem:hessian}, the Newton update is (setting~$\eta=1/d$ for ease of analysis)
\begin{align*}
\bW_1^{\Newton} = \frac1d \cH^{-1}[\bG_0] = \frac1d \bSigma_u^{-1} \bG_0 \bM^{-1}.
\end{align*}
For some sufficiently large constant~$C$, taking $B\gtrsim (Cd)^\alpha\log d$, it holds that $q_i\asymp p_i$ for all $i\le Cd$ and so by \citep[Theorem 4.6.1]{vershynin2018high}
\begin{align}\label{eq:bulk-lb}
\bM \succeq (Cd)^{-\alpha} \sum_{j=1}^{Cd} v_jv_j^\top \succeq \Theta(d^{-\alpha}) \cdot\bI_d,
\end{align}
so that $\nnorm{\bM^{-1}}_\op\lesssim d^\alpha$. We also have $\nnorm{\bar u}_2\lesssim 1/\sqrt{N}$ and
\begin{align*}
\norm{\bSigma_u - \frac1d \bI_d}_{\op}\lesssim \frac{1}{\sqrt{Nd}} + \nnorm{\bar u}_2^2 \lesssim \frac{1}{\sqrt{Nd}}
\end{align*}
with probability $1-e^{-\Omega(d)}$ by concentration of sample covariance \citep[Remark~4.7.3]{vershynin2018high}, as well as $\nnorm{\bG_0-\bG}_\op\lesssim 1/\sqrt{N}$ from Eq.~\eqref{eq:gradient-appx}. It follows that
\begin{align*}
&\norm{\bG\bM^{-1} - \frac1d\bSigma_u^{-1}\bG_0\bM^{-1}}_\op\\ &\le \nnorm{\bG-\bG_0}_\op\nnorm{\bM^{-1}}_\op + \norm{\bI_d - \frac1d\bSigma_u^{-1}}_\op \nnorm{\bG_0\bM^{-1}}_\op \\
&\lesssim \frac{d^\alpha}{\sqrt{N}} + \sqrt{\frac{d}{N}} \cdot d^\alpha \lesssim \frac{1}{\sqrt{d}},
\end{align*}
hence it will suffice to consider the update $\bG\bM^{-1}$.

Now instead of the auxiliary map $\phi$ (Eq.~\eqref{eq:signal-strength}), our analysis for the Newton update directly uses the Sherman--Morrison formula to analyze the effect of adding the $i$th term back into both $\bG,\bM$ on the logits. Indeed, for all $i,j\in[N]$, notice that
\begin{align*}
\gamma_{ij} &= u_j^\top \bG\bM^{-1} v_i \\
&= u_j^\top (\bG_{-i} +q_iu_iv_i^\top) \qty(\bM_{-i}^{-1} - \frac{q_i \bM_{-i}^{-1} v_i v_i^\top \bM_{-i}^{-1}}{1+q_iv_i^\top \bM_{-i}^{-1} v_i}) v_i\\
&= u_j^\top (\bG_{-i} +q_iu_iv_i^\top) \qty(1 - \frac{q_i v_i^\top \bM_{-i}^{-1}v_i}{1+q_iv_i^\top \bM_{-i}^{-1} v_i}) \bM_{-i}^{-1}v_i\\
&= \frac{u_j^\top \bG_{-i}\bM_{-i}^{-1}v_i + q_i\langle u_i,u_j\rangle v_i^\top \bM_{-i}^{-1} v_i}{1+q_iv_i^\top \bM_{-i}^{-1} v_i}.
\end{align*}
Then the logit gap for all $j\ne i$ may be expressed as
\begin{align*}
\gamma_{ii} - \gamma_{ij} = \underbrace{\frac{q_i v_i^\top \bM_{-i}^{-1} v_i}{1+q_iv_i^\top \bM_{-i}^{-1} v_i}(\nnorm{u_i}_2^2 - \langle u_i,u_j\rangle)}_{=:(\bA)} + \underbrace{\frac{(u_i-u_j)^\top \bG_{-i}\bM_{-i}^{-1}v_i}{1+q_iv_i^\top \bM_{-i}^{-1} v_i}}_{=:(\bB)}.
\end{align*}
We analyze both terms in turn.

\paragraph{Signal term~($\bA$).}
By the same argument as in Eq.~\eqref{eq:bulk-lb}, we have for the leave-one-out matrix $\bM_{-i} \succeq \Theta(d^{-\alpha}) \cdot\bI_d$ and $\nnorm{\bM_{-i}^{-1}}_\op\lesssim d^\alpha$. On the other hand, by the same argument as in Lemma~\ref{lem:spectrum-new}, it holds that $\lam_{d/2}(\bM_{-i}) \lesssim d^{-\alpha}$. We thus have the tight characterization
\begin{align*}
\E[v_i^\top\bM_{-i}^{-1}v_i] = \frac1d \Tr(\bM_{-i}^{-1}) \asymp d^\alpha.
\end{align*}
It also holds that $\nnorm{\bM_{-i}^{-1}}_{\F} \lesssim  d^\frac12 \nnorm{\bM_{-i}^{-1}}_\op \lesssim d^{\alpha+\frac12}$, thus by the Hanson--Wright inequality
\begin{align*}
\abs{v_i^\top\bM_{-i}^{-1}v_i - \E[v_i^\top\bM_{-i}^{-1}v_i]} \lesssim \frac{\sqrt{\log d}}{d} \nnorm{\bM_{-i}^{-1}}_{\F} + \frac{\log d}{d} \nnorm{\bM_{-i}^{-1}}_\op = o(d^\alpha),
\end{align*}
we have that $v_i^\top\bM_{-i}^{-1}v_i \asymp d^\alpha$ w.h.p. Moreover, $\nnorm{u_i}_2^2 - \langle u_i,u_j\rangle = \Theta(1)$ w.h.p., hence~($\bA$) is lower bounded as
\begin{align*}
\frac{q_i v_i^\top \bM_{-i}^{-1} v_i}{1+q_iv_i^\top \bM_{-i}^{-1} v_i}(\nnorm{u_i}_2^2 - \langle u_i,u_j\rangle) \gtrsim \frac{q_i d^\alpha}{1+q_i d^\alpha}.
\end{align*}

\paragraph{Noise term~($\bB$).} Let the leave-two-out gradient be $\bG_{-i,-j} := \sum_{k\ne i,j} q_ku_kv_k^\top$, with the convention that $\bG_{-i,-i}=\bG_{-i}$ $\forall i$. To control the two terms in the numerator of ($\bB$), we bound
\begin{align*}
\delta_{ij} := u_j^\top\bG_{-i,-j}\bM_{-i}^{-1}v_i, \quad\forall i,j\in[N].
\end{align*}
Conditioned on all variables except~$u_j$, this is distributed as $\delta_{ij}\sim\cN(0,\sigma_{ij}^2)$ where
\begin{align*}
\sigma_{ij}^2 = \frac1d v_i^\top \underbrace{\bM_{-i}^{-1} \bG_{-i,-j}^\top \bG_{-i,-j} \bM_{-i}^{-1}}_{=:\bX_{ij}} v_i.
\end{align*}
We invoke the Gaussian representation from Eq.~\eqref{eq:z-rep} (note the reversed order since we condition on the embedding instead of the unembedding vectors):
\begin{align*}
\bG_{-i,-j} \deq \frac{1}{\sqrt{d}}\bZ \bN_{-i,-j}^{1/2}, \quad \bN_{-i,-j} = \sum_{k\ne i,j} q_k^2 v_k v_k^\top
\end{align*}
where~$\bZ$ has i.i.d. standard Gaussian entries. Then $\nnorm{\bZ}_\op = \Theta(\sqrt{d})$ w.h.p., so that
\begin{align*}
\bG_{-i,-j}^\top\bG_{-i,-j} \deq \frac1d \bN_{-i,-j}^{1/2} \bZ^\top\bZ \bN_{-i,-j}^{1/2} \preceq \Theta(1) \cdot \bN_{-i,-j}.
\end{align*}

Still conditioning on all $v_1,\cdots,v_N$ except~$v_i$, it follows that
\begin{align*}
\E[v_i^\top\bX_{ij} v_i] = \frac{\Tr(\bX_{ij})}{d} &\lesssim \frac1d \Tr(\bM_{-i}^{-1} \bN_{-i,-j}\bM_{-i}^{-1}) = \frac1d \sum_{k\ne i,j} q_k^2 v_k^\top \bM_{-i}^{-2}v_k.
\end{align*}
We now do a leave-two-out argument for~$\bM_{-i}$. For each $k\ne i$, let $\bM_{-i,-k} = \bM_{-i} - q_k v_kv_k^\top$. Again by the Sherman--Morrison formula,
\begin{align*}
\bM_{-i}^{-1} = \bM_{-i,-k}^{-1} - \frac{q_k \bM_{-i,-k}^{-1} v_k v_k^\top \bM_{-i,-k}^{-1}}{1+q_k v_k^\top \bM_{-i,-k}^{-1} v_k}.
\end{align*}
Then we can directly compute
\begin{align*}
&v_k^\top \bM_{-i}^{-2}v_k\\
&= v_k^\top \qty(\bM_{-i,-k}^{-1} - \frac{q_k \bM_{-i,-k}^{-1} v_k v_k^\top \bM_{-i,-k}^{-1}}{1+q_k v_k^\top \bM_{-i,-k}^{-1} v_k})^2 v_k \\
&= v_k^\top \bM_{-i,-k}^{-2} v_k - \frac{2q_k v_k^\top \bM_{-i,-k}^{-1}v_kv_k^\top \bM_{-i,-k}^{-2}v_k}{1+q_kv_k^\top \bM_{-i,-k}^{-1}v_k} + \frac{q_k^2 v_k^\top\bM_{-i,-k}^{-1} v_k v_k^\top \bM_{-i,-k}^{-2}v_k v_k^\top \bM_{-i,-k}^{-1}v_k}{(1+q_k v_k^\top \bM_{-i,-k}^{-1} v_k)^2} \\
&= \frac{v_k^\top \bM_{-i,-k}^{-2}v_k - q_k v_k^\top \bM_{-i,-k}^{-1}v_kv_k^\top \bM_{-i,-k}^{-2}v_k}{1+q_kv_k^\top \bM_{-i,-k}^{-1}v_k} + \frac{q_k^2(v_k^\top \bM_{-i,-k}^{-1}v_k)^2 v_k^\top \bM_{-i,-k}^{-2}v_k}{(1+q_kv_k^\top \bM_{-i,-k}^{-1}v_k)^2} \\
&= \frac{v_k^\top \bM_{-i,-k}^{-2}v_k}{(1+q_kv_k^\top \bM_{-i,-k}^{-1}v_k)^2}.
\end{align*}
Moreover, $\bM_{-i,-k}^{-2}\preceq O(d^{2\alpha}) \cdot\bI_d$ and $v_k^\top \bM_{-i,-k}^{-1}v_k \asymp d^\alpha$ by the same argument as in the leave-one-out case above. Therefore,
\begin{align*}
\E[v_i^\top\bX_{ij} v_i] &= \frac1d \sum_{k\ne i,j} \frac{q_k^2 v_k^\top \bM_{-i,-k}^{-2}v_k}{(1+q_kv_k^\top \bM_{-i,-k}^{-1}v_k)^2} \\
&\lesssim \frac1d \sum_{k\ne i} \frac{q_k^2 d^{2\alpha}}{(1+q_kd^\alpha)^2} \\
&\lesssim 1 + \frac1d \sum_{k>d} q_k^2 d^{2\alpha} \lesssim (\log d)^2.
\end{align*}
Here, we have used that $\nnorm{q_{>d}}_2 \lesssim d^{1/2-\alpha}\log d$ due to Lemma~\ref{lem:mini-truncated} and $B\gtrsim d^\alpha$. It also immediately follows that $\nnorm{\bX_{ij}}_\op \le \nnorm{\bX_{ij}}_{\F} \le \Tr(\bX_{ij}) \lesssim d(\log d)^2$. Hence by the Hanson--Wright inequality, 
\begin{align*}
\sigma_{ij}^2 = \frac1d v_i^\top\bX_{ij} v_i \lesssim \frac{\E[v_i^\top\bX_{ij} v_i]}{d} + \frac{\sqrt{\log d}}{d^2} \nnorm{\bX_{ij}}_{\F} + \frac{\log d}{d^2} \nnorm{\bX_{ij}}_\op \lesssim \frac{(\log d)^3}{d}.
\end{align*}
It follows from concentration of Gaussian maxima (w.r.t.~$j$) and union bounding (w.r.t.~$i$) that
\begin{align*}
\sup_{i,j}|\delta_{ij}| \lesssim \sqrt{\log d} \cdot \sup_{i,j}\sigma_{ij} \lesssim \frac{(\log d)^2}{\sqrt{d}}
\end{align*}
with probability $1-O(d^{-M})$. This directly bounds $u_i \bG_{-i}\bM_{-i}^{-1}v_i = \delta_{ii}$, while
\begin{align*}
u_j^\top \bG_{-i}\bM_{-i}^{-1}v_i
= u_j^\top (\bG_{-i,-j} + q_ju_jv_j^\top) \bM_{-i}^{-1}v_i
= \delta_{ij} + q_j \nnorm{u_j}_2^2\, v_j^\top \bM_{-i}^{-1}v_i.
\end{align*}
For the second term, a final Sherman--Morrison expansion gives
\begin{align*}
v_j^\top \bM_{-i}^{-1}v_i
& = v_j^\top \qty(\bM_{-i,-j}^{-1} - \frac{q_j \bM_{-i,-j}^{-1}v_jv_j^\top \bM_{-i,-j}^{-1}}{1+q_jv_j^\top \bM_{-i,-j}^{-1}v_j})v_i\\
& = \frac{v_j^\top \bM_{-i,-j}^{-1}v_i}{1+q_jv_j^\top \bM_{-i,-j}^{-1}v_j} \\
&\lesssim \frac{d^\alpha}{1+q_jd^\alpha} \sqrt{\frac{\log d}{d}},
\end{align*}
where we have again used $v_j^\top \bM_{-i,-j}^{-1}v_j \asymp d^\alpha$ and $\nnorm{\bM_{-i,-j}^{-1}}_\op \lesssim d^\alpha$. Putting things together, we may bound ($\bB$) as
\begin{align*}
\abs{\frac{(u_i-u_j)^\top \bG_{-i}\bM_{-i}^{-1}v_i}{1+q_iv_i^\top \bM_{-i}^{-1} v_i}} \lesssim \frac{1}{1+q_id^\alpha} \qty(|\delta_{ii}| + |\delta_{ij}| + \frac{q_j d^\alpha}{1+q_jd^\alpha} \sqrt{\frac{\log d}{d}}) \lesssim \frac{1}{1+q_id^\alpha} \frac{(\log d)^2}{\sqrt{d}}.
\end{align*}
We have thus shown the logit gap is lower bounded as
\begin{align*}
\gamma_{ii} - \gamma_{ij} \gtrsim \frac{1}{1+q_id^\alpha} \qty(q_id^\alpha - O\qty(\frac{(\log d)^2}{\sqrt{d}})) \gtrsim \min\{q_id^\alpha,1\} - O\qty(\frac{(\log d)^2}{\sqrt{d}}),
\end{align*}
and so item~$i$ is recovered if $q_i\gtrsim \widetilde{\Theta}(d^{-\alpha-1/2})$. The rest of the proof follows similarly to Section~\ref{sec:putting} and Corollary~\ref{cor:muon-loss}; note that here we must take $\eta=\widetilde{\Theta}(1/\sqrt{d})$ rather than $\widetilde{\Theta}(\sqrt{d})$ since we began by scaling down the step size by $1/d$.

\clearpage
\section{Proofs for Optimality}

\subsection{Proof of Proposition~\ref{prop:hunt-stein}}

We first show that $\Spec(d)$ is equal to the set of (bounded) spectral estimators $h$ which maps an SVD $\bX=\bA\bS\bB^\top$ to $h(\bX)=\bA h(\bS)\bB^\top$ where $h(\bS)$ is diagonal. Clearly, any such estimator is bi-orthogonally invariant: for any $\bU,\bV\in O(d)$,
\begin{align*}
h(\bU\bX\bV^\top) = h(\bU \bA\bS\bB^\top \bV^\top) = \bU\bA h(\bS) \bB^\top \bV^\top = \bU h(\bX) \bV^\top.
\end{align*}
Conversely, let $h\in\Spec(d)$. Since $h(\bX)=\bA h(\bS)\bB^\top$ by equivariance, it suffices to show that the matrix $h(\bS)$ is diagonal. Let $\bD=\diag(\pm 1,\dots,\pm 1)\in O(d)$ be any diagonal sign matrix.
Since $\bD \bS \bD = \bS$, equivariance yields
\[
h(\bS)=h(\bD \bS \bD)=\bD h(\bS)\bD.
\]
Looking at each entry, this implies for $i\neq j$ that $h(\bS)_{ij} = (\bD h(\bS)\bD)_{ij} = \bD_{ii}\bD_{jj}h(\bS)_{ij}$. Choosing $\bD$ with $\bD_{ii}=1$ and $\bD_{jj}=-1$ forces $h(\bS)_{ij}=0$ as desired.

We proceed to prove minimax optimality of $\Spec(d)$. Let $\nu_d$ denote the normalized Haar measure on $O(d)$ and let $(\bU,\bV)\sim \nu_d\otimes\nu_d$. Given any measurable $h$, define its conjugation
\begin{align*}
h^{\bU,\bV}(\bX) := \bU^\top h(\bU\bX\bV^\top)\bV,
\end{align*}
and also define its bi-orthogonal symmetrization
\begin{align*}
\bar{h}(\bX) := \E_{\bU,\bV}[h^{\bU,\bV}(\bX)] = \E_{\bU,\bV}[\bU^\top h(\bU\bX\bV^\top)\bV].
\end{align*}
The map $\bar h$ is well-defined since $h$ is measurable and $\nnorm{h}_{\F}$ is finite. Moreover for any $\bA,\bB\in O(d)$,
\begin{align*}
\bar{h}(\bA\bX\bB^\top) &= \E_{\bU,\bV}[\bU^\top h(\bU \bA\bX\bB^\top \bV^\top)\bV]\\
&= \bA \E_{\bU,\bV}[(\bU\bA)^\top h(\bU \bA\bX\bB^\top \bV^\top)\bV\bB]\bB^\top = \bA \bar{h}(\bX)\bB^\top,
\end{align*}
hence $\bar{h}\in\Spec(d)$.

Now for any $\bU,\bV\in O(d)$, the loss $L(\bW;\cB) = L(\bW;(u_i,v_i)_{i\in[N]},\cB)$ is invariant to the simultaneous change of basis
\begin{align*}
u_i'= \bU^\top u_i, \quad v_i'=\bV^\top v_i, \quad \bW'= \bU^\top\bW\bV
\end{align*}
since the values of all logits $u_j^\top\bW v_i$ remain unchanged. Denoting the corresponding transformed gradient as $\bG_0' = \sum_i q_i(u_i' - \bar u')(v_i')^\top = \bU^\top\bG_0\bV$, this implies
\begin{align*}
\cR(h) &= \E\left[L(h(\bG_0); (u_i,v_i)_{i\in[N]},\cB)\right] \\
&= \E\left[L(\bU^\top h(\bG_0)\bV; (\bU^\top u_i,\bV^\top v_i)_{i\in[N]},\cB)\right] \\
&= \E\left[ L(h^{\bU,\bV}(\bG_0'); (u_i', v_i')_{i\in[N]},\cB)\right] \\
&= \cR(h^{\bU,\bV}).
\end{align*}
For the last inequality, we have used that $(u_i,v_i)_{i\in[N]} \deq (u_i',v_i')_{i\in[N]}$ due to isotropy of the Gaussian distribution. Also note that the map $\bW\mapsto L(\bW)$ is convex due to convexity of log-sum-exp. Taking expectations over $(\bU,\bV)\sim \nu_d\otimes\nu_d$ and applying Jensen's inequality yields
\begin{align*}
\cR(h) = \E_{\bU,\bV}[\cR(h^{\bU,\bV})] &= \E_{\bU,\bV}\left[\E\left[L(h^{\bU,\bV}(\bG_0))\right] \right]\\
&\ge \E\left[L(\E_{\bU,\bV}[h^{\bU,\bV}(\bG_0)])\right] = \cR(\bar h).
\end{align*}
Therefore, the infimum must be attained by a spectral estimator.

\subsection{Proof of Lemma~\ref{lem:sval}}

We present the proof for the full uncentered gradient $\bG = \sum_{i\in[N]} p_iu_iv_i^\top$; the leave-one-out case follows similarly. As in the proof of Lemma~\ref{lem:symmetrize}, we have
\begin{align*}
\bG \deq \frac{1}{\sqrt{d}} \bM^{1/2}\bZ, \quad\text{where}\quad \bZ_{k\ell} \sim \cN(0,1) \;\;\text{i.i.d.}
\end{align*}
where~$\bM$ is the weighted covariance matrix of $u_i$ from Eq.~\eqref{eq:marugame}. We first upper bound the sum of inverse singular values of~$\bG$, i.e., the nuclear norm of~$\bG^{-1}$. By submultiplicativity,
\begin{align}\label{eq:nuclear}
\nnorm{\bG^{-1}}_* = \sqrt{d} \nnorm{\bZ^{-1}\bM^{-1/2}}_* \le \sqrt{d} \nnorm{\bZ^{-1}}_* \nnorm{\bM^{-1/2}}_\op.
\end{align}
We now bound each term. For~$\nnorm{\bM^{-1/2}}_\op$, we have that for a sufficiently large constant~$C$ \citep[Theorem 4.6.1]{vershynin2018high},
\begin{align*}
\bM \succeq \sum_{i=d}^{Cd} p_i^2 u_iu_i^\top \succeq (Cd)^{-2\alpha} \sum_{i=d}^{Cd} u_iu_i^\top \succeq (Cd)^{-2\alpha} \cdot \Theta(1)\bI_d
\end{align*}
and so $\nnorm{\bM^{-1/2}}_\op \lesssim d^\alpha$ with probability $1-e^{-\Omega(d)}$.

To bound $\nnorm{\bZ^{-1}}_*$, we require more precise control on the singular spectrum of~$\bZ$. Denote the~$d\times k$ matrix consisting of the first~$k$ columns of~$\bZ$ as~$\bZ_{1:k}$. By the Courant--Fisher theorem,
\begin{align*}
s_k(\bZ)^2 = \lam_k(\bZ^\top\bZ) = \max_{\dim E=k}\min_{x\in E,\norm{x}=1} x^\top \bZ^\top\bZ x \ge \min_{\norm{x}=1} \nnorm{\bZ_{1:k}x}_2^2 = s_{\min}(\bZ_{1:k})^2
\end{align*}
so that $s_k(\bZ)\ge s_{\min}(\bZ_{1:k})$ for all $k=1,\cdots,d$. Moreover by Theorem 1.1 of \citet{rudelson2009smallestsingularvaluerandom}, we have for all~$t>0$,
\begin{align*}
\Pr(s_{\min}(\bZ_{1:k}) \le t(\sqrt{d}-\sqrt{k-1})) \le (Ct)^{d-k+1} + e^{-\Omega(d)}
\end{align*}
for some constant~$C$. Taking $t=(\log d)^{-1}$ and union bounding,
\begin{align*}
\Pr(s_{\min}(\bZ_{1:k}) \le \frac{\sqrt{d}-\sqrt{k-1}}{\log d}:\forall k) \le \sum_{k=1}^d \qty(\frac{C}{\log d})^{d-k+1} + e^{-\Omega(d)} \lesssim \frac{1}{\log d}.
\end{align*}
This further implies
\begin{align*}
s_k(\bZ)\ge s_{\min}(\bZ_{1:k}) \ge \frac{\sqrt{d}-\sqrt{k-1}}{\log d} \ge \frac{d-k+1}{2\sqrt{d}\log d}
\end{align*}
for all~$k$, hence
\begin{align*}
\nnorm{\bZ^{-1}}_* = \sum_{k=1}^d \frac{1}{s_k(\bZ)} \le 2\sqrt{d}\log d \sum_{k=1}^d \frac{1}{d-k+1} \lesssim \sqrt{d}(\log d)^2.
\end{align*}
We conclude from Eq.~\eqref{eq:nuclear} that
\begin{align*}
\nnorm{\bG^{-1}}_* \lesssim \sqrt{d} \cdot \sqrt{d}(\log d)^2 \cdot d^\alpha = d^{\alpha+1}(\log d)^2.
\end{align*}
Finally, for the minimum singular value, since $\bM,\bZ$ have full rank almost surely, we can lower bound
\begin{align*}
s_{\min}(\bG) &= \frac{1}{\sqrt{d}}\min_{\nnorm{x}=1} \nnorm{\bM^{1/2}\bZ x}\\
&\ge \frac{1}{\sqrt{d}}\min_{\norm{x}=1} \norm{\bM^{1/2}\frac{\bZ x}{\nnorm{\bZ x}}} \min_{\nnorm{x}=1} \nnorm{\bZ x}\\
&= \frac{1}{\sqrt{d}}\lam_{\min}(\bM^{1/2}) s_{\min}(\bZ)\\
&\gtrsim \frac{1}{d^{\alpha+1}\log d},
\end{align*}
as was to be shown.

\subsection{Properties of the Cubic Newton--Schulz Iteration}\label{sec:newton}

\begin{lemma}\label{lem:newton}
Let~$h(z) = \frac32 z -\frac12 z^3$ be the cubic Newton--Schulz map and let~$h^{(k)}$ denote its $k$-fold iterate. Then the map $z\mapsto h^{(k)}(z)/z$ is nonincreasing and $\lim_{k\to\infty} h^{(k)}(z) = 1 = \sgn(z)$ for all $z\in (0,\sqrt{3})$.
\end{lemma}

In other words, the line segment connecting the origin and the point $(z,h^{(k)}(z))$ on the graph of~$h^{(k)}$ becomes flatter as~$z$ increases, which is clear from visual inspection.

\begin{proof}
Since $h$ maps~$[0,\sqrt{3}]$ onto~$[0,1]$, it holds that $0\le h^{(k)}\le 1$ on $[0,\sqrt{3}]$ for all $k\in\NN$. We prove by induction that $h^{(k)}(z)/z$ is nonincreasing on $[0,\sqrt{3}]$, equivalently
\begin{align}\label{eq:blueberry}
\qty(\frac{h^{(k)}(z)}{z})' = \frac{z(h^{(k)})'(z)-h^{(k)}(z)}{z^2} \le 0 \quad\Leftrightarrow\quad z(h^{(k)})'(z) \le h^{(k)}(z).
\end{align}
For $k=1$, the claim is clear. Assume Eq.~\eqref{eq:blueberry} holds for $k\in\NN$, then
\begin{align*}
z(h^{(k+1)})'(z) &= z h'(h^{(k)}(z)) (h^{(k)})'(z) = \frac32 \qty(1-h^{(k)}(z)^2) z(h^{(k)})'(z) \le \frac32 \qty(1-h^{(k)}(z)^2) h^{(k)}(z).
\end{align*}
Thus,
\begin{align*}
h^{(k+1)}(z)
= h(h^{(k)}(z))
= \qty(\frac32-\frac{1}{2}h^{(k)}(z)^2) h^{(k)}(z) \ge z(h^{(k+1)})'(z).
\end{align*}
This proves the first claim. For the pointwise limit, note that the first iterate $h(z)\in[0,1]$ for any $z\in (0,\sqrt3)$ and also $x\le h(x)\le 1$ for all $x\in [0,1]$, so the sequence of iterates $\{h^{(k)}(z)\}_{k\ge 1}$ is monotone increasing. Hence it must converge to a positive fixed point of~$h$, the only solution being~$1$.
\end{proof}

\clearpage
\section{Proofs for Multiple Steps}

\subsection{Auxiliary Results}

We first collect some necessary concentration inequalities.

\begin{lemma}\label{lem:gordon-weighted}
Let $\bA\in\RR^{N\times m}$ be fixed and let $\bZ\in\RR^{d\times N}$ have i.i.d. standard Gaussian entries. Then for every $t\ge 0$,
\begin{align*}
\Pr\qty(\nnorm{\bZ\bA}_{\op} \ge \nnorm{\bA}_{\F} + (\sqrt d + t)\nnorm{\bA}_{\op})
\le 2e^{-ct^2}
\end{align*}
for some universal constant $c$.
\end{lemma}

\begin{proof}
Let $\cT := \SS^{m-1}\times \SS^{d-1}$ and define the Gaussian processes
\begin{align*}
X_{u,v} := \langle \bZ \bA u, v\rangle,
\quad
Y_{u,v} := \langle g,\bA u\rangle + \norm{\bA}_{\op}\langle h,v\rangle,
\end{align*}
where $g\sim \cN(0,\bI_N)$ and $h\sim \cN(0,\bI_d)$ are independent, so that
\begin{align}\label{eq:zaop}
\norm{\bZ\bA}_{\op} = \sup_{(u,v)\in \cT} X_{u,v}.
\end{align}
We compare the increments of $X,Y$. For $(u,v),(w,z)\in \cT$,
\begin{align*}
\E\left[ (X_{u,v}-X_{w,z})^2\right]
&= \E\left[\langle \bZ, v(\bA u)^\top - z(\bA w)^\top\rangle_{\F}^2 \right]\\
&= \nnorm{v(\bA u)^\top - z(\bA w)^\top}_{\F}^2 \\
&= \norm{\bA u}_2^2 + \norm{\bA w}_2^2 - 2\langle \bA u,\bA w\rangle \langle v,z\rangle.
\end{align*}
On the other hand,
\begin{align*}
\E \left[(Y_{u,v}-Y_{w,z})^2\right]
&= \norm{\bA(u-w)}_2^2 + \norm{\bA}_{\op}^2 \norm{v-z}_2^2 \\
&= \norm{\bA u}_2^2 + \norm{\bA w}_2^2 - 2\langle \bA u,\bA w\rangle
    + 2\norm{\bA}_{\op}^2(1-\langle v,z\rangle).
\end{align*}
Therefore,
\begin{align*}
\E \left[(Y_{u,v}-Y_{w,z})^2\right] - \E \left[(X_{u,v}-X_{w,z})^2\right]
&= 2(1-\langle v,z\rangle)\qty(\norm{\bA}_{\op}^2 - \langle \bA u,\bA w\rangle) \ge 0.
\end{align*}
Hence by the Sudakov--Fernique inequality and Eq.~\eqref{eq:zaop}, we obtain
\begin{align*}
\E [\norm{\bZ\bA}_{\op}]
= \E \left[\sup_{(u,v)\in \cT} X_{u,v}\right]
\le \E \left[\sup_{(u,v)\in \cT} Y_{u,v}\right].
\end{align*}
Since~$g,h$ are independent, the right-hand side is further bounded as
\begin{align*}
\E \left[\sup_{(u,v)\in \cT} Y_{u,v}\right]
&= \E \left[\sup_{u\in\SS^{m-1}} \langle g,\bA u\rangle\right]
   + \nnorm{\bA}_{\op}\E \left[\sup_{v\in\SS^{d-1}} \langle h,v\rangle\right] \\
&= \E \nnorm{\bA^\top g}_2 + \nnorm{\bA}_{\op}\E \norm{h}_2 \\
&\le \nnorm{\bA}_{\F} + \sqrt d\nnorm{\bA}_{\op}.
\end{align*}
Now for the tail estimate, the map $\bZ\mapsto \nnorm{\bZ\bA}_{\op}$ is $\norm{\bA}_{\op}$-Lipschitz with respect to the Frobenius norm, thus by Gaussian concentration we have
\begin{align*}
\Pr\qty(
\nnorm{\bZ\bA}_{\op}
\ge \E[\nnorm{\bZ\bA}_{\op}] + t\nnorm{\bA}_{\op}
)
\le 2e^{-ct^2}.
\end{align*}
Combining this with the expectation bound yields the claimed bound.
\end{proof}

\begin{lemma}\label{lem:bilinear-random}
Let $\bA\in\RR^{d\times d}$ be fixed and $u,v\sim\cN(0,\frac1d\bI_d)$ i.i.d. Then with probability $1-O(d^{-M})$,
\begin{align*}
|u^\top \bA v| \lesssim \frac{\sqrt{\log d}}{d}\nnorm{\bA}_{\F} + \frac{\log d}{d} \nnorm{\bA}_\op.
\end{align*}
\end{lemma}

\begin{proof}
By rotational invariance, we may assume $\bA=\diag(\sigma_1,\cdots,\sigma_d)$ with $\sigma_1,\cdots,\sigma_d\ge 0$. Then $u^\top \bA v = \sum_i\sigma_i u_iv_i$ and $du_iv_i$ is subexponential with $\nnorm{du_iv_i}_{\psi_1} = O(1)$. By the subexponential Bernstein inequality,
\begin{align*}
\Pr\qty(|u^\top\bA v| \ge \tau) \le 2\exp\qty(-C\min\left\{\frac{d^2\tau^2}{\sum_i\sigma_i^2}, \frac{d\tau}{\max_i \sigma_i}\right\})
\end{align*}
from which the statement follows.
\end{proof}

\begin{lemma}\label{lem:opf}
Let $\bM = \sum_{i=1}^N q_i u_i u_i^\top$ where $u_i\sim\cN(0,\frac1d \bI_d)$ i.i.d. and $q_i\ge 0$. Then with probability $1-O(d^{-M})$,
\begin{align*}
\nnorm{\bM}_\op \lesssim \frac{\nnorm{q}_1}{d} + \nnorm{q}_\infty, \quad \nnorm{\bM}_{\F} \lesssim \frac{\nnorm{q}_1}{\sqrt{d}} + \nnorm{q}_2.
\end{align*}
\end{lemma}

\begin{proof}
Let $\bQ:=\diag(q_1,\dots,q_N)$ and $\bZ=\sqrt{d}\bmat{u_1\;\cdots\; u_N} \in\RR^{d\times N}$ so that~$\bZ$ has i.i.d. standard Gaussian entries and $\bM = \frac{1}{d}\bZ\bQ\bZ^\top$. By Lemma~\ref{lem:gordon-weighted}, it follows w.h.p. that
\begin{align*}
\nnorm{\bZ\bQ^{1/2}}_{\op} \le \nnorm{\bQ^{1/2}}_{\F} + 2\sqrt{d} \nnorm{\bQ^{1/2}}_\op = \sqrt{\nnorm{q}_1} + 2\sqrt{d\nnorm{q}_\infty},
\end{align*}
and hence $\norm{\bM}_{\op} = \frac1d \nnorm{\bZ\bQ^{1/2}}_{\op}^2 \lesssim \frac1d \norm{q}_1 + \norm{q}_\infty$.

For the second assertion, define $f(\bZ) = \nnorm{\bZ\bQ\bZ^\top}_{\F}^{1/2} = \nnorm{\bZ\bQ^{1/2}}_{S_4}$ where $\norm{\cdot}_{S_4}$ is the Schatten $4$-norm. It holds that $\E[f(\bZ)^4] = d\norm{q}_1^2 + d(d+2) \norm{q}_2^2$ and moreover~$f$ is $\norm{q}_\infty^{1/2}$-Lipschitz:
\begin{align*}
|f(\bZ)-f(\bZ')| \le \nnorm{(\bZ-\bZ')\bQ^{1/2}}_{S_4} \le \nnorm{(\bZ-\bZ')\bQ^{1/2}}_{\F} \le \norm{q}_\infty^{1/2} \nnorm{\bZ-\bZ'}_{\F}.
\end{align*}
It follows from Lipschitz concentration that
\begin{align*}
f(\bZ)^2 \lesssim \E[f(\bZ)^4]^{1/2} + \norm{q}_\infty (\log d)^2 \le \sqrt{d}\norm{q}_1 + d\norm{q}_2,
\end{align*}
and the statement follows.
\end{proof}

 We will make use of the following concentration phenomenon for Gaussian maxima.

\begin{lemma}[superconcentration of Gaussian maxima]\label{lem:superconcentration}
Let $Z_1,\cdots,Z_n$ be i.i.d standard Gaussian. There exists a constant $C>0$ such that
\begin{align*}
\Pr(\max_{i\in[n]} Z_i \le \sqrt{2\log n} - \frac{C\log\log n}{\sqrt{\log n}}) = n^{-\omega(1)}.
\end{align*}
\end{lemma}

\begin{proof}
The Gaussian cumulative distribution function satisfies
\begin{align*}
\Pr(Z \le a) \le 1-\frac{a}{a^2+1} \frac{1}{\sqrt{2\pi}} \exp(-\frac{a^2}{2}), \quad \forall a>0.
\end{align*}
Let 
\begin{align*}
a_n = \sqrt{2\log n} - \frac{C\log\log n}{\sqrt{2\log n}}
\end{align*}
so that $a_n^2 \le 2\log n - 2C\log\log n+1$ for sufficiently large $n$.
It follows that
\begin{align*}
\Pr(\max_{i\in[n]} Z_i \le a_n) &\le \qty(1-\frac{a_n}{a_n^2+1} \frac{1}{\sqrt{2\pi}} \exp(-\frac{a_n^2}{2}))^n \\
&\le \qty(1-\frac{1}{2\sqrt{2\pi} a_n} \exp(-\log n + C\log\log n -\frac12))^n \\
&\le \qty(1-\frac{e^{C\log\log n}}{4n\sqrt{\pi e \log n}})^n \\
&\le \exp(-\frac{(\log n)^{C-0.5}}{4\sqrt{\pi e}}) \\
& = n^{-\omega(1)}
\end{align*}
for~$C>1.5$, where we have used $1-z\le e^{-z}$ for the last inequality.
\end{proof}

\subsection{Proof of Theorem~\ref{thm:multi}}

The logits~$\gamma_{t,ij}$ for $i,j\in[N]$ at step~$t$ are given as (denoting $h_t = h_{\lam_t}$ for brevity)
\begin{align*}
\gamma_{t,ij} := u_j^\top\bW_t v_i = \eta \sum_{s=0}^{t-1} u_j^\top h_s(\bbG_s) v_i, \quad \bbG_s = \sum_{i > d_s} q_i^{(s)} u_i v_i^\top.
\end{align*}
We choose~$\lam_0=\lam$ as in Theorem~\ref{thm:main} and
\begin{align}\label{eq:without-you}
d_t \asymp \min\left\{ \frac{d^{2-(1-\frac{1}{2\alpha})^t}}{(\log d)^{14}}, \qty(\frac{B}{\log d})^\frac{1}{\alpha}\right\}, \quad \lam_t \asymp d_t^{\frac12-\alpha}d^{-\frac12} \log d, \quad\forall t\ge 1.
\end{align}
Note that we have made no attempt to optimize the log factors.

We first analyze the dynamics of the signal logits. Fix an item~$i$ satisfying $i\lesssim B^{\frac{1}{\alpha}}(\log d)^{-\frac{1}{\alpha}}$ and $d_{\tau-1}<i\le d_\tau$ for some $1\le \tau\le T$; when $\tau=1$, the argument in Section~\ref{sec:signal} directly applies, so we assume $\tau\ge 2$. In particular, this implies $B\gtrsim d_{\tau-1}^\alpha\log d$. Let the leave-one-out gradient at step~$t$ be $\bbG_{t,-i} := \bbG_t - q_i^{(t)}u_i v_i^\top$ and define the function
\begin{align*}
\phi_t(q) = u_i^\top h_t(\bbG_{t,-i} + q u_i v_i^\top) v_i, \quad q\ge 0
\end{align*}
so that $\gamma_{t,ii} = \phi_t(q_i^{(t)})$. Repeating the argument in Lemma~\ref{lem:phi-prime}, we may express~$\phi_t'(0)$ in terms of the SVD of $\bbG_{t,-i}$ via the Daleckii--Krein formula, 
\begin{align*}
\phi_t'(0) &\asymp \frac{1}{d^2}\sum_{k \neq \ell}\ \frac{h(s_k(\bbG_{t,-i})) + h(s_\ell(\bbG_{t,-i}))}{s_k(\bbG_{t,-i}) + s_\ell(\bbG_{t,-i})} + \frac{h(s_k(\bbG_{t,-i})) - h(s_\ell(\bbG_{t,-i}))}{s_k(\bbG_{t,-i}) - s_\ell(\bbG_{t,-i})}\\
&\qquad + \frac{1}{d^2}\sum_k h'(s_k(\bbG_{t,-i})).
\end{align*}
Since $h_t(z)/z \le \lam_t^{-1}$ for all $z>0$, we immediately obtain the upper bound $\phi_t'(0)\lesssim \lam_t^{-1}$. For the lower bound, as in Lemma~\ref{lem:symmetrize}, we control the singular values of $\bbG_{t,-i}$ using the eigenvalues of the corresponding weighted covariance matrix:
\begin{align*}
s_k(\bbG_{t,-i}) \lesssim \lam_k(\bM_t)^{1/2}, \quad \bM_t := \sum_{j>d_t} (q_j^{(t)})^2 u_jv_j^\top.
\end{align*}
When $t=0$, the bulk eigenvalue satisfies $\lam_{d/2}(\bM_0) \lesssim d^{-2\alpha}(\log d)^2 \lesssim \lam_0^2$ by Lemma~\ref{lem:spectrum-new}. When $t\ge 1$, we instead use the following uniform bound.
\begin{lemma}\label{lem:hojicha}
For all $1\le t\le T$, it holds with probability~$1-O(d^{-M})$ over sampling of~$q^{(t)}$ that
\begin{align*}
    \nnorm{\bM_t}_\op \lesssim
    %\begin{cases} d_t^{1-2\alpha}\frac{(\log d)^2}{d} & B\gtrsim d_t^\alpha,\\ frac{(\log d)^2}{d} \frac{d_t^{1-\alpha}}{B} & B\lesssim d_t^\alpha.\end{cases} =
\max\left\{d_t^{1-2\alpha}, \frac{d_t^{1-\alpha}}{B}\right\}\frac{(\log d)^2}{d}.
\end{align*}
\end{lemma}

\begin{proof}
%Since the case~$t=0$ is handled by Lemma~\ref{lem:spectrum-new}, we assume~$t\ge 1$. 
First suppose $B\gtrsim d_t^\alpha$. Choose a positive integer $K\asymp \frac1d B^{1/\alpha}$ and define the sets $I_k := \{d_t+(k-1)d+1, \cdots, d_t+kd\}$ for $k\ge 1$. Consider the decomposition
\begin{align}\label{eq:qahwa-time}
\bM_t = \sum_{k\in[K]} \underbrace{\sum_{i\in I_k} (q_i^{(t)})^2u_iu_i^\top}_{=:\bM_{t,k}} + \underbrace{\sum_{i> d_t+Kd} (q_i^{(t)})^2 u_iu_i^\top}_{=:\bM_{t,\tail}}.
\end{align}
As in Lemma~\ref{lem:spectrum-new}, we have that
\begin{align*}
\nnorm{\bM_{t,k}}_\op \lesssim \max_{i\in I_k} (q_i^{(t)})^2 \lesssim \max_{i\in I_k} p_i^2 + \qty(\frac{\log d}{B})^2, \quad \nnorm{\bM_{t,\tail}}_\op \lesssim \qty(\frac{\log d}{B})^2 \frac{B^{1/\alpha}}{d}.
\end{align*}
Therefore from $B\gtrsim d_t^\alpha$,
\begin{align*}
\nnorm{\bM_t}_\op &\lesssim \sum_{k\in[K]} (d_t+kd)^{-2\alpha} + \qty(\frac{\log d}{B})^2 \frac{B^{1/\alpha}}{d}\\
&\lesssim d^{-2\alpha} \qty(\frac{d_t}{d})^{1-2\alpha} + \qty(\frac{\log d}{d_t^\alpha})^2 \frac{d_t}{d} \lesssim d_t^{1-2\alpha}\frac{(\log d)^2}{d}.
\end{align*}
Now suppose~$B\lesssim d_t^\alpha$. The number of items~$N_t$ satisfying $i>d_t$ in a minibatch of size~$B$ is distributed as $\Bin(B,\rho_t)$ where $\rho_t = \sum_{i >d_t} p_i\asymp d_t^{1-\alpha}$, so that $N_t \asymp B\rho_t \asymp Bd_t^{1-\alpha}$. We thus set $K=0$ in Eq.~\eqref{eq:qahwa-time} and bound using $p_i\lesssim d_t^{-\alpha} \lesssim 1/B$,
\begin{align*}
\nnorm{\bM_t}_\op \le \nnorm{\bM_{t,\tail}}_\op \lesssim \qty(\frac{\log d}{B})^2 \frac{N_t}{d} \lesssim \frac{(\log d)^2}{d} \frac{d_t^{1-\alpha}}{B}.
\end{align*}
Combining both cases gives the desired bound.
\end{proof}

For $t\le \tau-1$, since $B\gtrsim d_{\tau-1}^\alpha\log d\ge d_t^\alpha$, it follows that $\nnorm{\bM_t}_\op \lesssim d_t^{1-2\alpha}d^{-1}\log d \lesssim \lam_t^2$ and so
\begin{align*}
\phi_t'(0) \gtrsim \frac{h(s_{d/2}(\bbG_{t,-i}))}{s_{d/2}(\bbG_{t,-i})} \gtrsim \frac{1}{s_{d/2}(\bbG_{t,-i}) + \lam_t} \asymp \frac{1}{\lam_t}.
\end{align*}
Moreover $\phi_t'(q)\ge 0$ and $|\phi_t''(q)|\lesssim \lam_t^{-2}$ hold identically as in Lemma~\ref{lem:phi-prime} and Lemma~\ref{lem:phi-prime-prime}. Expanding $\phi_t$ around zero, we obtain for all $t\le \tau-1$,
\begin{align}\label{eq:macadamia}
u_i^\top h_t(\bbG_t) v_i = \Theta\qty(\frac{q_i^{(t)}}{\lam_t}) + O\qty(\frac{(q_i^{(t)})^2}{\lam_t^2} + \sqrt{\frac{\log d}{d}}).
\end{align}
Since $p_i\gtrsim\frac{\log d}{B}$, a Chernoff bound gives $q_i^{(t)}\asymp p_i\lesssim d_{\tau-1}^{-\alpha}$. Thus for $t\le\tau-1$,
\begin{align}\label{eq:cheesecake}
\frac{q_i^{(t)}}{\lam_t} \lesssim \frac{d_{\tau-1}^{-\alpha}}{\lam_{\tau-1}}  \asymp \frac{1}{\log d}\sqrt{\frac{d}{d_{\tau-1}}} = o(1)
\end{align}
from Eq.~\eqref{eq:without-you}, so the second-order term is always dominated by the first. On the other hand, when $t\ge\tau$ the embeddings~$u_i,v_i$ do not appear in $\bbG_t$ at all, hence only the noise term shows up in Eq.~\eqref{eq:macadamia}. It follows that for $t\ge \tau$,
\begin{align*}
\gamma_{t,ii} &\ge \eta\cdot u_i^\top h_{\tau-1}(\bbG_{\tau-1}) v_i - \sum_{s=0,s\ne \tau-1}^{t-1} \eta\cdot |u_i^\top h_s(\bbG_s) v_i|\\
&\gtrsim \frac{\eta q_i^{(\tau-1)}}{\lam_{\tau-1}} - \sum_{s=0}^{\tau-2} \frac{\eta q_i^{(s)}}{\lam_s} - \eta t \sqrt{\frac{\log d}{d}} \\
&\gtrsim \eta p_i\qty(\frac{1}{\lam_{\tau-1}} - \sum_{s=0}^{\tau-2} \frac{1}{\lam_s}) - o(1)\\
&\gtrsim \frac{\eta d_\tau^{-\alpha}}{\lam_{\tau-1}} - o(1)\\
&\asymp \frac{\sqrt{d}}{(\log d)^4} \qty(\frac{d^{2-(1-\frac{1}{2\alpha})^\tau}}{(\log d)^{14}})^{-\alpha} \qty(\frac{d^{2-(1-\frac{1}{2\alpha})^{\tau-1}}}{(\log d)^{14}})^{\alpha-\frac12} \frac{\sqrt{d}}{\log d} -o(1) \\
& = (\log d)^2 -o(1).
\end{align*}
It remains to bound the interaction logits. For the first gradient step, we have shown in Proposition~\ref{prop:interaction} that $u_j^\top h_0(\bbG_0) v_i$ is $\widetilde{O}(1/\sqrt{d})$ w.h.p. The interaction terms after the first step are much simpler to control; the Lipschitz concentration argument in Section~\ref{sec:gl} will suffice. We only consider the case $j<i$ by symmetry. Fix~$\tau$ such that $d_{\tau-1}<j\le d_\tau$. By the same argument as Lemma~\ref{lem:gl-iota}, the map
\begin{align*}
(u,v) \mapsto \iota(u)^\top h_t\qty(\bbG_{t,-j} + q_j^{(t)} \iota(u)\iota(v)^\top) v_i
\end{align*}
for $t\le\tau-1$ has zero mean and Lipschitz constant $O(1+\lam_t^{-1}q_j^{(t)})=O(1)$ by Eq.~\eqref{eq:cheesecake}, therefore $u_j^\top h_t(\bbG_t) v_i$ concentrates as $\widetilde{O}(1/\sqrt{d})$. Moreover when $t\ge\tau$, $\bbG_t$ is independent of~$u_j$ so the same order concentration holds. Hence for all $t\le T$,
\begin{align*}
|\gamma_{t,ij}| \le \eta \sum_{s=0}^{t-1} |u_j^\top h_s(\bbG_s) v_i| \lesssim \eta \frac{(\log d)^3}{\sqrt{d}} + \eta t\sqrt{\frac{\log d}{d}} = o(1).
\end{align*}
Together, we have for all $t\ge\tau$ that
\begin{align}\label{eq:superpoly-recovery}
\hp_t(i\mid i) = \frac{e^{\gamma_{t,ii}}}{\sum_j e^{\gamma_{t,ij}}} \ge \frac{e^{(\log d)^2}}{e^{(\log d)^2} + Ne^{o(1)}} \ge 1-d^{-\omega(1)},
\end{align}
and so item~$i$ is recovered at all steps~$t\ge\tau$. We remark that by considering $t<\tau$ and choosing~$i>d_{\tau-1}\polylog(d)$, essentially the same argument shows instead that $|\gamma_{t,ii}| = o(1)$, hence the item will have near-uniform logits.

Finally for the loss guarantee, it similarly follows that $\gamma_{t,ii}\ge -o(1)$ for all $t,i$ so that no item will be significantly misclassified at any point during training: $\hp_t(i\mid i) \gtrsim 1/N$. Thus,
\begin{align*}
L(\bW_t) = \mathbb{E}_{i \sim p} [-\log \hp_t(i \mid i)] \lesssim d^{-\omega(1)} + \sum_{i>d_t} p_i \log N = \widetilde{O}(d_t^{1-\alpha}).
\end{align*}

\subsection{Proof of Theorem~\ref{thm:gd-multi}}

With learning rate schedule $\{\eta_t\}_{t\ge 0}$, the logits at step~$t$ are given as
\begin{align*}
\gamma_{t,ij} := u_j^\top \bW_t v_i = \sum_{s=0}^{t-1} \eta_s\cdot u_j^\top \bbG_s v_i = \sum_{s=0}^{t-1} \sum_{k>d_s} \eta_s q_k^{(s)} \langle u_j,u_k\rangle \langle v_i,v_k\rangle.
\end{align*}
We first prove the lower bound. As usual, we assume the high-probability event Eq.~\eqref{eq:usual} when needed. Define the sequence~$d_0=1$ and
\begin{align}\label{eq:d-recursion}
d_{t+1} := \begin{cases}
\min\{d^\frac{1}{2\alpha}d_t (\log d)^{-\frac{5}{\alpha}}, B^\frac{1}{\alpha} (\log d)^{-\frac{1}{\alpha}} \} & d_t<d, \\
\min\{d^\frac{1}{\alpha}d_t^{1-\frac{1}{2\alpha}}(\log d)^{-\frac{5}{\alpha}} , B^\frac{1}{\alpha} (\log d)^{-\frac{1}{\alpha}} \} & d_t>d,
\end{cases}
\end{align}
and set~$\eta_t \asymp d_{t+1}^\alpha (\log d)^2$. It is straightforward to check for both cases of Eq.~\eqref{eq:d-recursion} that
\begin{align}\label{eq:dim-properties}
\qty(\frac{d_{t+1}}{d_t})^\alpha \lesssim \frac{\sqrt{d}}{(\log d)^5},\quad d_{t+1}^\alpha d_t^{\frac12-\alpha} \lesssim \frac{d}{(\log d)^5}.
\end{align}
Fix an item $d_{\tau-1}<i\le d_\tau$ so that $B\gtrsim d_{\tau-1}^\alpha\log d$, and fix~$j\ne i$ with $d_{\tau'-1}<j\le d_{\tau'}$. We control the bulk of the gradient as follows.

\begin{lemma}\label{lem:bilinear-frob-bound}
It holds for all $i,j$ and $t\le\tau-1$ that
\begin{align*}
\norm{\sum_{k>d_t,k\ne i,j} q_k^{(t)}u_k v_k^\top}_{\F} \lesssim d_t^{\frac12-\alpha} \log d.
\end{align*}
\end{lemma}

\begin{proof}
As in Section~\ref{sec:large-setup}, gather all appearing terms into matrices
\begin{align*}
\bU &= \bmat{u_{d_t+1} \;\cdots\; u_N}, \bV = \bmat{v_{d_t+1} \;\cdots\; v_N} \in\RR^{d\times \Theta(N)},\quad \bQ = \diag\qty(q_{d_t+1}^{(t)},\cdots,q_N^{(t)}).
\end{align*}
Since $\sqrt{d}\bU$ has i.i.d. standard Gaussian entries, by Lemma~\ref{lem:gordon-weighted},
\begin{align*}
\nnorm{\bU\bQ\bV^\top}_{\F}\le \sqrt{d} \,\nnorm{\bU\bQ\bV^\top}_\op \lesssim \nnorm{\bQ\bV^\top}_{\F} + \sqrt{d}\,\nnorm{\bQ\bV^\top}_\op \lesssim \sqrt{d}\,\nnorm{\bQ\bV^\top}_\op.
\end{align*}
Applying Lemma~\ref{lem:gordon-weighted} again to~$\sqrt{d}\bV$ gives
\begin{align*}
\nnorm{\bQ\bV^\top}_\op \lesssim \frac{1}{\sqrt{d}} \nnorm{\bQ}_{\F} + \nnorm{\bQ}_\op \le \frac{\nnorm{q_{>d_t}^{(t)}}_2}{\sqrt{d}} + \nnorm{q_{>d_t}^{(t)}}_\infty.
\end{align*}
Since $B\gtrsim d_{\tau-1}^\alpha\log d\ge d_t^\alpha \log d$, we have that $\norm{q_{>d_t}}_2 \lesssim d_t^{\frac12-\alpha}\log d$ and $\nnorm{q_{>d_t}}_\infty \le d_t^{-\alpha}$ w.h.p. by Lemma~\ref{lem:mini-truncated}. Plugging in above gives the desired result.
%\begin{align*}
%\norm{\sum_{k} q_k u_k v_k^\top}_{\F}^2 &= \sum_{k,\ell} q_k q_\ell \langle u_k,u_\ell \rangle \langle v_k,v_\ell \rangle \lesssim \sum_k q_k^2 + \frac{\log d}{d}\sum_{k\ne\ell} q_k q_\ell \le \nnorm{q_{>d_t}}_2^2 + \frac{\log d}{d} \nnorm{q_{>d_t}}_1^2.
%\end{align*}
\end{proof}

Combining Lemma~\ref{lem:bilinear-random} and Lemma~\ref{lem:bilinear-frob-bound}, we have that
\begin{align}\label{eq:bilinear-master}
\abs{\sum_{k>d_t,k\ne i,j} q_k^{(t)} \langle u_j,u_k\rangle \langle v_i,v_k\rangle} \lesssim \frac{\log d}{d}\, \norm{\sum_{k>d_t,k\ne i,j} q_k^{(t)}u_kv_k^\top}_{\F} \lesssim \frac{(\log d)^2}{d} d_t^{\frac12-\alpha}.
\end{align}
Now we bound the interaction logits at step~$\tau$ as
\begin{align*}
&\gamma_{\tau,ij} \\
&= \sum_{t=0}^{\tau-1} \eta_t q_i^{(t)} \langle u_i,u_j\rangle \nnorm{v_i}_2^2 + \sum_{t=0}^{\tau\wedge\tau'-1}\eta_t q_j^{(t)} \nnorm{u_j}_2^2 \langle v_i,v_j\rangle + \sum_{t=0}^{\tau-1} \eta_t \sum_{k>d_t,k\ne i,j} q_k^{(t)} \langle u_j,u_k\rangle \langle v_i,v_k\rangle \\
&\lesssim \sqrt{\frac{\log d}{d}} \sum_{t=0}^{\tau-1} \eta_t q_i^{(t)} + \sqrt{\frac{\log d}{d}}\sum_{t=0}^{\tau\wedge\tau'-1}\eta_t q_j^{(t)} + \frac{(\log d)^2}{d} \sum_{t=0}^{\tau-1} \eta_t d_t^{\frac12-\alpha} \\
&\lesssim \eta_{\tau-1}\sqrt{\frac{\log d}{d}} \qty(p_i + \frac{\log d}{B}) + \eta_{\tau\wedge\tau'-1}\sqrt{\frac{\log d}{d}} \qty(p_j + \frac{\log d}{B}) + \frac{(\log d)^2}{d} \sum_{t=0}^{\tau-1} \eta_t d_t^{\frac12-\alpha} \\
&\lesssim (\log d)^2 \sqrt{\frac{\log d}{d}} \qty(\frac{d_\tau^\alpha}{d_{\tau-1}^\alpha} + \frac{d_{\tau'}^\alpha}{d_{\tau'-1}^\alpha}) + \frac{(\log d)^4}{d} \sum_{t=0}^{\tau-1} d_{t+1}^\alpha d_t^{\frac12-\alpha} \\
&\lesssim \frac{1}{\log d},
\end{align*}
where we have used Eq.~\eqref{eq:bilinear-master}, the usual Chernoff bounds with $B\gtrsim d_{\tau-1}^\alpha\log d$, and Eq.~\eqref{eq:dim-properties} for the last inequality. Next, for the signal logit,
\begin{align*}
\gamma_{\tau,ii} &= \sum_{t=0}^{\tau-1} \eta_t q_i^{(t)} \nnorm{u_i}_2^2\nnorm{v_i}_2^2 + \sum_{t=0}^{\tau-1} \eta_t \sum_{k>d_t,k\ne i} q_k^{(t)} \langle u_i,u_k\rangle \langle v_i,v_k\rangle\\
&\gtrsim \eta_{\tau-1} q_i^{(\tau-1)} - O\qty(\frac{1}{\log d}) \gtrsim (\log d)^2,
\end{align*}
where we have again used Eq.~\eqref{eq:bilinear-master} with $i=j$ and $\eta_{\tau-1} q_i^{(\tau-1)} \gtrsim d_\tau^\alpha(\log d)^2 p_i \gtrsim (\log d)^2$. Therefore item~$i$ is recovered as in Eq.~\eqref{eq:superpoly-recovery}.

We now prove the upper bound. Let $\{\eta_t\}_{t\ge 0}\subset\RR_{\ge 0}$ be any learning rate schedule and suppose $T=o(\sqrt{\log d})$. We will recursively show that the largest item recovered by~$\bW_{t+1}$ must satisfy w.h.p.
\begin{align}\label{eq:sgd-upper-recursion}
d_{\tau+1} \lesssim \begin{cases}
\min\{d^\frac{1}{2\alpha}d_\tau (T\log d)^\frac{1}{\alpha}, (TB)^\frac{1}{\alpha}\} & d_\tau \lesssim d(\log d)^{-4}, \\
\min\{ d^\frac{1}{\alpha} d_\tau^{1-\frac{1}{2\alpha}} (T\log d)^\frac{1}{\alpha}, (TB)^\frac{1}{\alpha}\} & d_\tau \gtrsim d(\log d)^{-4}.
\end{cases}
\end{align}

\paragraph{Case I: $d_\tau\lesssim d(\log d)^{-4}$ and $B\gtrsim d^\frac12 d_\tau^\alpha\log d$.}
The argument for this regime is a slightly more involved version of Theorem~\ref{thm:gd}. Consider a fixed item~$i\asymp d^\frac{1}{2\alpha}d_\tau$ so that $i\lesssim B^{\frac{1}{\alpha}}(\log d)^{-\frac{1}{\alpha}}$ and competitors~$j$ with $d_\tau< j\le 2d_\tau$. Then $q_j^{(t)}/q_i^{(t)}\asymp p_j/p_i \asymp \sqrt{d}$ for all $t\le\tau$ and
\begin{align}
&\gamma_{\tau+1,ij} - \gamma_{\tau+1,ii} \notag\\
&= \sum_{t=0}^\tau \eta_t q_i^{(t)} \langle u_i,u_j\rangle \nnorm{v_i}_2^2 + \sum_{t=0}^\tau \eta_t q_j^{(t)} \nnorm{u_j}_2^2 \langle v_i,v_j\rangle + \sum_{t=0}^\tau \eta_t \sum_{k>d_t,k\ne i,j} q_k^{(t)} \langle u_j,u_k\rangle \langle v_i,v_k\rangle \notag \\
&\qquad - \sum_{t=0}^\tau \eta_t q_i^{(t)} \nnorm{u_i}_2^2\nnorm{v_i}_2^2 - \sum_{t=0}^\tau \eta_t \sum_{k>d_t,k\ne i} q_k^{(t)} \langle u_i,u_k\rangle \langle v_i,v_k\rangle \notag \\
& = \sum_{t=0}^\tau \eta_t \qty(q_j^{(t)} \langle u_j-u_i,u_j\rangle \langle v_i,v_j\rangle + q_i^{(t)} \langle u_i,u_j-u_i\rangle \nnorm{v_i}_2^2 + \sum_{k>d_\tau,k\ne i,j} q_k^{(t)} \langle u_j-u_i,u_k\rangle \langle v_i,v_k\rangle) \notag \\
&\qquad + (u_j-u_i)^\top \underbrace{\sum_{t=0}^\tau \eta_t \sum_{d_t<k\le d_\tau} q_k^{(t)} \langle v_i,v_k\rangle u_k}_{=: w}. \label{eq:agnostic-competition}
\end{align}
Here, we have separated into terms involving items~$k>d_\tau$ (including the signal and competitor items~$i,j$), which can be controlled stepwise, and terms involving items~$k\le d_\tau$ arising from previous gradients, which we control as a group. Let us first examine the terms in the brackets. We have that
\begin{align*}
\abs{q_j^{(t)} \langle -u_i,u_j\rangle \langle v_i,v_j\rangle + q_i^{(t)} \langle u_i,u_j-u_i\rangle \nnorm{v_i}_2^2} \lesssim \frac{\log d}{d} p_j + p_i \lesssim d^{-\frac12}d_\tau^{-\alpha}
\end{align*}
and also by Eq.~\eqref{eq:bilinear-master}
\begin{align*}
\abs{\sum_{k>d_\tau,k\ne i,j} q_k^{(t)} \langle u_j-u_i,u_k\rangle \langle v_i,v_k\rangle} \lesssim \frac{(\log d)^2}{d} d_\tau^{\frac12-\alpha},
\end{align*}
which is dominated by the previous upper bound under $d_\tau\lesssim d(\log d)^{-4}$. Hence we may choose $C=\Theta(1)$ so that $\langle v_i,v_j\rangle \ge C/\sqrt{d}$ implies for all $t\le\tau$,
\begin{align*}
&q_j^{(t)} \langle u_j-u_i,u_j\rangle \langle v_i,v_j\rangle + q_i^{(t)} \langle u_i,u_j-u_i\rangle \nnorm{v_i}_2^2 + \sum_{k>d_\tau,k\ne i,j} q_k^{(t)} \langle u_j-u_i,u_k\rangle \langle v_i,v_k\rangle \\
&\gtrsim \frac{Cq_j^{(t)}}{\sqrt{d}} - \Theta(d^{-\frac12}d_\tau^{-\alpha}) >0.
\end{align*}
Then conditioned on $v_i$ satisfying $\nnorm{v_i}=\Theta(1)$, $\langle v_i,v_j\rangle \ge C/\sqrt{d}$ holds with constant probability independently for each~$d_\tau< j\le 2d_\tau$, so the set $\cJ$ of such items~$j$ has size $\Theta(d_\tau)$ w.h.p.

Now conditioning on variables $v_1,\cdots,v_N$ and $u_1,\cdots,u_{d_\tau}$ (and thus~$w$ and~$\cJ$), the scalars $u_i^\top w$ and $u_j^\top w$ for $j\in\cJ$ are i.i.d. Gaussian, hence the largest among them is \emph{not} $u_i^\top w$ with probability $1-\Theta(d_\tau^{-1})$. It follows from Eq.~\eqref{eq:agnostic-competition} that
\begin{align*}
\max_{j\ne i} \gamma_{\tau+1,ij} - \gamma_{\tau+1,ii} > \max_{j\in\cJ} (u_j-u_i)^\top w > 0
\end{align*}
and thus item~$i$ cannot be recovered with probability $1-\Theta(d_\tau^{-1})$, showing that $d_{\tau+1} \lesssim d^\frac{1}{2\alpha} d_\tau$.

\paragraph{Case II: $d_\tau\gtrsim d(\log d)^{-4}$ and $B\gtrsim dd_\tau^{\alpha-\frac12}\log d$.} Fix an item~$i\asymp d^{\frac1\alpha} d_\tau^{1-\frac1{2\alpha}}$ with $i\lesssim B^{\frac1\alpha}(\log d)^{-\frac1\alpha}$, noting that $i\gg d$, and a competitor $j\in\cJ=\{i+1,\cdots,i+\sqrt{d}\}$. Decompose
\begin{align}
&\gamma_{\tau+1,ij} - \gamma_{\tau+1,ii} \notag\\
%&= \sum_{t=0}^\tau \eta_t q_i^{(t)} \langle u_i,u_j\rangle \nnorm{v_i}_2^2 + \sum_{t=0}^\tau \eta_t q_j^{(t)} \nnorm{u_j}_2^2 \langle v_i,v_j\rangle + \sum_{t=0}^\tau \eta_t \sum_{k>d_t,k\ne i,j} q_k^{(t)} \langle u_j,u_k\rangle \langle v_i,v_k\rangle \\
%&\qquad - \sum_{t=0}^\tau \eta_t q_i^{(t)} \nnorm{u_i}_2^2\nnorm{v_i}_2^2 - \sum_{t=0}^\tau \eta_t \sum_{k>d_t,k\ne i} q_k^{(t)} \langle u_i,u_k\rangle \langle v_i,v_k\rangle \\
& = (u_j-u_i)^\top \underbrace{\sum_{t=0}^\tau \eta_t \qty(\sum_{d_t<k\le d_\tau} q_k^{(t)} \langle v_i,v_k\rangle u_k + \sum_{k>d_\tau, k\notin\cJ\cup\{i\}} q_k^{(t)} \langle v_i,v_k\rangle u_k)}_{=:w} \label{eq:matcha1} \\
&\qquad + \sum_{t=0}^\tau \eta_t \sum_{k\in\cJ} q_k^{(t)} \langle u_j-u_i, u_k\rangle \langle v_i,v_k\rangle + \sum_{t=0}^\tau \eta_t q_i^{(t)} \langle u_j-u_i,u_i\rangle \nnorm{v_i}_2^2. \label{eq:matcha2}
\end{align}
The competing fluctuations will come from Eq.~\eqref{eq:matcha1}. Rewrite
\begin{align*}
w = \sum_{k\le d_\tau} \sum_{t: d_t<k} \eta_t q_k^{(t)} \langle v_i,v_k\rangle u_k + \sum_{k>d_\tau, k\notin\cJ\cup\{i\}} \sum_{t=0}^\tau \eta_t q_k^{(t)} \langle v_i,v_k\rangle u_k.
\end{align*}
In particular, $w$ is isotropic Gaussian conditioned on all $\{v_i\}_{i\in[N]}$, so $\nnorm{w}_2^2$ concentrates as
\begin{align*}
\Ebig{\nnorm{w}_2^2\mid \{v_i\}_{i\in[N]}} &= v_i^\top \qty(\sum_{k\le d_\tau} \sum_{t: d_t<k} \eta_t^2 (q_k^{(t)})^2 v_kv_k^\top + \sum_{k>d_\tau, k\notin\cJ\cup\{i\}} \sum_{t=0}^\tau \eta_t^2 (q_k^{(t)})^2 v_kv_k^\top) v_i \\
&= \sum_{t=0}^\tau \eta_t^2 \cdot v_i^\top \underbrace{\qty(\sum_{k>d_t, k\notin \cJ\cup\{i\}} (q_k^{(t)})^2 v_k v_k^\top)}_{=:\bOmega_t} v_i.
\end{align*}
Denote by~$\tq^{(t)}\in\RR^{N-d_t-d-1}$ the vector consisting of all~$(q_k^{(t)})^2$ with $k>d_t$, $k\notin \cJ\cup\{i\}$ for each $t\le\tau$. The number of items~$k>i+\sqrt{d}$ in the minibatch is~$O(Bi^{-\alpha})$ by the Chernoff bound. We have that
\begin{align*}
\nnorm{\tq^{(t)}}_1 & = \sum_{k>d_t, k\notin \cJ\cup\{i\}}(q_k^{(t)})^2 = \sum_{k>d_t, k\notin \cJ\cup\{i\}} p_k^2 \pm O\qty(Bi^{-\alpha}\qty(\frac{\log d}{B})^2) \asymp d_t^{1-2\alpha}, \\
\nnorm{\tq^{(t)}}_2 & \lesssim \sum_{d_t<k<i} p_k^4 + \sum_{k>i+\sqrt{d}} (q_k^{(t)})^4 \lesssim d_t^{1-4\alpha} + Bi^{-\alpha} \qty(\frac{\log d}{B})^4 \lesssim d_t^{1-4\alpha} + \frac{d_\tau^{2-4\alpha} \log d}{d^4} \lesssim d_t^{1-4\alpha}, \\
\nnorm{\tq^{(t)}}_\infty &\le \max_{k>d_t, k\notin \cJ\cup\{i\}} p_k^2 + \qty(\frac{\log d}{B})^2 \lesssim d_t^{-2\alpha}.
\end{align*}
Then $\Tr(\bOmega_t) \asymp \nnorm{\tq^{(t)}}_1$ and by the Hanson-Wright inequality and Lemma~\ref{lem:opf}, we have w.h.p.
\begin{align*}
\abs{v_i^\top\bOmega_t v_i - \frac{\Tr(\bOmega_t)}{d}} &\lesssim \frac{\sqrt{\log d}}{d} \nnorm{\bOmega_t}_{\F} + \frac{\log d}{d} \nnorm{\bOmega_t}_\op \\
%&\lesssim \frac{\sqrt{\log d}}{d} \qty(\frac{\nnorm{\tq^{(t)}}_1}{\sqrt{d}} + \nnorm{\tq^{(t)}}_2) + \frac{\log d}{d} \qty(\frac{\nnorm{\tq^{(t)}}_1}{d} + \nnorm{\tq^{(t)}}_\infty) \\
&\lesssim \frac{\sqrt{\log d}}{d} \qty(\frac{d_t^{1-2\alpha}}{\sqrt{d}} + d_t^{1-4\alpha}) + \frac{\log d}{d} \qty(\frac{d_t^{1-2\alpha}}{d} + d_t^{-2\alpha}) = o\qty(\frac{d_t^{1-2\alpha}}{d}).
\end{align*}
Defining $\bar\eta_t := \max_{0\le s\le t} \eta_s$, noting that $p_i\asymp d^{-1}d_\tau^{\frac12-\alpha}$, we thus have
\begin{align*}
\nnorm{w}_2^2 \asymp \Ebig{\nnorm{w}_2^2 \mid \{v_i\}_{i\in[N]}} \asymp \sum_{t=0}^\tau \eta_t^2 \cdot v_i^\top\bOmega_t v_i \gtrsim \sum_{t=0}^\tau \eta_t^2 \cdot\frac{d_t^{1-2\alpha}}{d} \gtrsim d\bar\eta_\tau^2 p_i^2.
\end{align*}
Also, for Eq.~\eqref{eq:matcha2}, noting that items $k\in\cJ$ also satisfy $k\lesssim B^\frac{1}{\alpha}(\log d)^{-\frac{1}{\alpha}}$ so that $q_k^{(t)}\asymp p_i$, we may directly bound
\begin{align*}
\abs{\sum_{t=0}^\tau \eta_t \sum_{k\in\cJ} q_k^{(t)} \langle u_j-u_i, u_k\rangle \langle v_i,v_k\rangle} \lesssim \sum_{t=0}^\tau \eta_t p_i \qty(\sqrt{\frac{\log d}{d}} + |\cJ|\frac{\log d}{d}) \lesssim \frac{\log d}{\sqrt{d}} \tau\bar\eta_\tau p_i
\end{align*}
and
\begin{align*}
\abs{\sum_{t=0}^\tau \eta_t q_i^{(t)} \langle u_j-u_i,u_i\rangle \nnorm{v_i}_2^2} \lesssim \tau\bar\eta_\tau p_i.
\end{align*}
Defining the i.i.d. standard Gaussian variables $Z_k := \sqrt{d}u_k^\top \frac{w}{\norm{w}_2}$ for $k\in\cJ\cup\{i\}$, we have thus shown that
\begin{align*}
\gamma_{\tau+1,ij} - \gamma_{\tau+1,ii} \ge \langle u_j-u_i,w\rangle - O(\tau\bar\eta_\tau p_i) \ge \frac{\norm{w}_2}{\sqrt{d}} (Z_j-Z_i-O(T)).
\end{align*}
If item~$i$ was recovered at step~$\tau+1$, it follows that $Z_i\ge Z_j - O(T)$ for all $j\in\cJ$. On the other hand, by Gaussian superconcentration (Lemma~\ref{lem:superconcentration}) it holds that $\max_{j\in\cJ} Z_j =  \sqrt{2\log |\cJ|} +o(1)$. By Mill's inequality, supposing $T=o(\sqrt{\log d})$,
\begin{align*}
\Pr(Z_i \ge \sqrt{2\log |\cJ|} - O(T)) \lesssim \frac{e^{-\frac12\qty(\sqrt{2\log |\cJ|}-O(T))^2}}{\sqrt{2\log |\cJ|}} \lesssim \frac{e^{O\qty(T\sqrt{\log|\cJ|})}}{|\cJ|\sqrt{\log|\cJ|}} = \frac{o(\poly(d))}{\sqrt{d}}.
\end{align*}
Hence item~$i$ cannot be recovered with constant probability among competitors~$\cJ$ (in fact, superconcentration is not needed to show an $o(1)$ bound for each step~$\tau$, but we elect to demonstrate the stronger near-uniform bound).

\paragraph{Case III: batch size threshold.} Items $i\asymp (TB)^\frac{1}{\alpha}$ have a constant probability of not being sampled in any minibatch, $q_i^{(0)}=\cdots = q_i^{(T)}=0$ so that $\bW_{\tau+1}$ is independent of~$u_i,v_i$. Fixing $(\log d)^2$ such items and comparing to item~$1$, it holds that $\gamma_{\tau+1,i1} - \gamma_{\tau+1,ii} = (u_i-u_1)^\top\bW_{\tau+1}v_i$ has probability $\frac12$ of being positive independently for each~$i$, and so at least one of these items will not be recovered w.h.p.

Therefore, if $d_\tau\gtrsim d(\log d)^{-4}$ but $B\lesssim dd_\tau^{\alpha-\frac12}\log d$, then $d_{\tau+1} \lesssim (TB)^\frac{1}{\alpha} \lesssim d^\frac{1}{2\alpha}d_\tau (T\log d)^\frac{1}{\alpha}$; and if $d_\tau\gtrsim d(\log d)^{-4}$ but $B\lesssim dd_\tau^{\alpha-\frac12}\log d$, then $d_{\tau+1} \lesssim (TB)^\frac{1}{\alpha} \lesssim d^\frac{1}{\alpha} d_\tau^{1-\frac{1}{2\alpha}} (T\log d)^\frac{1}{\alpha}$. Combining with the previous cases concludes Eq.~\eqref{eq:sgd-upper-recursion}.

\end{document}